\documentclass{article}

\usepackage{microtype}
\usepackage{graphicx}
\usepackage{subfigure}
\usepackage{booktabs} 
\usepackage{wrapfig}
\usepackage{hyperref}
\usepackage{enumitem}

\usepackage[accepted]{icml2023}

\usepackage{amsmath}
\usepackage{amssymb}
\usepackage{mathtools}
\usepackage{amsthm}
\usepackage{hyperref}
\usepackage{url}

\usepackage{algorithm}
\usepackage{color}
\usepackage[english]{babel}
\usepackage{grffile}
\usepackage{epstopdf}
\usepackage{algpseudocode}
\usepackage{multirow}
\usepackage[T1]{fontenc}
\usepackage{bbm}
\usepackage{comment}
\usepackage{dsfont}
\usepackage{makecell}
\usepackage{enumitem} 
\usepackage{listings}
\usepackage{subfigure}

\usepackage[capitalize,noabbrev]{cleveref}

\theoremstyle{plain}
\newtheorem{theorem}{Theorem}[section]

\newtheorem{lemma}[theorem]{Lemma}
\newtheorem{fact}[theorem]{Fact}

\theoremstyle{definition}
\newtheorem{definition}[theorem]{Definition}

\theoremstyle{remark}
\newtheorem{remark}[theorem]{Remark}

\newcommand{\R}{\mathbb{R}}
\DeclareMathOperator{\dist}{dist}
\DeclareMathOperator{\diag}{diag}

\newcommand{\ov}{\overline}
\newcommand{\maxip}{\mathsf{MaxIP}} 
\newcommand{\lsh}{$\mathsf{LSH}$}

\newcommand{\ann}{\mathsf{ANN}} 

\DeclareMathOperator{\tr}{tr}
\newcommand{\Tmat}{{\cal T}_{\mathrm{mat}}}
\DeclareMathOperator*{\E}{{\mathbb{E}}}

\usepackage[textsize=tiny]{todonotes}

\newcommand{\TODO}[1]{{\color{red}[TODO:]}}

\newcommand{\Zhao}[1]{{\color{blue}[Zhao: #1]}}

\usepackage{amsmath,amsfonts,bm,amssymb,amsthm,yfonts}

\def\eqref#1{equation~\ref{#1}}

\def\1{\bm{1}}

\def\rr{{\textnormal{r}}}

\def\vone{{\bm{1}}}

\def\vg{{\bm{g}}}

\def\vk{{\bm{k}}}

\def\vm{{\bm{m}}}

\def\vq{{\bm{q}}}

\def\vv{{\bm{v}}}
\def\vw{{\bm{w}}}
\def\vx{{\bm{x}}}
\def\vy{{\bm{y}}}
\def\vz{{\bm{z}}}

\def\rr{\mathbb{R}}

\DeclareMathAlphabet{\mathsfit}{\encodingdefault}{\sfdefault}{m}{sl}
\SetMathAlphabet{\mathsfit}{bold}{\encodingdefault}{\sfdefault}{bx}{n}

\def\t{\intercal}

\newcommand\name{\textsc{dejavu}}

\usepackage[compact]{titlesec}
\titlespacing{\section}{0pt}{*0.3}{*0}
\titlespacing{\subsection}{0pt}{*0.15}{*0}

\usepackage[subtle, mathdisplays=tight, charwidths=tight, leading=normal]{savetrees}

\def\setstretch#1{\renewcommand{\baselinestretch}{#1}}
\icmltitlerunning{Deja Vu: Contextual Sparsity for Efficient LLMs at Inference Time}

\begin{document}


\twocolumn[
\icmltitle{Deja Vu: Contextual Sparsity for Efficient LLMs at Inference Time}




\begin{icmlauthorlist}
\icmlauthor{Zichang Liu}{rice}
\icmlauthor{Jue Wang}{zju}
\icmlauthor{Tri Dao}{stanford}
\icmlauthor{Tianyi Zhou}{ucsd}
\icmlauthor{Binhang Yuan}{eth}
\icmlauthor{Zhao Song}{adobe}
\icmlauthor{Anshumali Shrivastava}{rice}
\icmlauthor{Ce Zhang}{eth}
\icmlauthor{Yuandong Tian}{fair}
\icmlauthor{Christopher Ré}{stanford}
\icmlauthor{Beidi Chen}{cmu,fair}
\end{icmlauthorlist}

\icmlaffiliation{adobe}{Adobe Research}
\icmlaffiliation{ucsd}{University of California, San Diego}
\icmlaffiliation{cmu}{Carnegie Mellon University}
\icmlaffiliation{zju}{Zhe Jiang University}
\icmlaffiliation{fair}{Meta AI (FAIR)}
\icmlaffiliation{stanford}{Stanford University}
\icmlaffiliation{rice}{Rice University}
\icmlaffiliation{eth}{ETH Zurich}

\icmlcorrespondingauthor{Zichang Liu}{zl71@rice.edu}
\icmlcorrespondingauthor{Tri Dao}{trid@stanford.edu}
\icmlcorrespondingauthor{Tianyi Zhou}{t8zhou@ucsd.edu}
\icmlcorrespondingauthor{Zhao Song}{zsong@adobe.com}
\icmlcorrespondingauthor{Beidi Chen}{beidic@andrew.cmu.edu}

\icmlkeywords{Machine Learning, ICML}

\vskip 0.3in
]



\printAffiliationsAndNotice{} 

\begin{abstract}

Large language models (LLMs) with hundreds of billions of parameters have
sparked a new wave of exciting AI applications. However, they are computationally expensive at inference time.
Sparsity is a natural approach to reduce this cost, but existing methods
either require costly retraining, have to forgo LLM's in-context learning
ability, or do not yield wall-clock time speedup on modern hardware.
We hypothesize that \emph{contextual sparsity}, which are small, input-dependent sets
of attention heads and MLP parameters that yield approximately the same output
as the dense model for a given input, can address these issues.
We show that contextual sparsity exists, that it can be accurately predicted,
and that we can exploit it to speed up LLM inference in wall-clock time
without compromising LLM's quality or in-context learning ability.
Based on these insights, we propose \name{}, a system that uses a low-cost
algorithm to predict contextual sparsity on the fly given inputs to each layer, along
with an asynchronous and hardware-aware implementation that speeds up LLM
inference. We validate that \name{} can reduce the inference latency of
OPT-175B by over 2$\times$ compared to the state-of-the-art FasterTransformer, and
over 6$\times$ compared to the widely used Hugging Face implementation, without
compromising model quality. The code is available at \url{https://github.com/FMInference/DejaVu}.
\end{abstract}

\vspace{-1mm}
\section{Introduction}
\label{sec:introduction}

Large language models (LLMs), such as GPT-3, PaLM, and OPT have demonstrated that an immense number of parameters unleashes impressive performance and emergent in-context-learning abilities---they can perform a task by conditioning on input-output examples, without updating their parameters~\cite{bommasani2021opportunities,liang2022holistic,brown2020language,min2022rethinking,chan2022data}. However, they are very expensive at inference time, especially for latency-sensitive applications~\cite{pope2022efficiently}. An ideal inference-time model should use less computation and memory while maintaining the performance and special abilities of pre-trained LLMs. The simplest and most natural approach is sparsification or pruning, which has a long history before the LLM era~\cite{lecun1989optimal}. Unfortunately, speeding up inference-time sparse LLMs in wall-clock time while maintaining quality and in-context learning abilities remains a challenging problem.

While sparsity and pruning have been well-studied, they have not seen wide adoption on LLMs due to the poor quality and efficiency trade-offs on modern hardware such as GPUs. First, it is infeasible to retrain or iteratively prune models at the scale of hundreds of billions of parameters. Thus, methods in iterative pruning and lottery ticket hypothesis~\cite{lee2018snip,frankle2018lottery} can only be applied to smaller-scale models. Second, it is challenging to find sparsity that preserves the in-context learning ability of LLMs. Many works have shown the effectiveness of task-dependent pruning~\cite{michel2019sixteen,bansal2022rethinking}, but maintaining different models for each task conflicts with the task independence goal of LLMs. Lastly, it is hard to achieve wall-clock time speed-up with unstructured sparsity due to its well-known difficulty with modern hardware~\cite{hooker2021hardware}. For example, recent development in zero-shot pruning like SparseGPT~\cite{frantar2023massive} finds 60\% unstructured sparsity but does not yet lead to any wall-clock time speedup.



\begin{figure}[]
  \vspace{-2mm}
  \centering
   \subfigure[Contextual Sparsity]{
    \hspace{0.5mm}\includegraphics[width=0.398\textwidth]{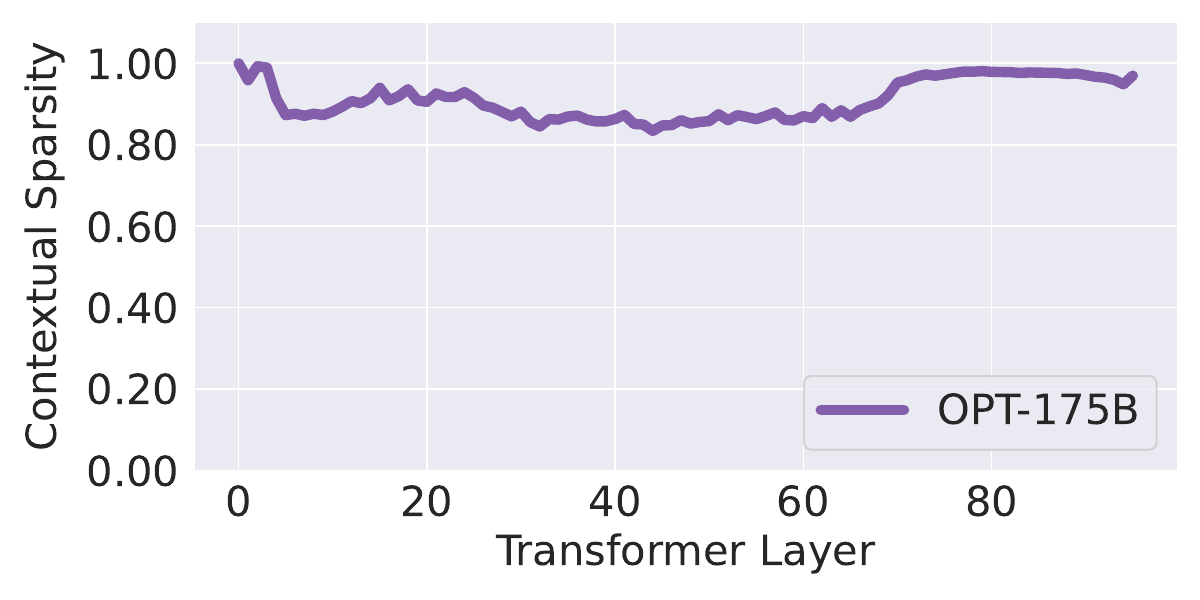}
    \label{fig:sparsity-175}
    }\\
    \vspace{-3mm}
      \subfigure[Accuracy-Efficiency Trade-offs]{
    \hspace{-2mm}\includegraphics[width=0.4\textwidth]{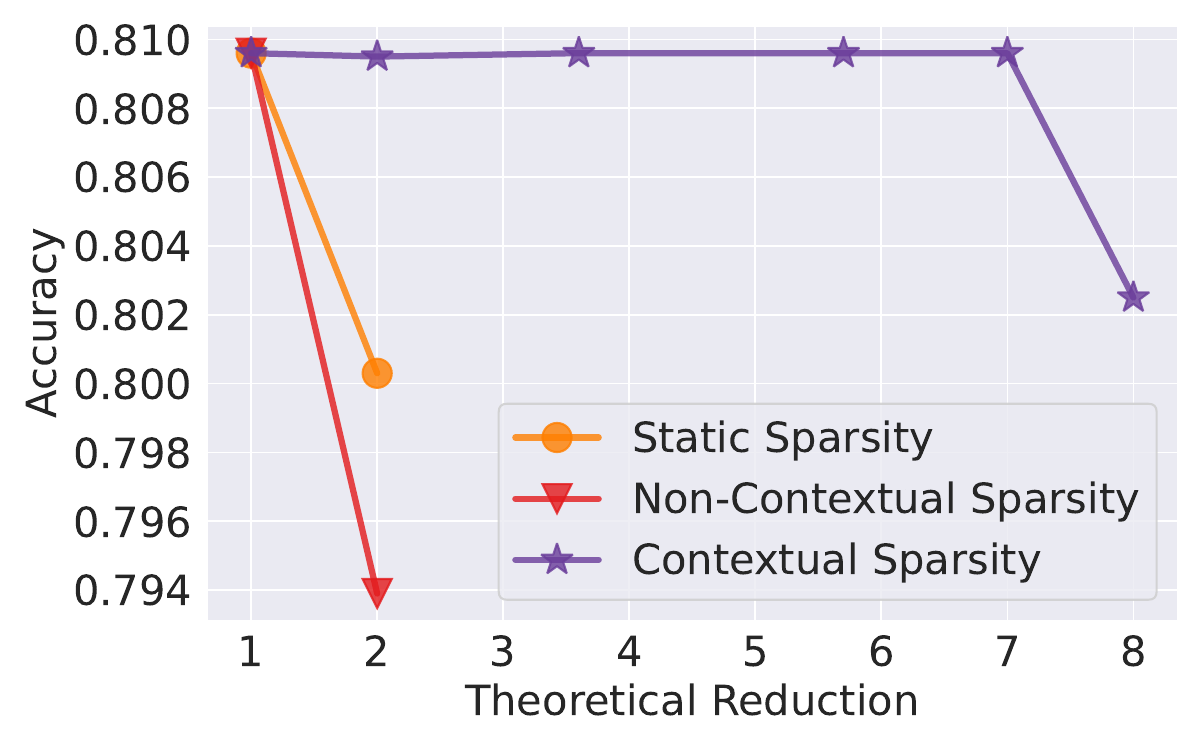}
    \label{fig:sparsity-175-compare}
      }
      \vspace{-4mm}
  \caption{(1) LLMs have up to 85\% contextual sparsity for a given input. (2) Contextual sparsity has much better efficiency-accuracy trade-offs (up to 7$\times$) than non-contextual sparsity or static sparsity.}
    \vspace{-2mm}
  \label{fig:contextual-static} 
  \vspace{-2mm}
\end{figure}

An ideal sparsity for LLMs should (i) not require model retraining, (ii) preserve quality and in-context learning ability, and (iii) lead to speed-up in wall-clock time on modern hardware. To achieve such  demanding requirements, we go beyond \emph{static} sparsity in previous works (e.g., structured/unstructured weight pruning). We instead envision \emph{contextual sparsity}, which are small, input-dependent sets of attention heads and MLP parameters that lead to (approximately) the same output as the full model for an input. Inspired by the connections between LLMs, Hidden Markov Models~\citep{xie2022an, baum1966statistical}, and the classic Viterbi algorithm~\citep{viterbi1967error}, we hypothesize that for pre-trained LLMs,
\vspace{-2mm}
\begin{center}
\textit{\textbf{contextual sparsity} exists given any input.}
\end{center}
\vspace{-1mm}
The hypothesis, if true, would enable us to cut off specific attention heads and MLP parameters (structured sparsity) on the fly for inference-time, without modifying pre-trained models.
However, there are three challenges. 

\underline{\emph{Existence}}: It is nontrivial to verify if such contextual sparsity exists, and naive verification can be prohibitively expensive. 
  \vspace{-0.1mm}
  
\underline{\emph{Prediction}}: Even if contextual sparsity exists, it is challenging to predict the sparsity for a given input in advance. 
  
  \vspace{-0.1mm}
\underline{\emph{Efficiency}}: Even if the sparsity can be predicted, it might be difficult to achieve end-to-end wall-clock time speedup. Taking OPT-175B as an example, the latency of one MLP block is only 0.2 ms on an 8$\times$A100 80GB machine. Without a fast prediction and optimized implementation, the overhead can easily increase the LLM latency rather than reduce it.

In this work, we address these challenges as follows:

\textbf{Existence}: Fortunately, we verify the existence of contextual sparsity with a surprisingly simple approach. To achieve essentially the same output, contextual sparsity is on average \textit{85\%} structured sparse and thereby potentially leads to a $7\times$ parameter reduction for each specific input while maintaining accuracy (Figure~\ref{fig:sparsity-175}). During explorations of contextual sparsity, we make important empirical observations and build a theoretical understanding of major components in LLMs that help address the prediction and efficiency challenge.

\textbf{Prediction}: We discover that contextual sparsity depends not only on individual input tokens (i.e., \emph{non-contextual} \emph{dynamic} sparsity) but also on their interactions (\emph{contextual dynamic} sparsity). Figure~\ref{fig:sparsity-175-compare} shows that with pure dynamic information, sparsity prediction is inaccurate. Only with token embeddings with sufficient contextual information can we predict sparsity accurately. Another finding is that \emph{contextual dynamic} sparsity for every layer can be predicted based on the ``similarity'' between layer parameters (heads/MLP) and the output from the previous layer, which carries the immediate contextual mixture of token embeddings.  





\textbf{Efficiency}: Because at inference time, model parameters are static, inspired by the classical nearest neighbor search (NNS) literature and its applications in efficient deep learning, it is possible to formulate the above similarity-based prediction as an NNS problem~\cite{indyk1998approximate,zhang2018navigating,chen2020slide}. However, as mentioned, the overhead might be difficult to overcome as we would need to perform on-the-fly predictions before every layer. Luckily, we exploit a phenomenon of LLM where token embeddings change slowly across layers due to residual connections (well-known in computer vision~\cite{he2016deep}). Since the inputs to a few consecutive layers are very similar, we can design an asynchronous lookahead predictor (Figure ~\ref{fig:workflow_main}).

\begin{figure}[]
\vspace{-2mm}
  \centering
      \hspace{0.5mm}\includegraphics[width=0.34\textwidth]{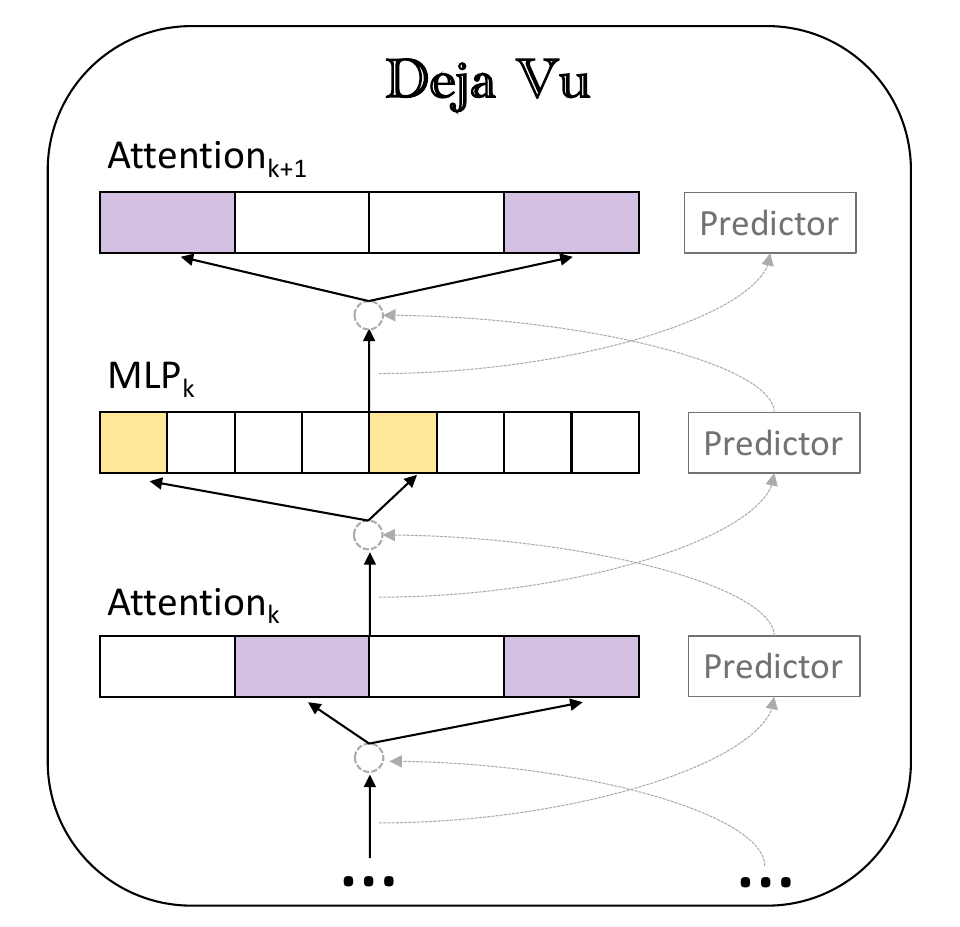}
          \vspace{-3mm}
  \caption{\name{} uses lookahead predictors to side-step prediction costs: given the input to the attention layer at block $k$, they (asynchronously) predict the contextual sparsity for the MLP at block $k$, and given the input to the MLP at block $k$, they predict the sparsity for the attention head at the next layer.}
  \label{fig:workflow_main} 
    \vspace{-4mm}
\end{figure}

Based on our findings, we present a system, \name{}, that exploits contextual sparsity and realizes efficient LLMs for latency-sensitive applications.
\vspace{-0.7mm}
\begin{itemize}[itemsep=0.1pt,topsep=0pt,leftmargin=*]
\item  In Section~\ref{sec:routing_mlp} and Section~\ref{sec:routing_attn}, we present a low-cost learning-based algorithm to predict sparsity on the fly. Given the input to a specific layer, it predicts a relevant subset of attention (heads) or MLP parameters in the next layer and only loads them for the computation.
\vspace{-0.2mm}
\item  In Section~\ref{sec:hide_overhead}, we propose an asynchronous predictor (similar to classic branch predictor~\cite{smith1998study}) to avoid the sequential overhead. A theoretical guarantee justifies that the cross-layer design suffices for accurate sparsity prediction.
    
\end{itemize}


After integrating hardware-aware implementation of sparse matrix multiply (Section~\ref{sec:sparse_matmul}), \name{} (written mostly in Python) can reduce latency of open-source LLMs such as OPT-175B by over 2$\times$ end-to-end without quality degradation compared to the state-of-the-art library FasterTransformer from Nvidia (written entirely in C++/CUDA), and over 2$\times$ compared to the widely used Hugging Face implementation at small batch sizes. Furthermore, we show several ablations on different components of \name{} and its compatibility with quantization techniques.

\section{Related Work and Problem Formulation}
We first briefly discuss the rich literature on efficient inference. Then, we introduce the latency breakdown in our setting. Last, we provide a formal problem formulation. 
\label{sec:obs_computation}

\subsection{Quantization, Pruning, Distillation for Inference}
Various relaxations have been studied for decades for model inference in machine learning. There are three main techniques: quantization~\cite{han2015deep, jacob2018quantization,nagel2019data,zhao2019improving}, pruning or sparsity~\cite{molchanov2016pruning,liu2018rethinking,hoefler2021sparsity}, and distillation~\cite{hinton2015distilling,tang2019distilling,touvron2021training}. They are orthogonal areas and usually excel in different settings. Recently, there is active research attempting to apply one or a combination of such techniques in LLM inference~\cite{yao2022zeroquant,park2022nuqmm,dettmers2022llm,frantar2022gptq,frantar2023massive,bansal2022rethinking,xiao2022smoothquant}. More discussion is presented in Appendix~\ref{appendix:related_work}.

\subsection{LLM Inference Latency Breakdown}
The generative procedure of LLMs consists of two phases: (i) the \textit{prompt} phase takes an input sequence to generate the keys and values (KV cache) for each transformer block of LLMs, which is similar to the forwarding pass of LLMs training; and (ii) the \textit{token generation} phase utilizes and updates the KV cache to generate tokens step by step, where the current token generation depends on previously generated tokens. 

This paper studies the setting where the token generation phase easily dominates the end-to-end inference time. As shown in Table~\ref{table:obs_break_down_stage}, generating a sequence of length 128 takes much longer time than processing a sequence of length 128 as prompt due to I/O latency of loading model parameters. In addition, Table~\ref{table:obs_break_down_block} shows that attention and MLP are both bottlenecks in LLMs, e.g., in 175B models, loading MLP parameters takes around $\frac{2}{3}$ of the total I/O and attention heads take the other $\frac{1}{3}$. Further, in the tensor-parallel regime, there are two communications between GPUs, one after the attention block, and the other one after the MLP block. As shown in Table~\ref{table:obs_break_down_allreduce}, communication between GPUs takes around 15 \% token generation latency. This paper focuses on making attention and MLP more efficient. Communication cost implies that the upper bound of such speed-up is around 6$\times$ when skipping all transformer blocks. 


\begin{table}[H]
\vspace{-4mm}
\scriptsize
\centering
\caption{Theoretical breakdown for prompting versus token generation (tensor model parallelism on 8 A100-80G GPUs).}
\vspace{2mm}
\resizebox{0.9\linewidth}{!}{
\centering
\Huge
\begingroup
\setlength{\tabcolsep}{10pt}
\renewcommand{\arraystretch}{1.4}
\begin{tabular}{c|c|c|c|c}
\toprule
  & TFLOPs & I/O & Compute Latency (ms) & I/O Latency (ms)	\\
\hline
  Prompting 128 & 44.6 &  330 GB & 17.87 &    20.6 \\
 \hline
 Token Generation 128 & 44.6  & 41 TB  & 17.87 & 2600   \\
 \bottomrule
\end{tabular}
\endgroup
}
\vspace{-4mm}
\label{table:obs_break_down_stage}
\end{table}

\begin{table}[H]
\vspace{-4mm}
\scriptsize
\centering
\caption{Theoretical breakdown for Attention block versus MLP block in one transformer layer when generating one token (tensor model parallelism on 8 A100-80G GPUs).}
\resizebox{0.9\linewidth}{!}{
\centering
\Huge
\begingroup
\setlength{\tabcolsep}{10pt}
\renewcommand{\arraystretch}{1.4}
\begin{tabular}{c|c|c|c|c}
\toprule
  & GFLOPs & I/O (GB) & Compute Latency (ms) & I/O Latency (ms)	\\
\hline
 Attention Block   & 1.21 & 1.12  & 0.00048 & 0.07 \\
 \hline
 MLP Block  & 2.41 & 2.25 & 0.00096 & 0.14 \\
 \bottomrule
\end{tabular}
\endgroup
}
\vspace{-4mm}
\label{table:obs_break_down_block}
\end{table}

\begin{table}[H]
\vspace{-4mm}
\scriptsize
\centering
\caption{Latency breakdown of generating 1 token under the setting of batch size 1 and prompt length 128 on 8 A100-80GB. }
\resizebox{0.7\linewidth}{!}{
\centering
\Huge
\begingroup
\setlength{\tabcolsep}{10pt}
\renewcommand{\arraystretch}{1.4}
\begin{tabular}{c|c|c|c}
\toprule
  All Reduce & MLP Block & Attention Block (ms) & Others	\\
\hline
    6 ms & 19ms  & 13ms & 2ms \\
 
 \bottomrule
\end{tabular}
\endgroup
}
\vspace{-4mm}
\label{table:obs_break_down_allreduce}
\end{table}

\subsection{Problem Formulation}
\label{sec:formulation}
The goal is to reduce the generation latency of LLMs by exploiting contextual sparsity. In the following, we formally define the sparsified attention and MLP blocks.

\textbf{Sparsified MLP:} There are two linear layers in one MLP block, $W^1$, $W^2 \in \R^{ d \times 4d}$.  Denote $y\in \R^{1 \times d}$ as the input to the MLP block in the current generation step.
Let each column (the weight of $i$-th neuron) of linear layers be $W^{1}_{i}$, $W^{2}_{i}\in \R^{d\times 1}$. With contextual sparsity, only a small set of them are required for computation. Let  $S_M \subseteq [4d]$ denote such set of neurons for input $y$. The sparsified MLP computation is
\begin{align}\label{eq:MLP_S_y}
    \mathsf{MLP}_{S_M}(y) = \sigma( y W^{1}_{S_M} ) (W^{2}_{S_M})^{\top}, 
\end{align} 

where $\sigma$ is the activation function, e.g., ReLU, GeLU.  
Note that since the computation in the first linear results in sparse activations, the second linear layer is also sparsified.

\textbf{Sparsified  Attention:}  Let $X \in \R^{n \times d}$ denote the embeddings of all tokens (e.g., prompts and previously generated tokens). Let $y \in \R^{1 \times d}$ be the input to the Multi-Head-Attention (MHA) in the current generation step. Suppose there are $h$ heads. For each $i\in [h]$, we use $W^K_i, W^Q_i, W^V_i \in \R^{d \times d_h}$ to denote key, query, value projections for the $i$-th head, and $W_i^O \in \R^{d_h \times d}$ for output projections. With contextual sparsity, we denote $S_A$ as a small set of attention heads leading to approximately the same output as the full attention for input $y$.
Following the notation system in \cite{as23}, sparsified MHA computation can be formally written as 
\begin{equation*}
    \mathsf{MHA}_{S_A} (y) = \sum_{i\in S_A} \underbrace{ H_i(y) }_{1 \times d_h} \underbrace{ W^O_i }_{d_h \times d}, 
\end{equation*} 

where $H_i(y) : \R^{d} \rightarrow \R^{d_h}$ and $D_i(y) \in \R$ can be written as
\begin{align}\label{eq:H_i_y}
H_i(y) := D_i(y)^{-1} \exp( y W^Q_i (W^K_i)^\top X^\top ) X W^V_i,
\end{align}
\begin{align*}
D_i(y) :=  \exp( y W^Q_i (W^K_i)^\top X^\top ) {\bf 1}_n. 
\end{align*}

For both MLP and Attention, given a compute budget, the goal is to find $S_M$ and $S_A$ that minimize the error between the sparse approximation and full computation.

\section{Pre-trained LLMs are Contextually Sparse} \label{sec:obs}
In this section, we present several key observations and theoretical understandings of sparsity in LLMs, upon which the \name{} design is based.
We first test the contextual sparsity hypothesis and verify that contextual sparsity exists in pre-trained LLMs in Section~\ref{sec:sparse_obs}. Then, we build an understanding of why contextual sparsity happens naturally even when LLMs are densely trained in Section~\ref{sec:obs_att_cluster}. Finally, we present an observation on residual connections and explain their relationship to contextual sparsity analytically in Section~\ref{sec:obs_slowly_changing}.

\begin{figure}[t]
  \centering
    \subfigure[Contextual sparsity in Attention Head]{
    \hspace{0mm}\includegraphics[width=0.36\textwidth]{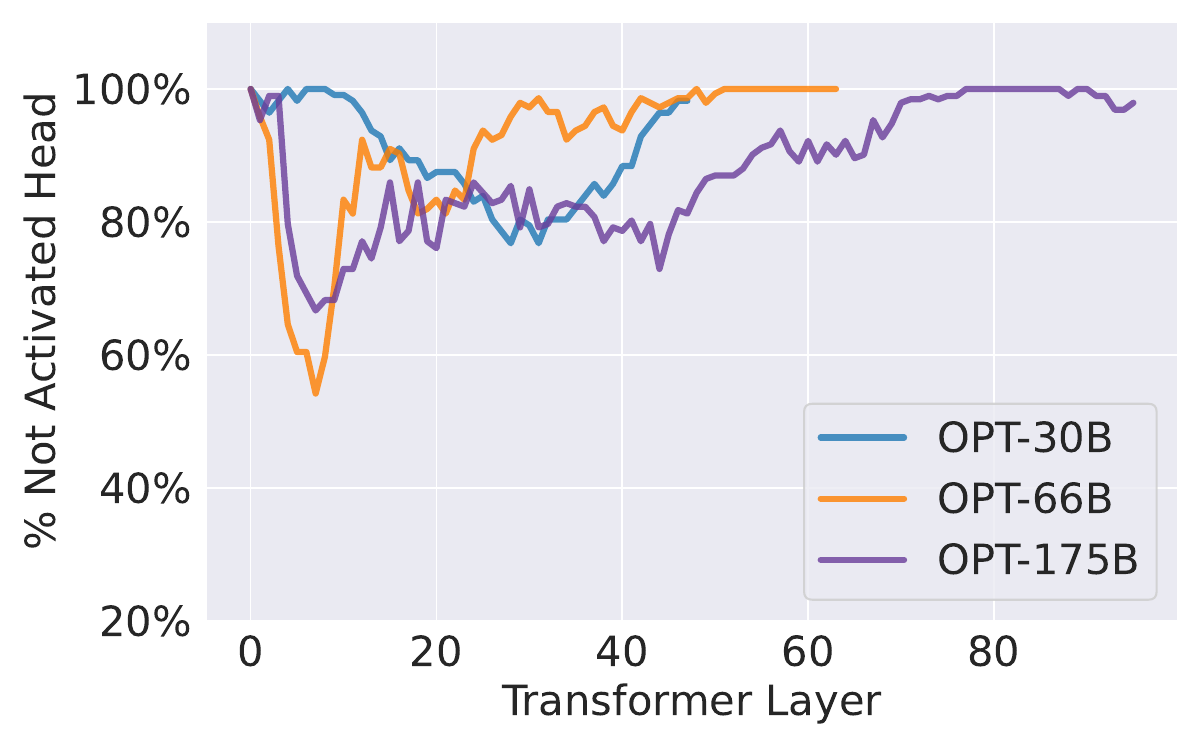}
    \label{obs:175b-sparsity-att} 
  }\\
  \vspace{-2mm}
  \subfigure[Contextual sparsity in MLP Block]{
    \includegraphics[width=0.36\textwidth]{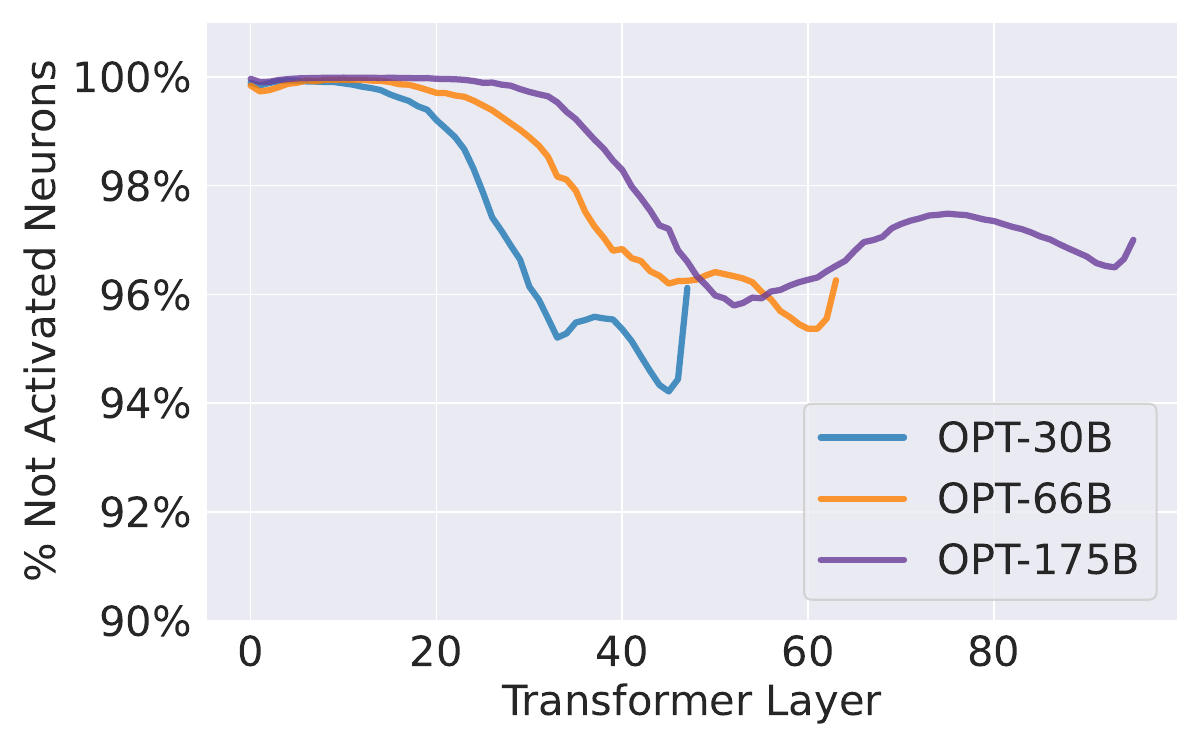}
    \label{obs:175b-sparsity-mlp} 
  }
    \vspace{-2mm}
  \caption{ In Figure (a), we plot the percentage of not-activated attention heads. By only keeping heads that yield large output norms, we can silence over 80\% attention heads for a given token. In Figure (b), we plot the average sparsity we impose on MLP layers. We can zero out over 95\% of MLP parameters for a given token.}
    \vspace{-4mm}
  \label{observation:sparsity} 
\end{figure}

\subsection{Contextual Sparsity Hypothesis}
\label{sec:sparse_obs}
Inspired by prior pruning literature~\cite{molchanov2016pruning}, we find a surprisingly simple method is sufficient to study and verify our hypothesis. In this section, we describe the testing procedure, observation details, and insights of this study.   

\textbf{Verification:} Our test is performed on OPT-175B, 66B, and 30B models and various downstream datasets such as OpenBookQA~\cite{OpenBookQA2018} and Wiki-Text~\cite{merity2016pointer}. We find the contextual sparsity for every input example with two forward passes of the model. In the first pass, we record a subset of parameters, specifically which attention heads and MLP neurons yield large output norms for the input. In the second pass, each input example only uses the recorded subset of parameters for the computation. Surprisingly, these two forward passes lead to similar prediction or performance on all in-context learning and language modeling tasks.

\textbf{Observation:}  Figure~\ref{observation:sparsity} shows that on average, we can impose up to 80\% sparsity on attention heads and 95\% sparsity on MLP neurons. As mentioned in Section~\ref{sec:obs_computation}, OPT-175B model has $2\times$ MLP parameters than those of attention blocks. Therefore total sparsity here is around 85\%. Since these are all structured sparsity (heads and neurons), predicting them accurately could potentially lead to $7 \times$ speedup.   

\textbf{Insight:} It is intuitive that we can find contextual sparsity in MLP blocks at inference time because of their activation functions, e.g., ReLU or GeLU~\cite{pmlr-v119-kurtz20a}. Similar observations were made by~\cite{sanjiv}. However, it is surprising that we can find contextual sparsity in attention layers. Note that, finding contextual sparsity in attention is not the same as head pruning. We cross-check that different examples have different contextual sparsity. Although $80\%$ of the parameters are not included in the paths for a given example, they might be used by other examples. Next, we will try to understand why contextual sparsity exists in attention blocks.

\subsection{Token Clustering in Attention Layers}
\label{sec:obs_att_cluster}

In the previous section, we have verified that there exists contextual sparsity for a given input in LLMs. In this section, we try to understand the reason for such phenomena, especially in attention layers. We first show an in-depth observation of attention. Then we present a hypothesis that self-attentions are conceptually clustering algorithms. Last we show analytical evidence to support this hypothesis.

\textbf{Observation:} Figure~\ref{fig:head_uniform} shows the attention map of three different heads from the same layer for an example input. The next token it should predict is ``Truck''. Darker color represents higher attention scores. We observe that the middle head is a relatively uniform token-mixing head while the top and bottom ones are ``heavy hitter'' attention heads (with high attention to ``like'' and ``shipping''). Unsurprisingly, only selecting heavy hitter heads but not uniform heads does not affect the prediction, since uniform heads do not model or encode important token interactions. In the next section, we will also explain in detail how the criteria for selecting uniform attention heads and heads with small output norms are highly correlated.    

\def\vone{\mathbf{1}}

\textbf{Hypothesis:} We hypothesize that the attention head is performing mean-shift clustering~\cite{derpanis2005mean}. 

Recall the notation defined in Section~\ref{sec:formulation}. For $i$-th head at current layer, $X = [x_1, \ldots, x_n]^{\top} \in \mathbb{R}^{n\times d}$ are the token embeddings in the previous time steps. $X W_i^K $ and $X W_i^V $ are the projection of embedding. For an input embedding $y$, the output $\tilde y_i = H_i(y)$, where $H_i(y)$ is defined in Eq.~\ref{eq:H_i_y}. 

\def\vm{\mathbf{m}}
For each $i \in [h]$, if we let $K_i(x_j,y) := \exp(y W_i^Q(W_i^K)^\top x_j)$ measure the similarity between $x_j$ and $y$, and define $m_i(y) := \frac{\sum_j K_i(x_j,y) x_j}{\sum_j K_i(x_j,y)}$, then we have $\tilde y_i =  m_i(y) W_i^V$. Further, if we set $W^V_i=I$ and consider the residue connection followed by layer norm, then in the next layer, the embedding $\hat y_i$ of the current token becomes $\hat y_i = \mathrm{Normalize}(y + \tilde y_i) = \mathrm{Normalize}(y + m_i(y))$, which has a fixed point $y = \gamma m_i(y)$ for any scalar $\gamma$. This iteration bears a resemblance to mean-shift clustering, which simply performs iteration $y \leftarrow m_i(y)$ until convergence. This has an obvious fixed point $y = m_i(y)$. 

Therefore, the self-attention head can be regarded as \emph{one mean-shift step} to push input embeddings of different tokens together, if they are already neighbors in a projection space specified by $W_i^Q (W_i^K)^\top $. Different heads learn different projection spaces to perform clustering. These dynamics explain the precise reason why token embeddings tend to cluster after going through more layers, resulting in high attention scores among cluster members, and low scores for non-members. Furthermore, the cluster patterns are different at different heads (More details in Appendix~\ref{sec:clustering understanding}).

The above analysis not only provides an understanding of why contextual sparsity exists naturally in pre-trained LLMs, but also inspires our design of ``similarity''-based sparsity prediction for \name{} in Section~\ref{sec:method}.  


\begin{figure}[]
 \vspace{-2mm}
  \centering
    \includegraphics[width=0.47\textwidth]{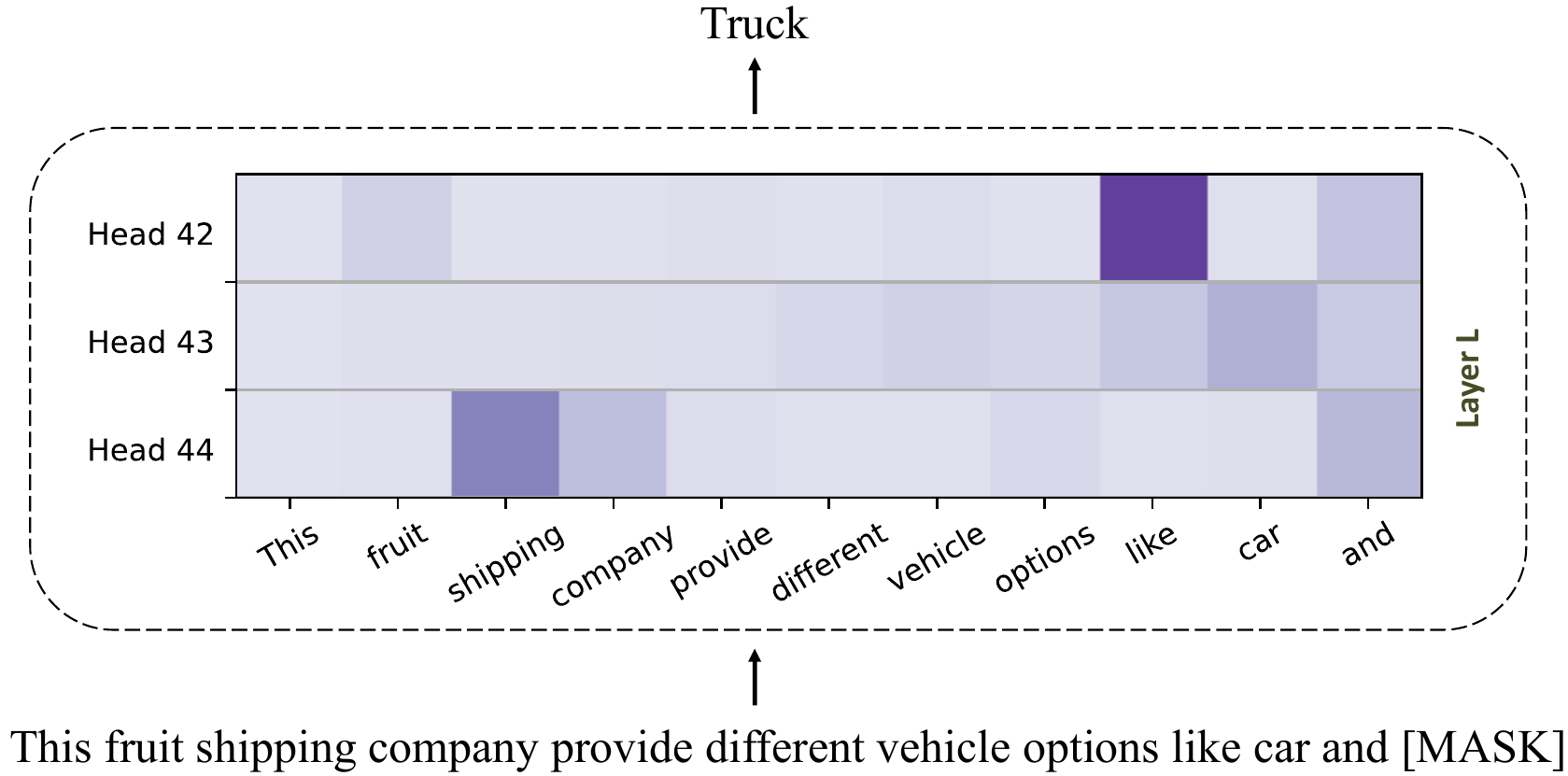}
    \vspace{-3mm}
  \caption{ We visualize the attention scores of three different heads for an exemplary sentence. Head 42 and Head 44 give heavy attention scores on particular tokens while Head 43 is more uniform.     }
  \label{fig:head_uniform} 
     \vspace{-4mm}
\end{figure}

\subsection{Slowly Changing Embeddings across Layers}
\label{sec:obs_slowly_changing}

\begin{figure}[]
\vspace{-2mm}
  \centering
 \subfigure[Model Comparison]{
    \includegraphics[width=0.22\textwidth]{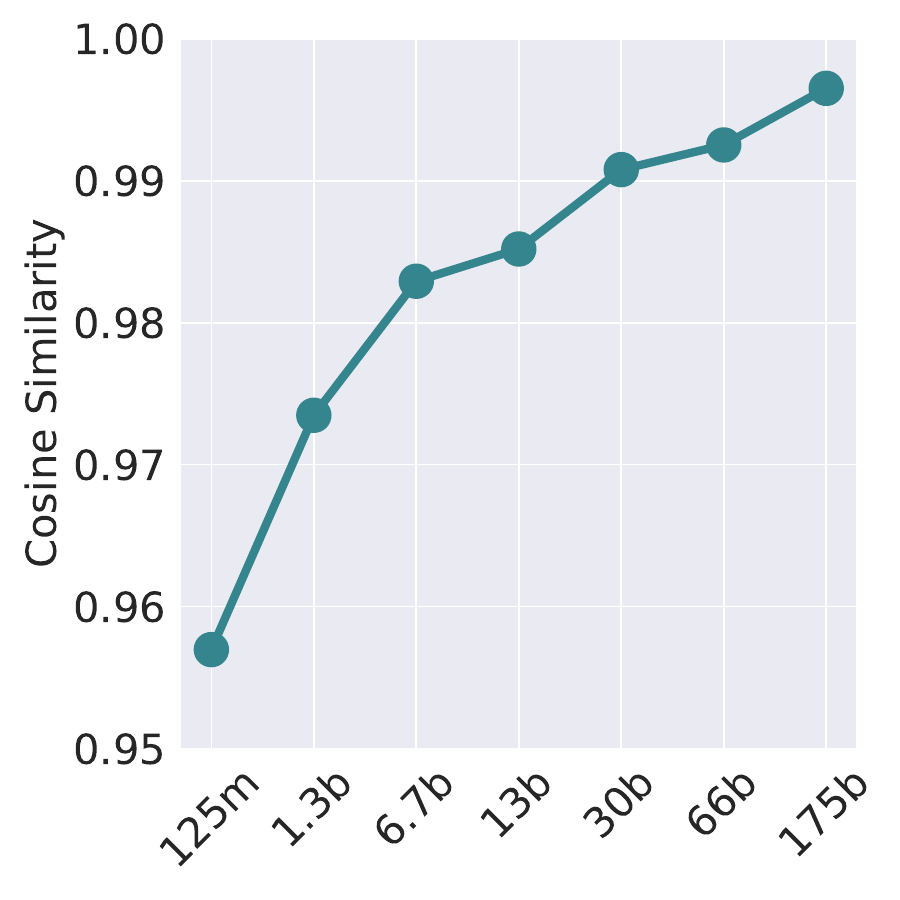}
    \label{obs:slowlyevoloving-all}
    }
  \subfigure[Across Layer]{
    \includegraphics[width=0.22\textwidth]{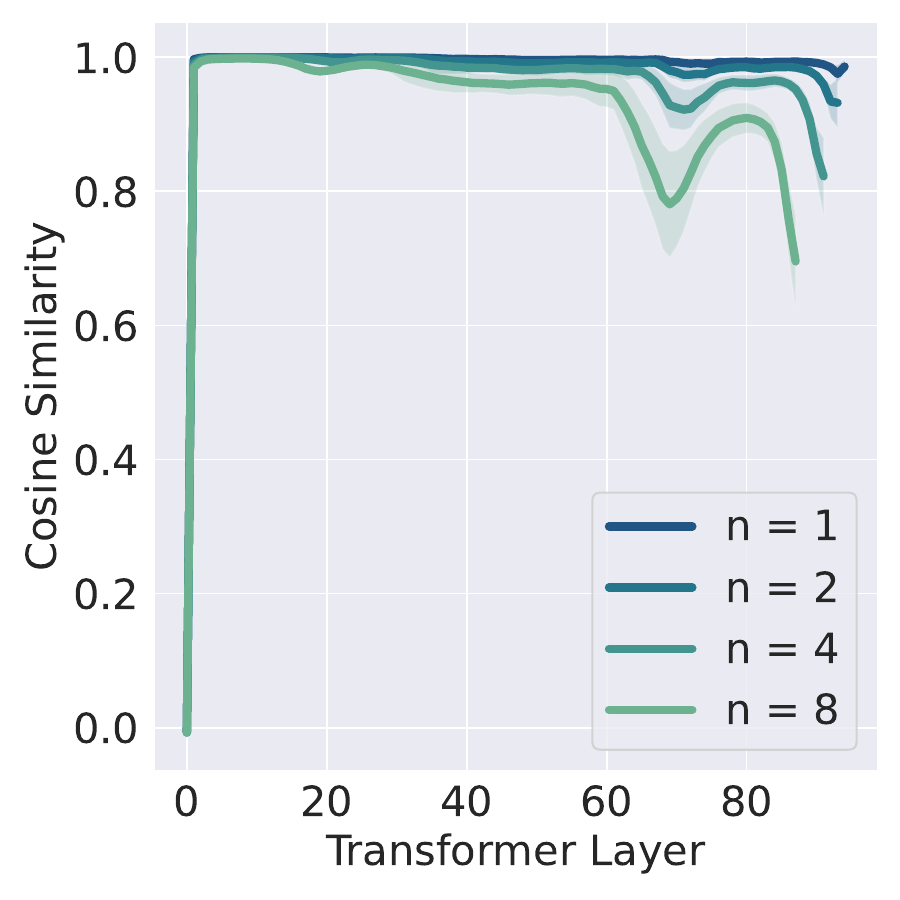}
    \label{obs:slowlyevoloving-175b}
  } \\
  \vspace{-4mm}
    \subfigure[Residual Around Attention]{
    \includegraphics[width=0.22\textwidth]{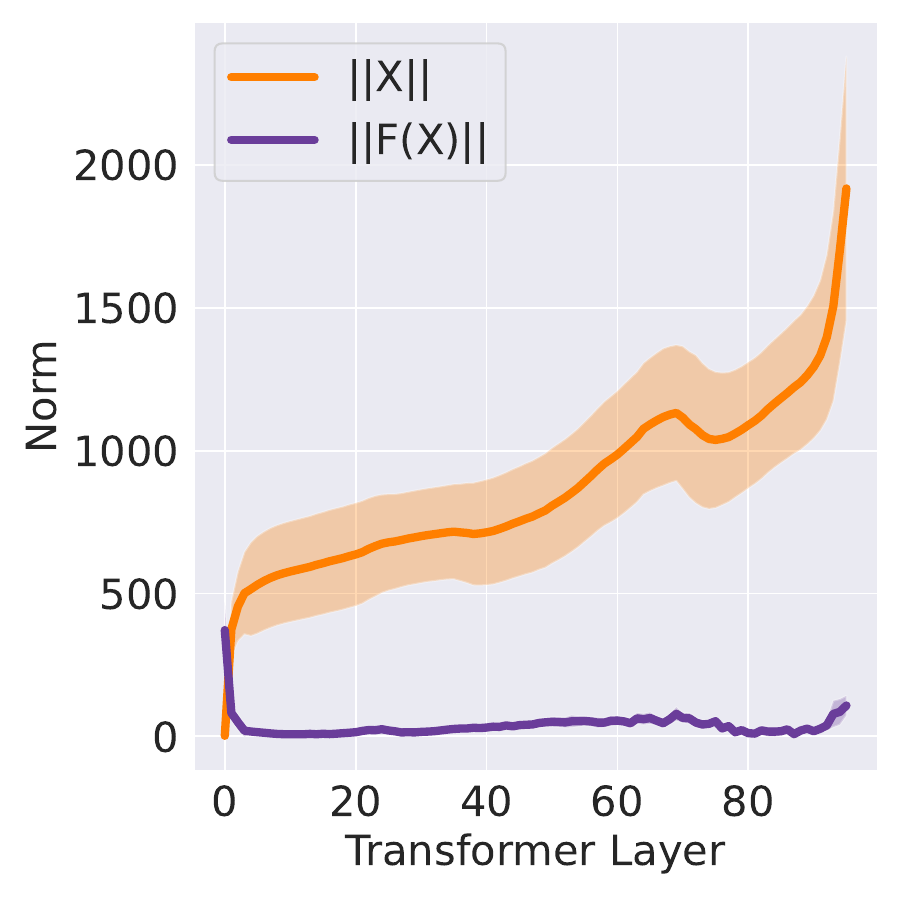}
    \label{obs:attention_residual}
  } 
    \subfigure[Residual Around MLP]{
    \hspace{1mm}\includegraphics[width=0.22\textwidth]{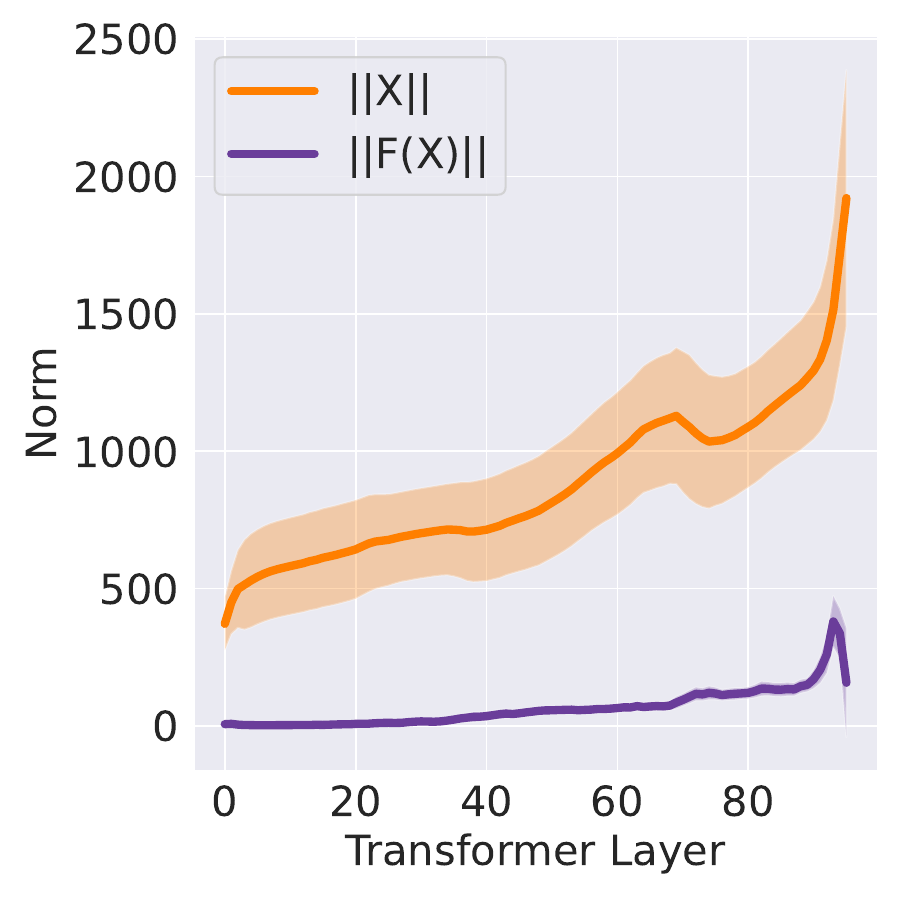}
    \label{obs:mlp_residaul}
  } 
  \vspace{-0.3em}
  \caption{\textbf{Slowly Changing Embedding.} Figure (a) shows the median cosine similarity between representations at two consecutive layers across all layers for different OPT models. All models show a similarity greater than 95\%. Figure (b) shows cosine similarity stays high even a few layers apart. For the residual connection $X' = X + F(X)$ inside each block, we plot the $\ell_2$ norm of $X$ and $F(X)$ in Figure (c) and Figure (d). $\|X\|$ is significantly higher than $\|F(X)\|$, which explains the slowly changing embedding.\vspace{-1em}}
  \label{observation_residual} 
\end{figure}

We first present our observation that embeddings change slowly across consecutive layers. Then we provide a detailed analysis on the phenomenon. Finally, we show its close connection with contextual sparsity.  Details are in Section~\ref{sec:appendix-obs}. 

\textbf{High similar embeddings in consecutive layers:} In Figure~\ref{obs:slowlyevoloving-all}, we show that for the same given input, the cosine similarity between embeddings or activations in two consecutive layers is exceptionally high on 7 different sizes of OPT models. Specifically, we collect activations from each layer while performing OPT model inference on C4 validation set~\cite{2019t5}. Taking OPT-175B as an example, starting from the second layer, the similarity between any two consecutive layers is around 0.99, which indicates that when an input is passed through the model, the direction of its embedding  changes slowly. Interestingly, the most drastic change happens in the first layer. Furthermore, we increase the gap and investigate the similarity between the embedding at layer $l$ and at layer $l + n$ shown in Figure~\ref{obs:slowlyevoloving-175b}. As we increase the gap, the similarity decreases as expected while the differences in cosine similarity between various choices of $n$ are smaller at the shallower layer. We plot the mean similarity, and the standard deviation is indicated by the shading. Similar plots on more models are presented in Appendix~\ref{sec:appendix-obs}.  

\textbf{Connection to residuals:} We verify that the high similarity in embeddings in LLM inference is due to the residual connection. We first dissect the computation graph inside each transformer layer to understand the cause behind this phenomenon. There are two residual connections inside a transformer layer, one around the attention block, and the other one around the MLP block. The residual connection can be written as $X + F(X)$, where $F$ is either the Multi-Head Attention or two MLP Layers.  In Figure~\ref{obs:attention_residual} and Figure~\ref{obs:mlp_residaul},  indeed we can see that $\|X\|$ is significantly greater than $\|F(X)\|$, confirming that embeddings are changing slowly because the residual norm is large.   

\textbf{Connection to Contextual Sparsity:} We take a step deeper trying to understand the reason behind the large residual norm with mathematical modeling.  
We discover that one possible reason for small $\|F(X)\|$ is due to high sparsity. For the MLP Block, high sparsity may contribute to the small norm of $F(X)$ because a large portion of outputs have small norms. Similar reasoning applies to the Attention Block, and thus a large number of attention heads yield small norm outputs.

\textbf{Residual Two Sides Bound:} Besides empirical reasoning, we formally define the computation of LLMs mathematically. Under our computation model, we can show that a shrinking property which is observed by our practical experiments. Proofs are in Appendix \ref{sec:subspace_embedding}, \ref{sec:distances_angles}, \ref{sec:function_approx}.

\begin{lemma}[Informal]
Let $0 < \epsilon_1 < \epsilon_2< 1$ be the lower and upper bound of the shrinking factor.
Let $x$ be the $y$ be the output. We have the residual connection $y = x + F(x)$. For the MLP block $F(x)$,  
 we have $\epsilon_1 \leq \| y - x \|_2 \leq \epsilon_2$. For the attention block $F(x)$,  we have $\epsilon_1 \leq \| y - x \|_2 \leq \epsilon_2 $.  
\end{lemma}







\section{\name}
\label{sec:method}

In this section, we present our framework for inference-time contextual sparsity search for LLMs.  We introduce the sparsity predictor for MLPs in Section~\ref{sec:routing_mlp} and for attention heads in Section~\ref{sec:routing_attn}. \name{}'s workflow is shown in Figure~\ref{fig:workflow_main}.  Section~\ref{sec:hide_overhead} discusses exploiting our observation on LLMs to avoid the sparse prediction overhead with theoretical guarantees. In Section~\ref{sec:sparse_matmul}, we present our optimized implementation that enables end-to-end latency reduction. More details are presented in Section~\ref{appendix:method}.



\subsection{Contextual Sparsity Prediction in MLP Blocks}
\label{sec:routing_mlp}
As explained in Section~\ref{sec:obs_computation}, MLP blocks are one of the major bottlenecks for the LLM generation ($\frac{2}{3}$ of the FLOPs and IOs). In this section, we discuss how we achieve wall-clock time speed-up with contextual sparsity in the MLP blocks.

\textbf{Challenge} Figure~\ref{obs:175b-sparsity-mlp} shows that for a given token, the contextual sparsity of 95\% is possible.  The contextual sparsity in the MLP block can be identified after computing the activation. However, this only demonstrates the existence of contextual sparsity but brings no benefits in terms of efficiency. A fast and precise prediction is needed to exploit contextual sparsity for end-to-end efficiency. The naive way is to select a subset of neurons randomly. Unsurprisingly, random selection fails to identify the accurate contextual sparsity, resulting in drastic model degradation. 

\textbf{A Near-Neighbor Search Problem:}  Recall that we verify the existence of contextual sparsity by recording which neurons yield significant norms. Essentially, given the input, the goal is to search for the neurons that have high inner products with the input, because the activation function ``filters" low activation. Thus, we formulate the contextual sparsity prediction of an MLP layer as the classical near-neighbor search problem under the inner product metric. 
\begin{definition}[Approximate $\maxip$ in MLP]\label{def:approximate_maxip:informal}
Let $c \in (0,1)$ and $\tau \in (0,1)$ denote two parameters.
Given an $n$-vector dataset $W^1 \subset \mathbb{S}^{d-1}$ on a unit sphere, the objective of the $(c,\tau)$-{$\maxip$} is to construct a data structure that, given a query $y \in \mathbb{S}^{d-1}$ such that $\max_{w\in W^1}\langle y, w \rangle \geq \tau$, it retrieves a vector $z$ from $W^1$ that satisfies $\langle y, z \rangle \geq c \cdot \max_{w \in W^1} \langle y,w \rangle$.
\end{definition}

\begin{remark}
Our $W^1$ (first linear layer) and $y$ (input embedding) in MLP blocks can be viewed as the dataset and query  in Definition~\ref{def:approximate_maxip:informal}  respectively.
\end{remark}



\textbf{Design} The standard state-of-the-art near-neighbor search methods and implementations slow down the computation. Take OPT-175B where $d$ is 12288 as an example. HNSW~\cite{malkov2018efficient} requires more than 10ms, and FAISS~\cite{johnson2019billion} requires more than 4ms, while the MLP computation is only 0.2ms. The high dimensionality and complications of data structure implementation on GPU make the search time longer than the MLP computation. Therefore, we choose a neural network classifier as our near-neighbor search method to exploit the fast matrix multiplication on GPU. For each MLP block, we train a small two-layer fully connected network to predict contextual sparsity. Collecting training data is straightforward because we know the contextual sparsity using dense computation. The training algorithm is summarized in Algorithm~\ref{alg:sparse_predictor_training}. The sparsified computation in $W^1$ has two steps: (1) Given $y$, the sparsity predictor $\mathsf{SP}_{M}$ predicts a set $S_M$ of important neurons in weights $W^1$. (2) Compute the sparsified MLP defined in Eq.~\eqref{eq:MLP_S_y}. Note here the sparsity in MLP is highly structured.

\begin{algorithm}
   \caption{Sparse Predictor Training}
   \label{alg:sparse_predictor_training}
\begin{algorithmic}
 \State \textbf{Input}: A pre-trained LLM block with parameter set $M$, token embedding set at block $M = \{ x_i \}_{i\in [N]}$, threshold $t$
\State \textbf{Sparse Predictor} ${\cal SP}$
\State ${\cal P}_+ \leftarrow \emptyset$, ${\cal P}_- \leftarrow \emptyset$ 
\For {$i=1 \to N$}
 \State ${\cal P}_+ \leftarrow {\cal P}_+ \cup \{(x_i, m_r ) ~|~ m_r \in M,  m_r(x_i) \geq t\}$
 \State ${\cal P}_- \leftarrow {\cal P}_- \cup \{(x_i, m_r ) ~|~ m_r \in M,  m_r(x_i) < t\}$
 \EndFor
\State ${\cal SP} \leftarrow \textsc{Train}( {\cal P}_{+}, {\cal P}_{-}, {\cal L})$ \Comment{${\cal L}$ is a loss function} 

\end{algorithmic}
\end{algorithm}

\subsection{Contextual Sparsity Prediction in Attention Blocks}
\label{sec:routing_attn}
Attention blocks take around 30\% I/Os in the generation. In this section, we describe how \name{} exploits contextual sparsity to speed up the Attention blocks.


\textbf{Challenge:} As discussed in Section~\ref{sec:sparse_obs}, only a few heads perform important computations for a given input token. Similar to the MLP blocks, a fast selection of attention heads without full computation is required to reduce end-to-end latency. Furthermore, one particular challenge of sparse prediction in attention blocks is attention's dependence on previous tokens. On the one hand, it is unclear whether the past token's key and value caches are needed for sparse prediction. On the other hand, it is unclear how to handle the missing KV cache of past tokens for the current token computation at the selected head.

\textbf{A Near-Neighbor Search Problem:} Head prediction can also be formulated as a near-neighbor search problem based on our understanding in Section~\ref{sec:obs_att_cluster}. Since each head is performing mean-shift clustering, after the first few layers, the current token embedding alone is sufficient for the prediction thanks to the token-mixing nature of the transformer. Therefore, the prediction can be based on the similarity between $y$ and head parameters. 

\textbf{Approach:} We design our attention sparse predictor to be the same architecture as the MLP sparse predictor. Each head is regarded as one class and a similar training process is used (Algorithm~\ref{alg:sparse_predictor_training}). Then, similar to how MLP prediction is performed, the attention sparsity predictor $\mathsf{SP}_A$ selects a set $S_A$ of heads $H_i$ (see Eq.~\eqref{eq:H_i_y}).  To address the problem of missing KV cache for a past token, we exploit the fact that the generation latency is I/O bounded while computation is essentially ``free". Specifically, for the predicted attention head of input $y$, we compute the corresponding keys, and values and store them in the KV cache. But we also save a copy of $y$ for all the other non-selected heads. Then during the future token generation, if there is missing KV cache in the selected heads, we could load stored token embeddings and compute the keys and values together. This requires almost minimal extra memory access (the main cost is loading the weight matrices).

\subsection{Reducing Overhead with Asynchronous Execution}
\label{sec:hide_overhead}
Sparse prediction overhead may easily increase the end-to-end latency rather than reduce it despite the reduction in FLOPs. Therefore, we introduce a look-ahead sparse prediction method, inspired by our observations in Section~\ref{sec:obs_slowly_changing}.

\textbf{Challenge:} Denote $y_l\in \R^{d}$ as the input to transformer layer $l$. We can write the computation at layer $l$ as
$
    \widetilde{y}_l \leftarrow \mathsf{MHA}^{l}(y_l),
    \widehat{y}_l \leftarrow \mathsf{MLP}^{l}(\widetilde{y}_l)
$.
With predictors $\mathsf{SP}_A^l$ and $\mathsf{SP}_M^l$, the computation at the transformer layer $l$ can be re-written as 
\begin{align*}
    S_A^l  \leftarrow \mathsf{SP}_A^{l}(y_l), \quad
    \widetilde{y}_l \leftarrow \mathsf{MHA}^{l}_{S_A^l}(y_l), \\
    S_M^l \leftarrow \mathsf{SP}_M^{l}(\widetilde{y}_l ), \quad
    \widehat{y}_l \leftarrow \mathsf{MLP}^{l}_{S_M^l}( \widetilde{y}_l)
\end{align*}
where set $S_A^l$ is the contextual sparsity for the Attention block, and set $S_M^l$ is the contextual sparsity for the MLP block at $l$-th layer.
Note that the computation at Attention and MLP blocks have to wait for the sparse predictor decision. This overhead potentially outweighs the saving from Attention and MLP blocks in terms of latency. 

\textbf{Approach:} In Section~\ref{sec:obs_slowly_changing}, we present the slowly evolving embedding phenomenon, which provides opportunities to relax the sequential computation to parallel computation. 
Along with the observation of low computation intensity during generation, we parallel the sparse prediction with the computation of each block ( See Figure~\ref{fig:workflow_main}).  The computation can be written as follows:
\begin{align*}
    \widetilde{y}_l \leftarrow \mathsf{MHA}^{l}_{S_A^l}(y_l), \quad
    \widehat{y}_l \leftarrow \mathsf{MLP}^{l}_{S_M^l}( \widetilde{y}_l ), \\
    S_{A}^{l+1} \leftarrow \mathsf{SP}_A^{l}(y_l), \quad
    S_M^{l+1}  \leftarrow \mathsf{SP}_M^{l}(y_l),
\end{align*}
We remark $S_{A}^{l+1}$ and $S_M^{l+1}$ can be computed in parallel with $\widetilde{y}_l$ or $\widehat{y}_l$, while the previous 4 steps are sequential.

\textbf{Theoretical guarantee:} The sparse predictor can make further cross-layer decisions because of the residual connection. We present an informal lemma statement regarding cross-layer prediction. It is well-known that ${\sf MaxIP}$ is equivalent to $\ell_2$ nearest neighbor search. For convenience, we use ${\sf MaxIP}$ here. We include more discussions and proofs in Section~\ref{sec:nearest_neighbor}.
\begin{lemma}[Informal]
Let $\epsilon \in (0,1)$.
Let 
$y_l$ be input at $l$-th layer.
Let $y_{l-1}$ be the input at $(l-1)$-th layer. Suppose that $\| y_l - y_{l-1} \|_2 \leq \epsilon$. For any parameters $c, \tau$ such that $\epsilon < O(c \tau)$. Then we can show that, solving ${\sf MaxIP}(c,\tau)$ is sufficient to solve ${\sf MaxIP}(0.99 c, \tau )$.  
\end{lemma}

\begin{figure*}[t]
  \centering
  \vspace{-2mm}
\subfigure[Language Modeling]{
    \includegraphics[width=0.248\textwidth]{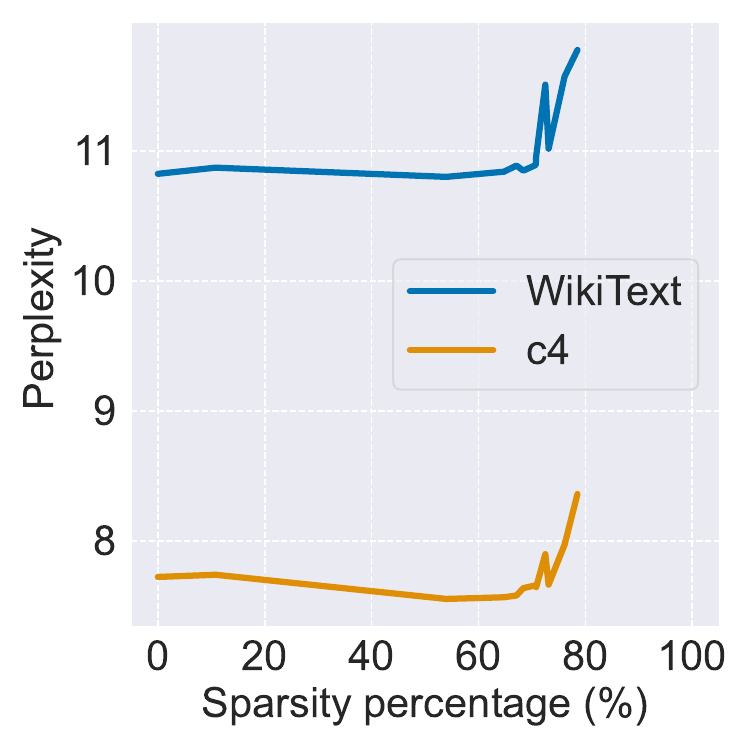}
  }
  \subfigure[Zero-Shot(Left). Five-Shot(Right)]{
    \includegraphics[width=0.61\textwidth]{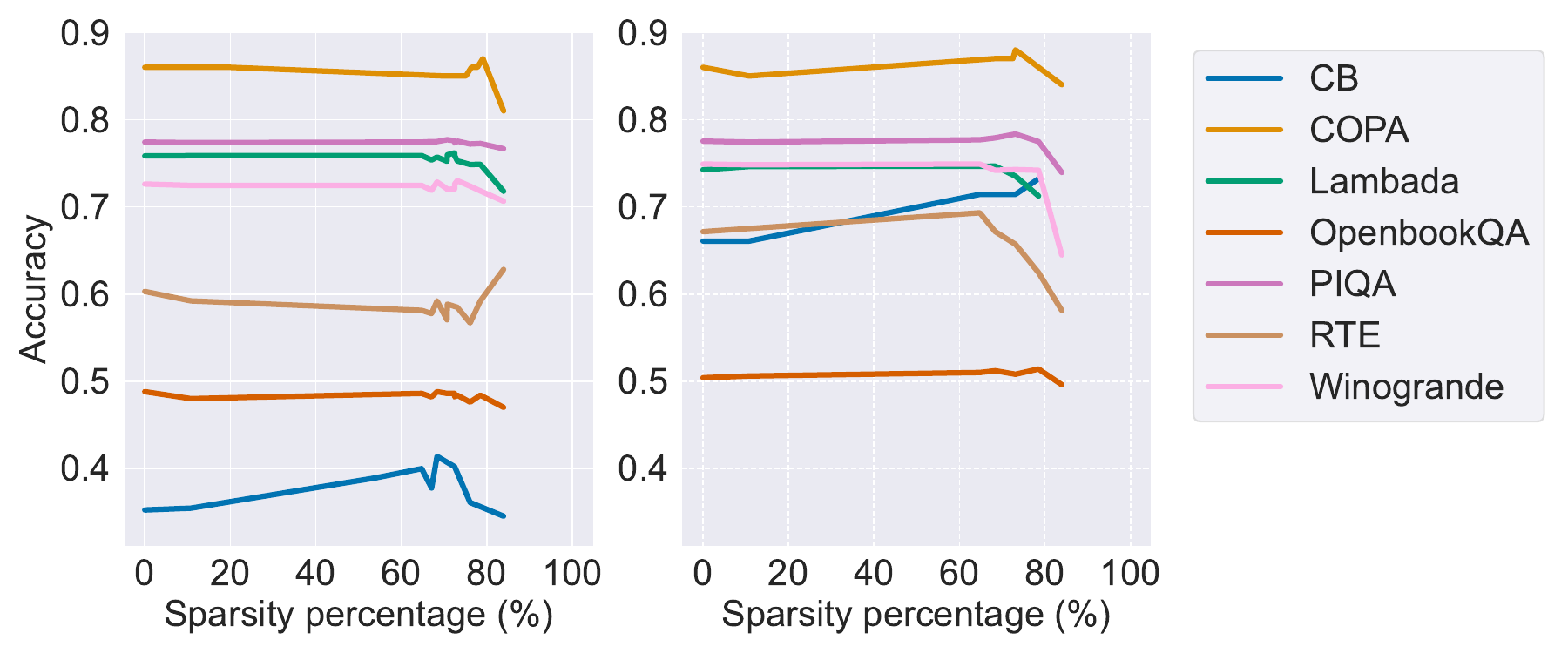}
  }
  \vspace{-3mm}
  \caption{\textbf{ Accuracy Trend for \name{}-OPT-175B}. This figure shows the accuracy of \name{}-OPT-175B on language modeling datasets and downstream tasks when we set different sparsity at test time. In general, \name{}-OPT-175B incurs no accuracy drop until 75\% sparsity.\vspace{-1em}}
   \label{exp:kshot} 
\end{figure*}

\subsection{Hardware-efficient Implementation}
\label{sec:sparse_matmul}

We describe how \name{} is implemented in a hardware-efficient manner to realize the
theoretical speedup of contextual sparsity.
Taking into account hardware characteristics leads to over 2$\times$ speedup
compared to an optimized dense model, and 4$\times$ faster than a standard sparse
implementation.

We highlight some hardware characteristics of GPUs:
\begin{itemize}[itemsep=0.0pt,topsep=0pt,leftmargin=*]
  \item Small-batch generation is bottlenecked by GPU memory I/Os~\citep{nvidia2022nvidia, ivanov2021data, dao2022flashattention}. This is
  because of low arithmetic intensity. For each element loaded from GPU memory,
  only a small number of floating point operations are performed.
  \item GPUs are block-oriented devices: loading a single byte of memory takes
  the same time as loading a block of memory around that same
  address~\cite{harris2013access}. The block size is
  usually 128 bytes for NVIDIA GPUs~\citep{cook2012cuda}.
\end{itemize}
These characteristics present some challenges in implementing contextual
sparsity.
However, they can be addressed with classical techniques in GPU programming.

\textbf{Kernel fusion:} A standard implementation of sparse matrix-vector
  multiply (e.g., in PyTorch) that separately indexes a subset of the matrix
  $W^1_{S_M}$ before multiplying with input $y$ would incur 3$\times$ the
  amount of memory I/Os. Therefore, to avoid such overhead, we fuse the indexing and the multiplication step. Specifically, we load a subset of
  $W^1_{S_M}$ to memory, along with $y$, perform the multiply, then
  write down the result.
  This fused implementation (in Triton~\citep{tillet2019triton}) yields up to
  4$\times$ speedup compared to a standard PyTorch implementation (Appendix~\ref{sec:mlp_attn_benchmarks}).
  
\textbf{Memory coalescing:} In the dense implementation, the weight matrices of two linear layers in MLP are stored as $(W^1)^\top$ and $W^2$ so that no extra transpose operation is needed. They are conventionally stored in row-major format. In the sparse implementation, it allows us to load $(W^1_{S_M})^\top$ optimally (the second dimension is contiguous in memory). However, for cases where we need to load $(W^2_{S_M})$, this format significantly slows down memory loading, as indices in $S_M$ point to non-contiguous memory. We simply store these matrices in column-major format (i.e., store $(W^2)^\top$ in row-major format), then use the same fused kernel above. Similarly, in attention blocks, we store attention output projection $W^O$ column-major format.

These two techniques (kernel fusion and memory-coalescing) make \name{}
hardware-efficient, yielding up to 2$\times$ speedup end-to-end compared to the state-of-the-art FasterTransformer (Section~\ref{sec:main_result}).






\section{Empirical Evaluation}
\label{sec:experitns}


\begin{figure}[]
  \centering
    \includegraphics[width=0.38\textwidth]{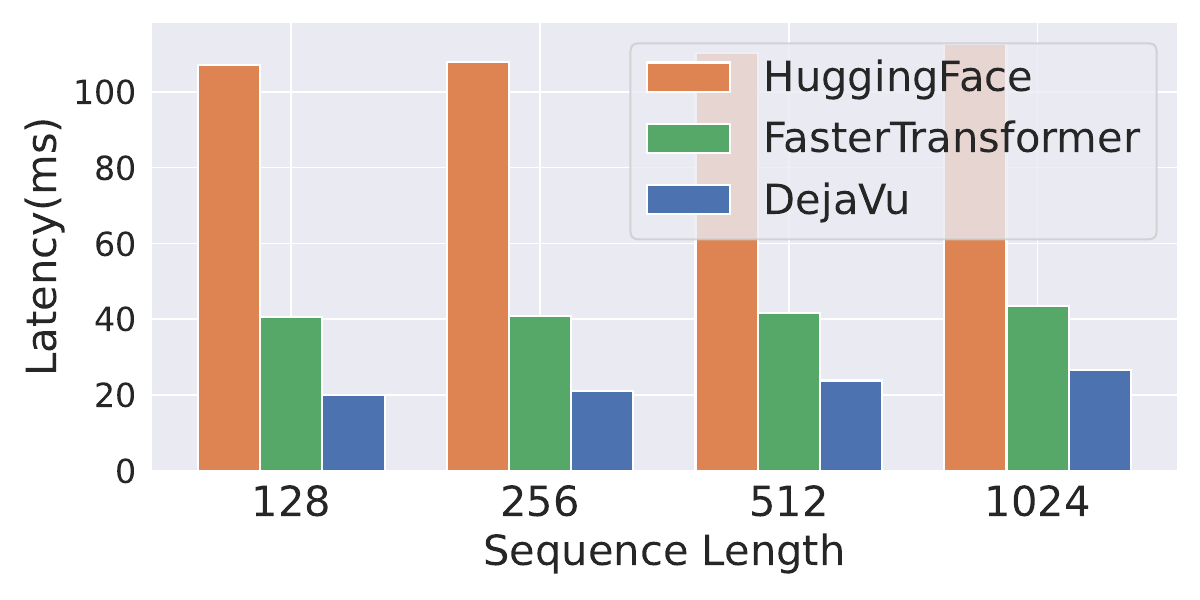}
    \vspace{-6mm}
  \caption{Average per-token latency (ms) with batch size 1 on 8 A100-80GB with NVLink
  when generating sequences with prompt lengths 128, 256, 512, and 1024, using FP16. \name{} speeds up generation by 1.8-2$\times$ compared to the
  state-of-the-art FT and by 4.8-6$\times$ compared to the widely used
  HF implementation.}
  \label{table:main_latency} 
     \vspace{-4mm}
\end{figure}

\begin{table*}[t]
\scriptsize
\centering
\vspace{-3mm}
\caption{Accuracy of zero-shot tasks and language modeling when sparsifying the MLP block and the Attention block separately. The sparsity is set at 85\% for MLP-block and 50\% for Attention-block. \name{} incurs no accuracy drop across the boards.}
\vspace{2mm}
\resizebox{0.8\linewidth}{!}{
\centering
\Huge
\begingroup
\setlength{\tabcolsep}{10pt}
\renewcommand{\arraystretch}{1.2}
\begin{tabular}{c||ccccccc|cc}
\specialrule{.15em}{.05em}{.05em}
 Model & CB & COPA & Lambada & OpenBookQA & PIQA  & RTE & Winogrande  &  Wikitext & C4	\\
\cline{1-10}
 OPT-175B & 0.3523 &  0.86   &  0.7584 & 0.446 & 0.8096 & 0.6029 &  0.7261  & 10.8221 & 7.7224 \\
 \cline{1-10}
 \name{}-MLP-OPT-175B & 0.3544 & 0.85 & 0.7619 & 0.446 & 0.8096 &   0.6065 &  0.7206 &  10.7988 & 7.7393  \\
 \cline{1-10}
 \name{}-Attention-OPT-175B & 0.3544 & 0.86  & 0.7586 & 0.4460 & 0.8063 &   0.5921 &  0.7245 & 10.8696 & 7.7393  \\
\specialrule{.15em}{.05em}{.05em}
\end{tabular}
\endgroup
}
\vspace{-2mm}
\label{table:exp-mlp-accruacy}
\end{table*}

In Section~\ref{sec:main_result}, we present the end-to-end results that show \name{} achieves over 2$\times$ reduction in token generation latency compared to the state-of-the-art FasterTransformer and over 6$\times$ compared to Hugging Face with no accuracy loss. In Section~\ref{sec:abalation_mlp_att}, we perform a list of ablation studies such as independent evaluation on the inference-time contextual sparsity of the MLP block and the Attention block (Details are presented in Section~\ref{sec:appendix-exp}). At last, we present the additional results to demonstrate the future possibility of sparsifying the entire LLMs via layer skipping in Section~\ref{sec:exp_skip_layer}.


\subsection{End-to-End Result}
\label{sec:main_result}

\textbf{Experiment Setting:}
We compare the accuracy of \name{}-OPT against the original OPT model on two language modeling datasets Wiki-Text~\cite{merity2016pointer} and C4~\cite{2019t5} and seven few-shot downstream tasks: CB~\cite{Marneffe2019TheCI}, COPA~\cite{gordon-etal-2012-semeval}, Lambada~\cite{radford2019language}, OpenBookQA~\cite{OpenBookQA2018}, PIQA~\cite{Bisk2020}, RTE~\cite{giampiccolo-etal-2007-third}, Winogrande~\cite{ai2:winogrande}. 
We use lm-eval-harness~\cite{eval-harness} for zero-shot and five-shot tasks.  We collect training data for the sparsity predictor using 500 random data points from the C4 training dataset. Our experiments are conducted on NVIDIA A100 80GB GPU servers.

\textbf{No accuracy drop until 75\% sparsity: } In Figure~\ref{exp:kshot}, we present \name{}-OPT-175B's accuracy trend.  In a zero-shot setting, the average accuracy across tasks does not drop until 75\% sparsity. A similar trend can be observed for the five-shot setting, which verifies the model's ability for in-context learning.  This result is exceptionally encouraging given our observation in Figure~\ref{fig:sparsity-175}, where we could impose 85\% sparsity when allowed full computation. 



\textbf{Over 2$\times$ latency reduction: } 
Figure~\ref{table:main_latency} presents the latency speed-up for the token generation with OPT-175B at batch size 1, where \name{} achieves the best performance. At around 75\% sparsity, \name{} speeds up generation by 1.8-2$\times$ compared to the state-of-the-art FasterTransformers (FT)\footnote{http://github.com/NVIDIA/FasterTransformer} and by 4.8-6$\times$ to Hugging Face (HF) implementation\footnote{http://github.com/huggingface/transformers}.



\vspace{-2mm}
\subsection{Ablation Results}
\label{sec:abalation_mlp_att}

\textbf{Contextual Sparsity for Larger Batches:} Although this paper focuses on latency-sensitive settings, we demonstrate that \name{} generalizes to larger batches. we present the Union contextual sparsity (fraction of neurons/heads that are not used by any of the inputs in the batch) of different batches sizes for MLP and Attention blocks, respectively, in Figure~\ref{main:exp_sparsity_batch} and \ref{appendix:exp_sparsity_batch}. The union operation is essential to realize a fast sparse GEMM. Surprisingly the number of MLP neurons and Attention heads that \name{} activated does not grow linearly with the batch size. This suggests a power law distribution rather than a uniform distribution of parameter access from all input examples. This provides an opportunity for potentially extending Dejavu to the high-throughout setting. For example, we can first pre-process the inputs and batch similar inputs to enjoy a higher level of union contextual sparsity.

\textbf{Contextual sparsity on MLP blocks:} We study the contextual sparsification of the MLP block in OPT-175B. We leave the Attention block as dense computation. Table~\ref{table:exp-mlp-accruacy} shows the model performance at 85\% sparsity. The MLP sparse predictor introduces no accuracy loss on both zero-shot tasks and language modeling. In the training of the MLP sparse predictor, we observe that the sparse predictor achieves high validation accuracy. The shallow layer seems easier to model because the predictor has validation accuracy over 99\% in the shallow layers and drops to around 93\% in the ending layers. 

\textbf{Contextual sparsity on attention blocks:} In this section, we study the sparse predictor for the Attention block on OPT-175B and leave the MLP block as dense computation. Table~\ref{table:exp-mlp-accruacy} displays the test accuracy on zero-shot tasks and perplexity on the language modeling datasets. In summary, the Attention sparse predictor introduces no accuracy loss at around 50\% sparsity. During the training of the Attention sparse predictor, we observe different trends compared to the MLP sparse predictor. The validation accuracy is around 93\% in the middle layers and near 99\% in the shallow and deep layers.

\textbf{Contextual Sparsity on Smaller Models:} Our main experiments focus on OPT-175B. Here, we verify \name{}'s effectiveness on a smaller model, specifically OPT-66B. In Table~\ref{table:exp-66b-accuracy}, we summarize the accuracy on zero-shot task at $50\%$ 
sparsity. Similar to \name{}-OPT-175B, we notice no accuracy loss.

\textbf{Contextual Sparsity on Other Models:} We expand the evaluation to another model family. In Table~\ref{table:exp-bloom-accuracy}, we summarize the accuracy at attention sparsity 50\% and MLP sparsity 30\%. Similar to OPT family, we notice no accuracy loss. The lower sparsity level in MLP is due to the difference in activation function.


\begin{table}[t]
\vspace{-4mm}
\scriptsize
\centering
\caption{\name{}-OPT66B on zero-shot downstream task.}
\vspace{2mm}
\resizebox{\linewidth}{!}{
\centering
\Huge
\begingroup
\setlength{\tabcolsep}{10pt}
\renewcommand{\arraystretch}{1.4}
\begin{tabular}{c||ccccccc}
\specialrule{.15em}{.05em}{.05em}
 Model & CB & COPA  & Lambada & OpenBookQA & PIQA  & RTE & Winogrande  	\\
\cline{1-8}
 OPT-66B& 0.3928 &  0.87  & 0.7508 & 0.426 & 0.7921 &  0.6028 & 0.6890 	\\
\cline{1-8}
\name{}-OPT-66B & 0.4285 & 0.87  & 0.7458 & 0.434 & 0.7933 & 0.5884 &  0.6898 \\
\cline{1-8}
\specialrule{.15em}{.05em}{.05em}
\end{tabular}
\endgroup
}
\vspace{-3mm}
\label{table:exp-66b-accuracy}
\end{table}

\begin{table}[t]
\vspace{-2mm}
\scriptsize
\centering
\caption{\name{}-BLOOM on zero-shot downstream task.}
\vspace{2mm}
\resizebox{\linewidth}{!}{
\centering
\Huge
\begingroup
\setlength{\tabcolsep}{10pt}
\renewcommand{\arraystretch}{1.4}
\begin{tabular}{c||ccccccc}
\specialrule{.15em}{.05em}{.05em}
 & CB    & COPA & OpenBookQA & PIQA  & RTE   & Winogrande & Lambada\\
 \cline{1-8}
BLOOM & 0.455 & 0.8  & 0448       & 0.79  & 0.617 & 0.704 & 0.677 \\
\cline{1-8}
Dejavu-BLOOM & 0.448 & 0.8  & 0.44       & 0.787 & 0.606 & 0.710    & 0.675      \\
\specialrule{.15em}{.05em}{.05em}
\end{tabular}
\endgroup
}
\vspace{-3mm}
\label{table:exp-bloom-accuracy}
\end{table}

\begin{figure}[t]
  \centering
    \includegraphics[width=0.36\textwidth]{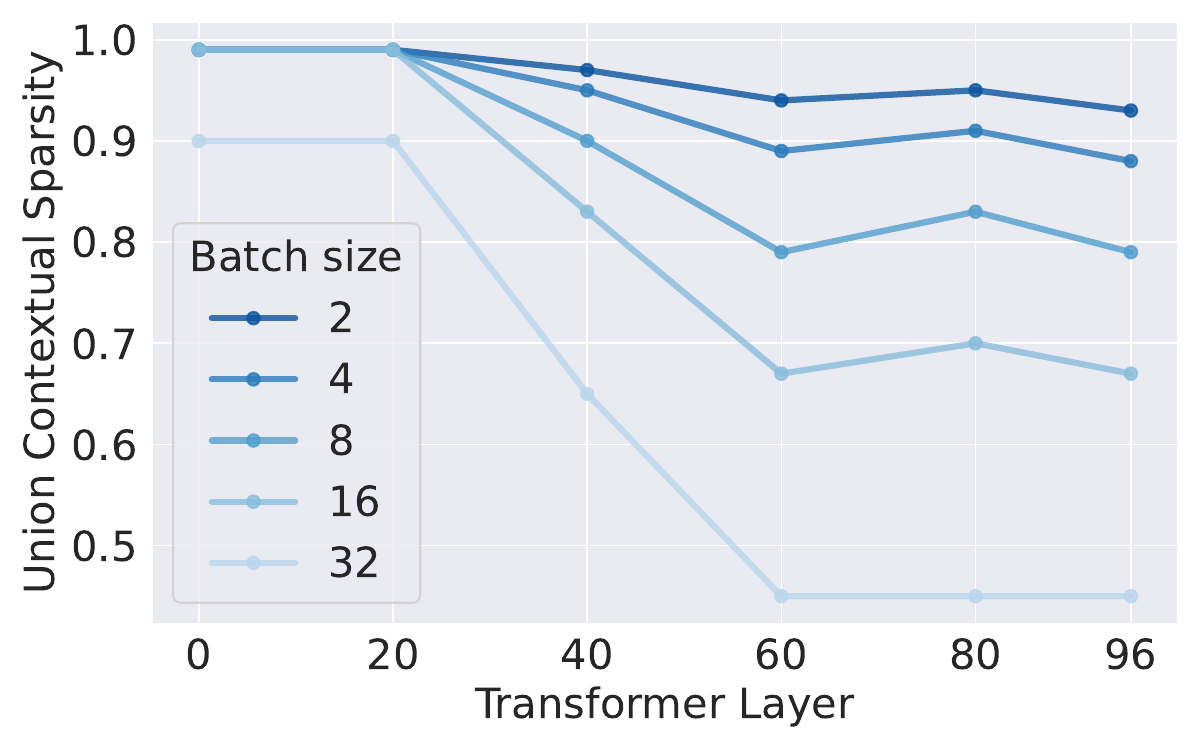}
   \vspace{-2mm}
  \caption{Union contextual sparsity with larger batch size.}
  \label{main:exp_sparsity_batch}
  \vspace{-2mm}
\end{figure}

\textbf{Non-Contextual Sparsity: } As we mentioned in Section~\ref{sec:introduction}, one could predict sparsity without contextual information. For non-contextual sparsity, we rely on the original embedding at the input layer. At every block, we first pass the original embedding to record a subset of parameters yielding a large norm. In the second pass, the embedding at every layer only uses the recorded subset. As shown in Figure~\ref{fig:contextual-static}, non-contextual prediction is not sufficient and leads to accuracy losses even at 50\% sparsity.  This result verifies our design choices of relying on the activation at every layer as input to make contextual sparsity predictions.

\textbf{Compatibility with Quantization:} Quantization is another promising direction for efficient language models. We investigate the possibility of combining contextual sparsity with quantization techniques. For \name{}-OPT-175B, we set the entire model sparsity at 75\%. For quantization, we apply 4-bit quantization on model weights (W4A16). As shown in Table~\ref{table:with-quantization}, the combination of quantization and \name{} almost always achieves better accuracy than \name{}  or quantization alone. This suggests that the approximation errors from these two directions do not get compounded. 
\begin{table}[t]
\vspace{-2mm}
\scriptsize
\centering
\caption{\name{}-OPT-175B with 4-bit quantization.}
\vspace{2mm}
\resizebox{\linewidth}{!}{
\centering
\Huge
\begingroup
\setlength{\tabcolsep}{10pt}
\renewcommand{\arraystretch}{1.4}
\begin{tabular}{c||ccccccc}
\specialrule{.15em}{.05em}{.05em}
 & CB    & COPA & OpenBookQA & PIQA  & RTE   & Winogrande & Lambada\\
 \cline{1-8}
OPT-175B                & 0.352 & 0.86 & 0.446      & 0.809 & 0.602 & 0.726      & 0.758           \\
\cline{1-8}
Dejavu-OPT-175B         & 0.402 & 0.85 & 0.450      & 0.802 & 0.592 & 0.726      & 0.753           \\
\cline{1-8}
OPT-175B + W4A16        & 0.356 & 0.85 & 0.44       & 0.806 & 0.574 & 0.714      & 0.757           \\
\cline{1-8}
Dejavu-OPT-175B + W4A16 & 0.365 & 0.86 & 0.452      & 0.805 & 0.592 & 0.726      & 0.754          \\
\specialrule{.15em}{.05em}{.05em}
\end{tabular}
\endgroup
}
\vspace{-3mm}
\label{table:with-quantization}
\end{table}

\section{Conclusion}

Our main goal is to make LLM inference efficient so that their powerful in-context learning abilities can be used in more application domains.
We observe that contextual sparsity can be accurately predicted with lightweight learning-based algorithms. This motivated us to design \name{} that uses asynchronous lookahead predictors and hardware-efficient sparsity to speed up LLM inference in wall-clock time. Our encouraging empirical results validate that contextual sparsity can reduce inference latency by over 2$\times$ compared to the state-of-the-art FasterTransformer without model quality drops. 
Our method is a step towards making LLMs more accessible to the general community, which could unlock exciting new AI applications.

\section*{Acknowledgements}
We would like to thank Ryan Spring, Laurel Orr, Guangxuan Xiao, Eric Han, Xun Huang, Daniel Y. Fu, Benjamin Spector, Ruan Silva, Diana Liskovich, and the anonymous reviewers for helpful discussions and feedback. We acknowledge the generous support by Together Computer, which enabled the necessary partial computations in this work.
\newpage
\bibliography{ref}
\bibliographystyle{icml2023}

\newpage
\appendix
\onecolumn



\textbf{Contents:} In Section \ref{appendix:related_work}, we present an extended discussion on LLM inference and related works. In Section \ref{sec:appendix-obs}, we provide more observation plots for slowly changing activation and further observation on the possibility of sparsifying LLMs via layer skipping. In Section \ref{sec:appendix-exp}, we provide experiment details. In Section~\ref{appendix:method}, we demonstrate implementation details. In Section \ref{sec:mlp_attn_benchmarks}, we provide detailed benchmarks regarding our implementation. 
In Section \ref{sec:notation_definition}, we define some basic notations and definitions.
In Section \ref{sec:subspace_embedding}, we define subspace embedding and show the norm preserving.
In Section \ref{sec:distances_angles}, we introduce  distances, angles, and inner product.
In Section \ref{sec:function_approx}, we provide the distance between different functions.
In Section \ref{sec:nearest_neighbor}, we provide the Near-neighbor Search data structure.
 In Section \ref{sec:clustering understanding}, we discuss self-attention as a clustering algorithm in depth.

\section{Related Work}
\label{appendix:related_work}
\textbf{Generative LLM inference.} Taking OPT-175B as an example, assume 6 A100 80GB PCIe, based on the hardware specifications, we compare two main phases of inference time LLM, namely prompting and token generation in Table~\ref{table:obs_break_down_stage}, and two major components, namely Multi-Head-Attention block and MLP block in Table~\ref{table:obs_break_down_block}. In practice, the token generation phase usually dominates the end-to-end test latency due to IO latency. Generating only two tokens is about the same latency as prompting. Further, during token generation, the MLP block is 2 $\times$ more expensive in both FLOPs and IO access. The hardware is often at low utilization because memory reads and writes are more limited on modern hardware than tensor core computation.

 Given the rapid development of LLM, there is an emergence of systems that are specialized for LLM inference, such as Faster Transformer~\cite{nvidiaft},  Orca~\cite{yu2022orca}, LightSeq~\cite{wang2021lightseq}, PaLM inference~\cite{pope2022efficiently}, TurboTransformers~\cite{fang2021turbotransformers}, and Deepspeed-Inference~\cite{aminabadi2022deepspeed}. In practice, the token generation phase usually dominates the end-to-end inference time. Although the state-of-the-art systems introduce some helpful system optimizations for speedup, there is a lack of careful algorithm and system co-design to unleash the full potential of hardware efficiency during the LLM inference computation.   

\textbf{Near-neighbor Search for Efficient Deep Neural Networks.} Near-neighbor Search is a well-studied problem with wide applications in recommendation system~\cite{ xue2017deep,hall2015fast}, question answering~\cite{boytsov2016off,seo2019real, chang2020pre} and natural language processing~\cite{bengio2003neural,lee2015reasoning}. There has been a line of work using Near-neighbor Search techniques such as Locality-sensitive hashing~\cite{gionis1999similarity} and Graph-based indexing~\cite{malkov2014approximate} for efficient deep neural network training or inference~\cite{zhang2018navigating,chen2019fast,chen2020slide,kkl20,chen2021mongoose,chen2021scatterbrain,liu2022halos}.

\noindent \textbf{Quantization, pruning, distillation for LLM inference.} Various system relaxations have been studied for decades for model inference in machine learning. For example, quantization~\cite{han2015deep, jacob2018quantization,nagel2019data,zhao2019improving}, pruning~\cite{molchanov2016pruning,liu2018rethinking,he2019filter,hoefler2021sparsity}, and distillation~\cite{hinton2015distilling,cho2019efficacy,tang2019distilling,touvron2021training}  have been applied to speed up the inference of the machine learning model. Active research has recently attempted to apply such techniques in LLM inference. For example, zeroQuant~\cite{yao2022zeroquant} and nuQmm~\cite{park2022nuqmm} implement customized CUDA kernels to support tenor-wise or group-wise quantization for LLM inference; LLM.int8 \cite{dettmers2022llm} adopts a mixed \texttt{INT8/FP16} computation to diminish the influence of activation outliers; SmoothQuant~\cite{xiao2022smoothquant} enables efficient 8-bit weight and activation for LLM inference; GPTQ~\cite{frantar2022gptq} adopts a one-shot weight quantization method based on approximate second-order information for accuracy and efficiency; SparseGPT~\cite{frantar2023massive} introduces an approximate sparse regression solver to enable the sparsity in LLM inference; \cite{bansal2022rethinking} has reported that a small set of attention heads can perform primitive induction operations associated with in-context learning, and use this property to prune LLM for acceleration. 



\noindent \textbf{Residual connections in neural networks.} Residual connection shows great advantages for neural network generalization, it provides additional paths for activations to reach the latter parts of the neural network by skipping some layers~\cite{he2016deep}. The advancement of residual connections can be viewed as ensembles of multiple shallow neural networks~\cite{veit2016residual}. Plenty of active research has discussed the effectiveness of residual connections~\cite{balduzzi2017shattered,bello2021revisiting,allen2019can,frei2019algorithm}. However, as far as we know, there is no former work that leverages the property of residual connections to improve the efficiency of LLM inference.

\section{Additional Observation on Slowly Changing Observation}
\label{sec:appendix-obs}

First, we present more plots on the cosine similarity between representations. Figure~\ref{appendix:observation_changing} plots the cosine similarity between activation across layers on OPT family. It is evident that similarity is high for the larger models.

There are two residual connections inside a transformer layer, one around the attention block, and the other one around the MLP block. The residual connection can be written as $X + F(X)$, where $F$ is either the Multi-Head Attention or two MLP Layer. Figure~\ref{appendix:residual} plots the cosine similarity between $X$ and $X + F(X)$, which is close to 1.0, and the cosine similarity between $X$ and $F(X)$, which is close to 0.0. This happens because $\|X\|$ is significantly greater than $\|F(X)\|$, shown in the purple. In the first layer, $\|F(X)\|$ is larger, which explains the low cosine similarity. The magnitude of the $L2$ norm is different across models, however, we observe a similar trend with models of different sizes. There exists a normalization layer before $F(X)$ and the layer normalization scale $\|X\|$ to a consistent magnitude across layers (e.g. 85 for OPT-30B, 110 for OPT175B), but not necessarily scale down $\|X\|$.

\begin{figure}[t]
  \centering
     \subfigure[OPT-1.3B]{
    \includegraphics[width=0.30\textwidth]{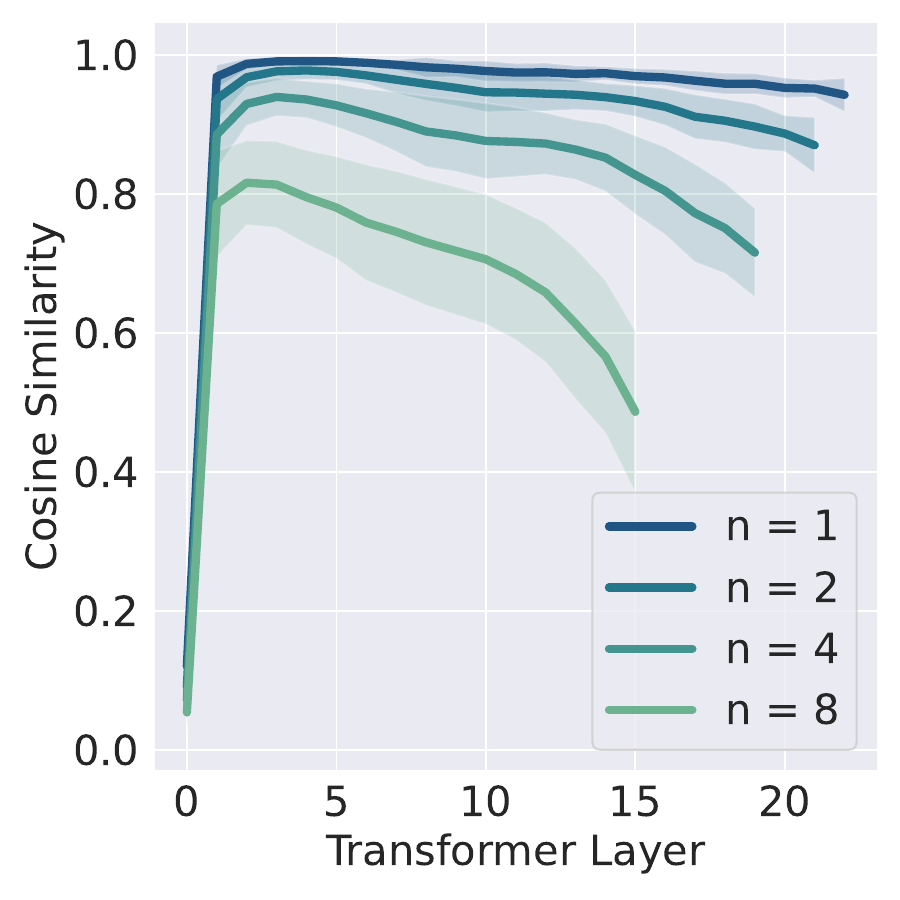}
    }
   \subfigure[OPT-6.7B]{
    \includegraphics[width=0.30\textwidth]{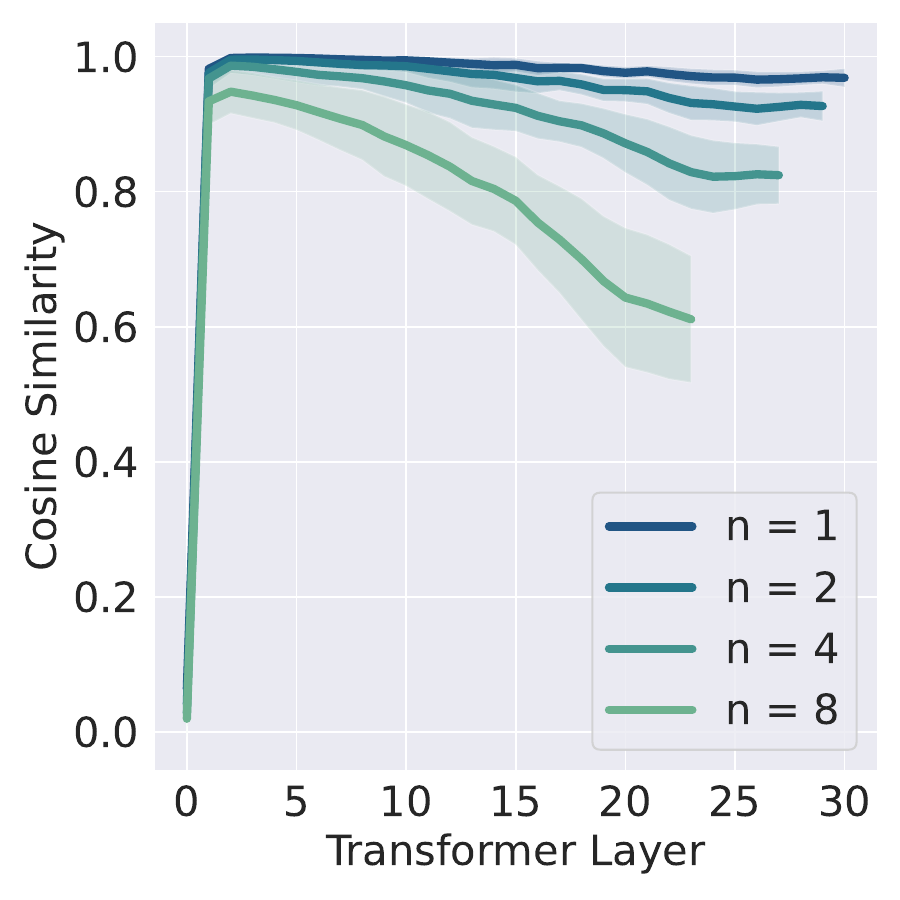}
    }
 \subfigure[OPT-13B]{
    \includegraphics[width=0.30\textwidth]{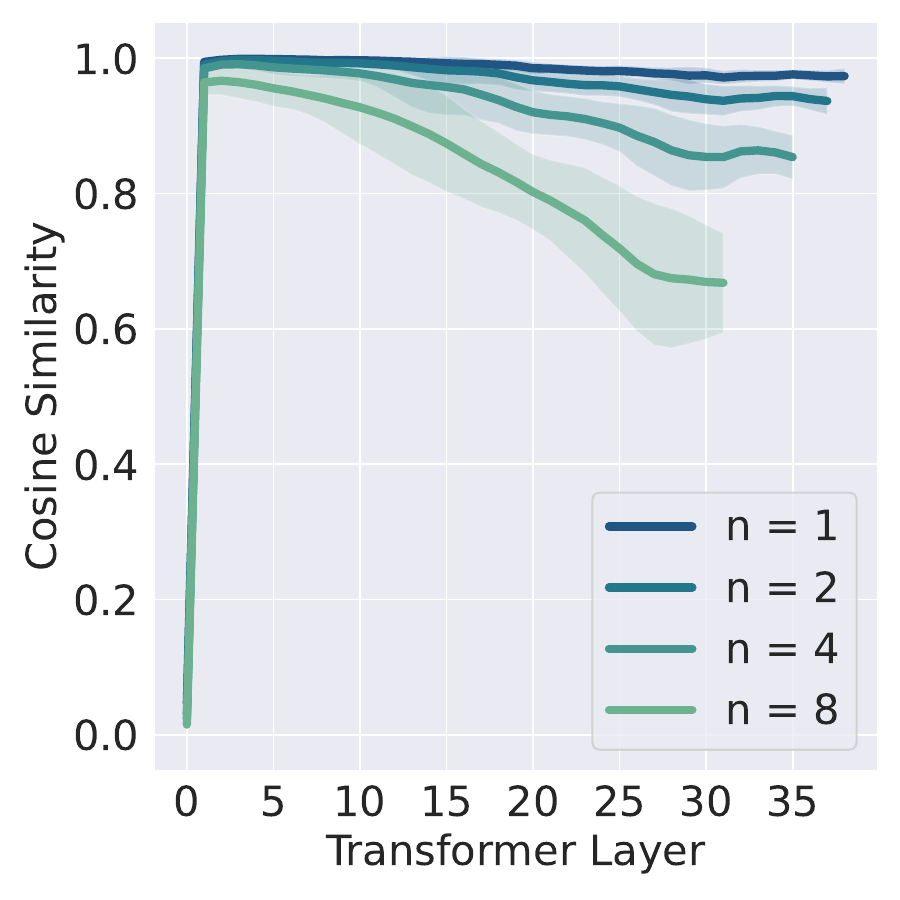}
    }
  \subfigure[OPT-30B]{
    \includegraphics[width=0.30\textwidth]{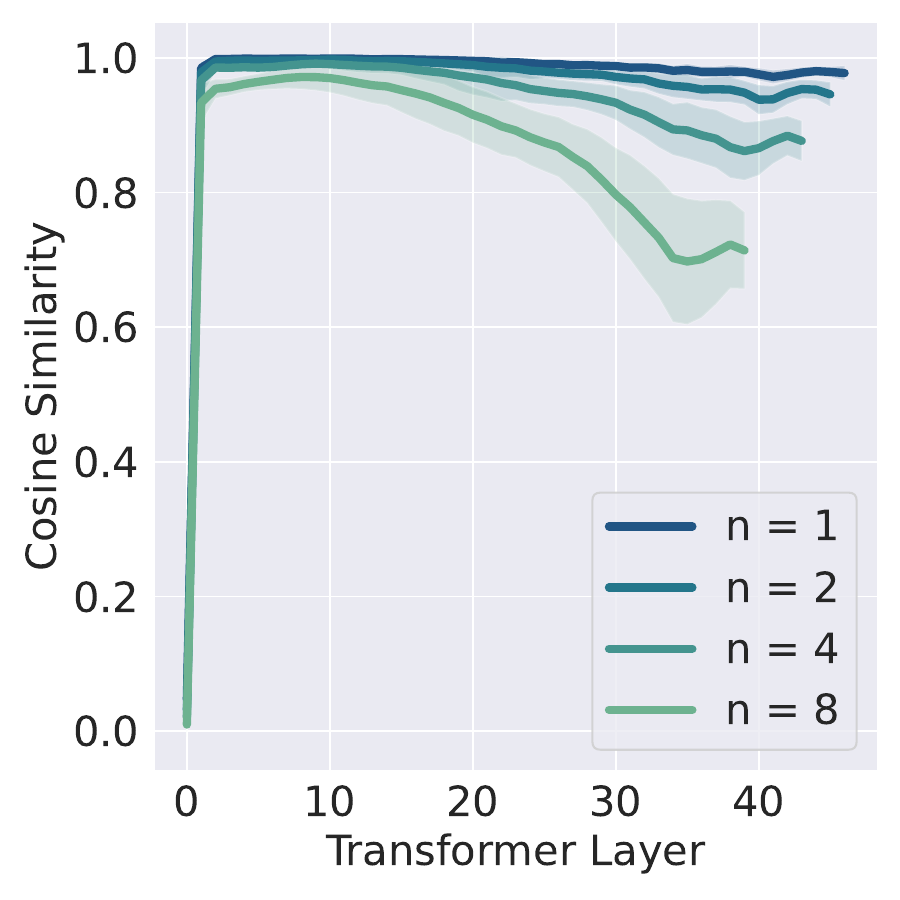}
  } 
    \subfigure[OPT-66B]{
    \includegraphics[width=0.30\textwidth]{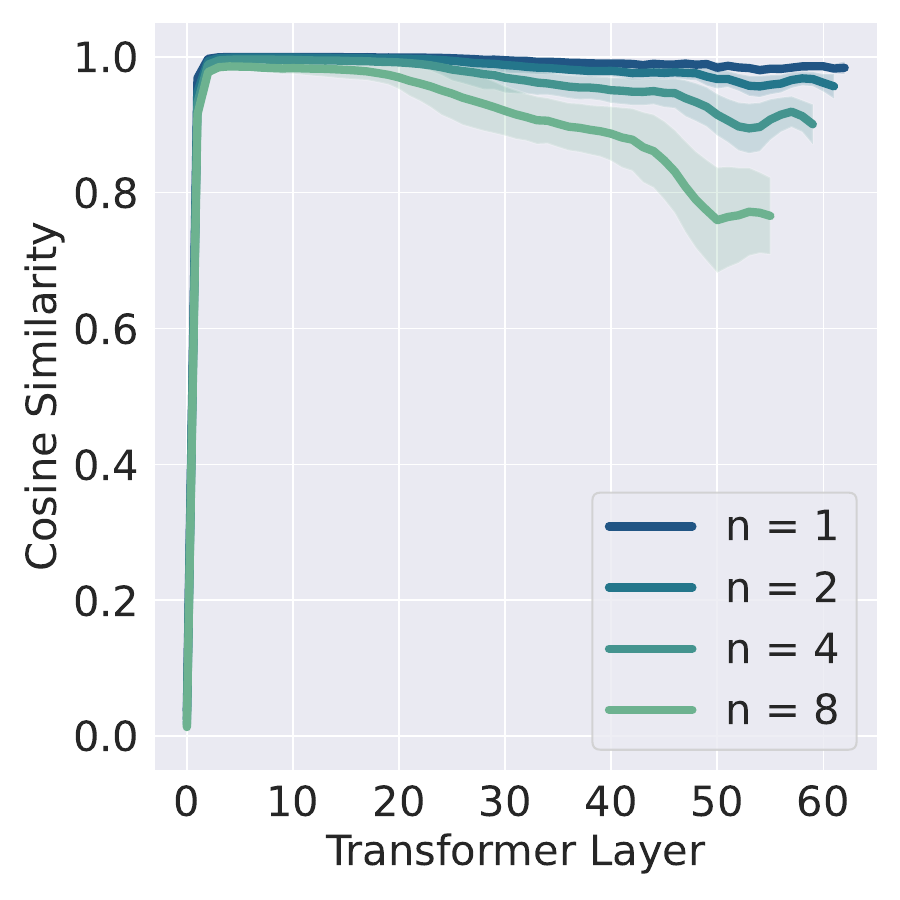}
  } 
    \subfigure[OPT-175B]{
    \hspace{1mm}\includegraphics[width=0.30\textwidth]{figure/observation/175b_between_layer_cos.pdf}
  } 
  
  \caption{ Cosine similarity between layer $l$ and layer $l+1$ for various model.}
  \label{appendix:observation_changing} 
\end{figure}

\begin{figure}[t]
  \centering 
    \subfigure[OPT-1.3b]{
    \includegraphics[width=0.30\textwidth]{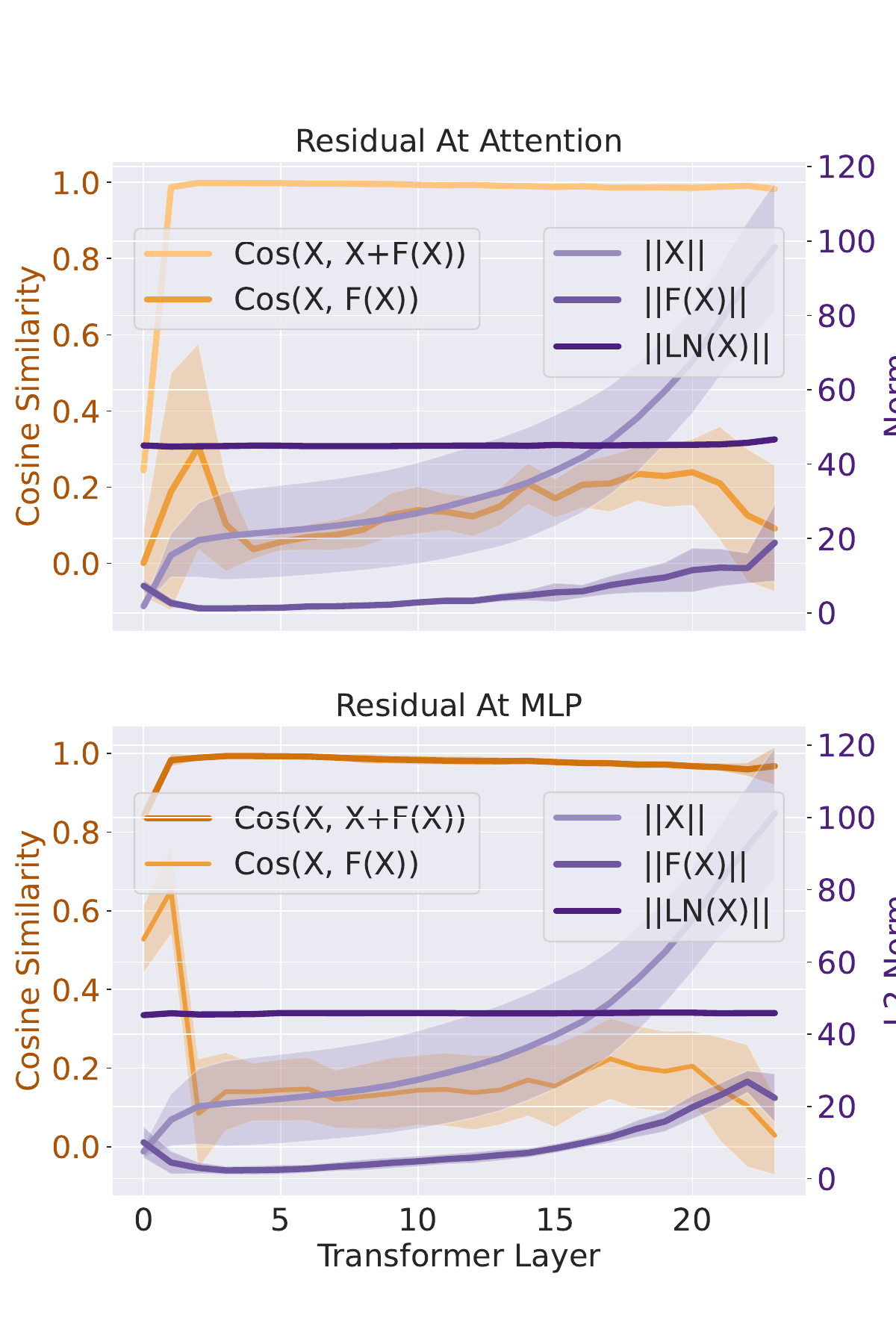}
    }
   \subfigure[OPT-6.7b]{
    \includegraphics[width=0.30\textwidth]{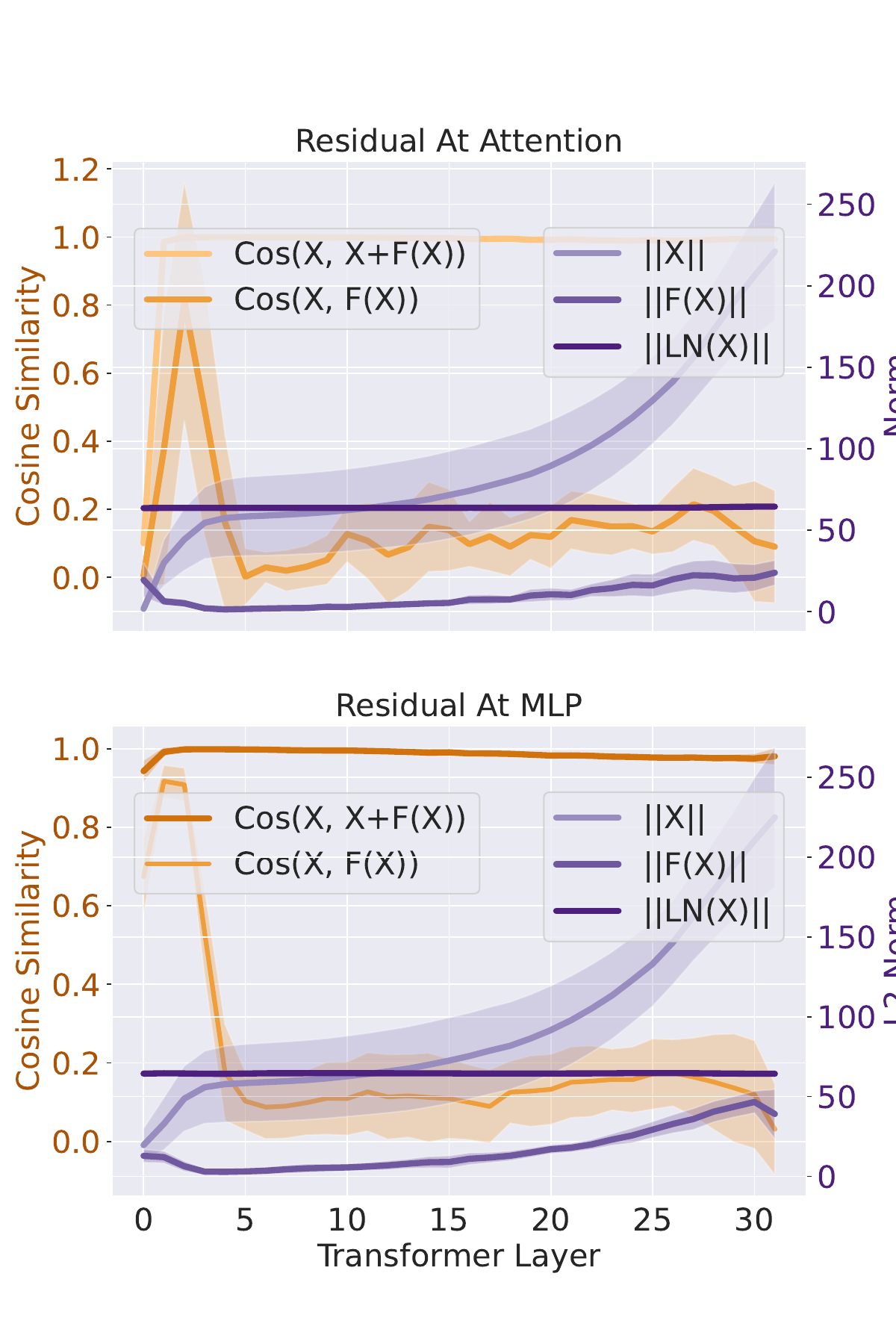}
    }
 \subfigure[OPT-13B]{
    \includegraphics[width=0.30\textwidth]{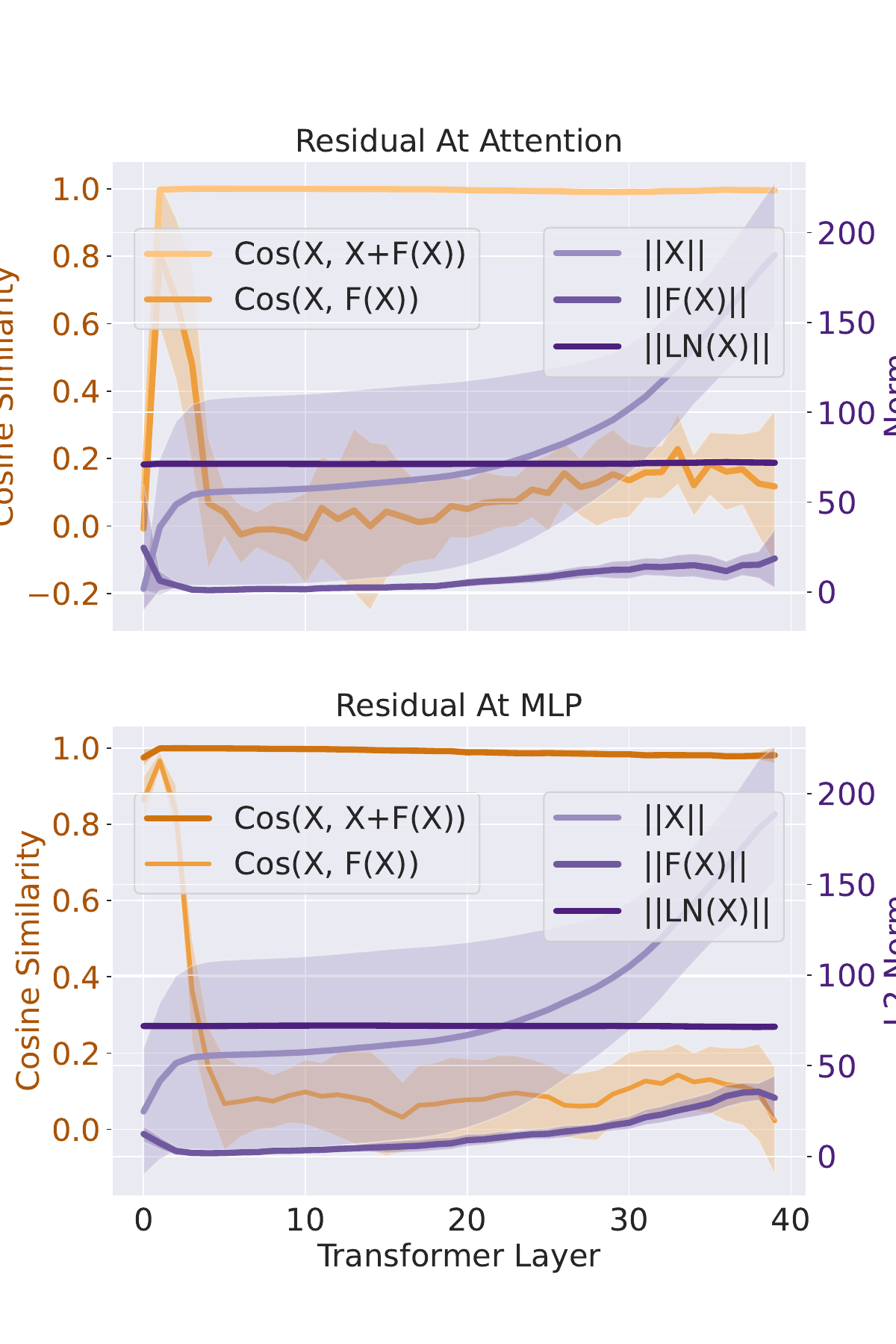}
    }
  \subfigure[OPT-30B]{
    \includegraphics[width=0.30\textwidth]{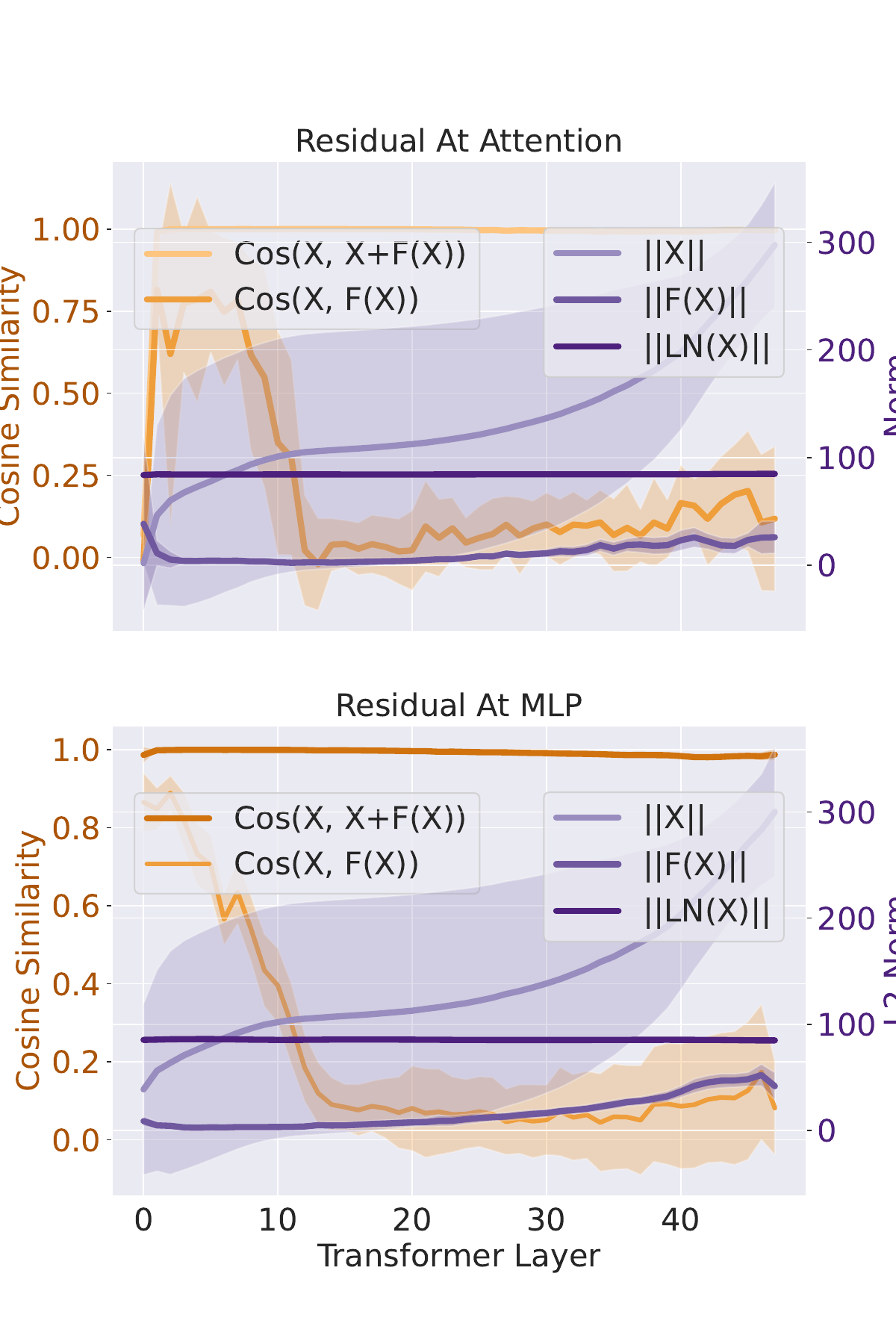}
  } 
    \subfigure[OPT-66B]{
    \includegraphics[width=0.30\textwidth]{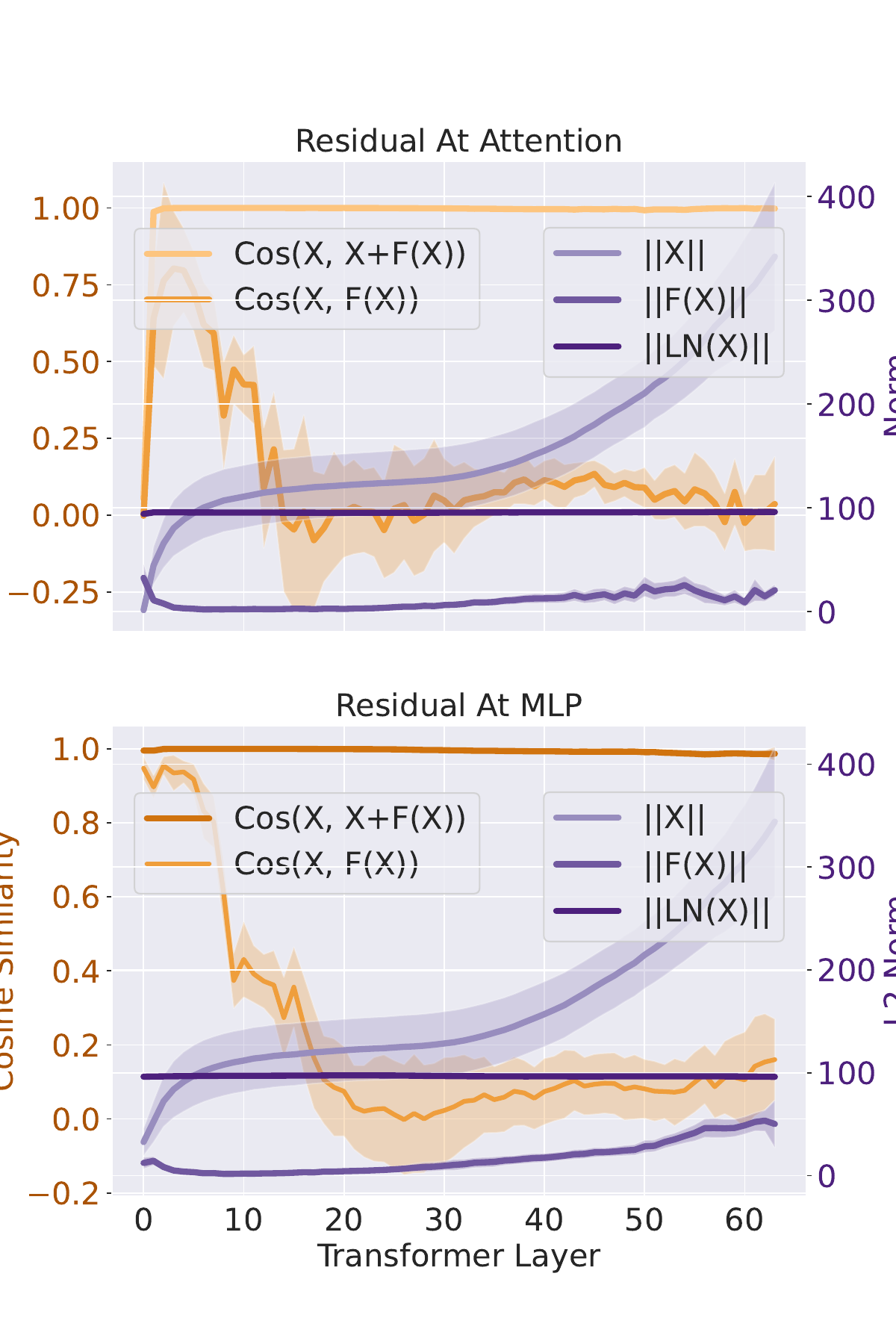}
  } 
    \subfigure[OPT-175B]{
    \hspace{1mm}\includegraphics[width=0.30\textwidth]{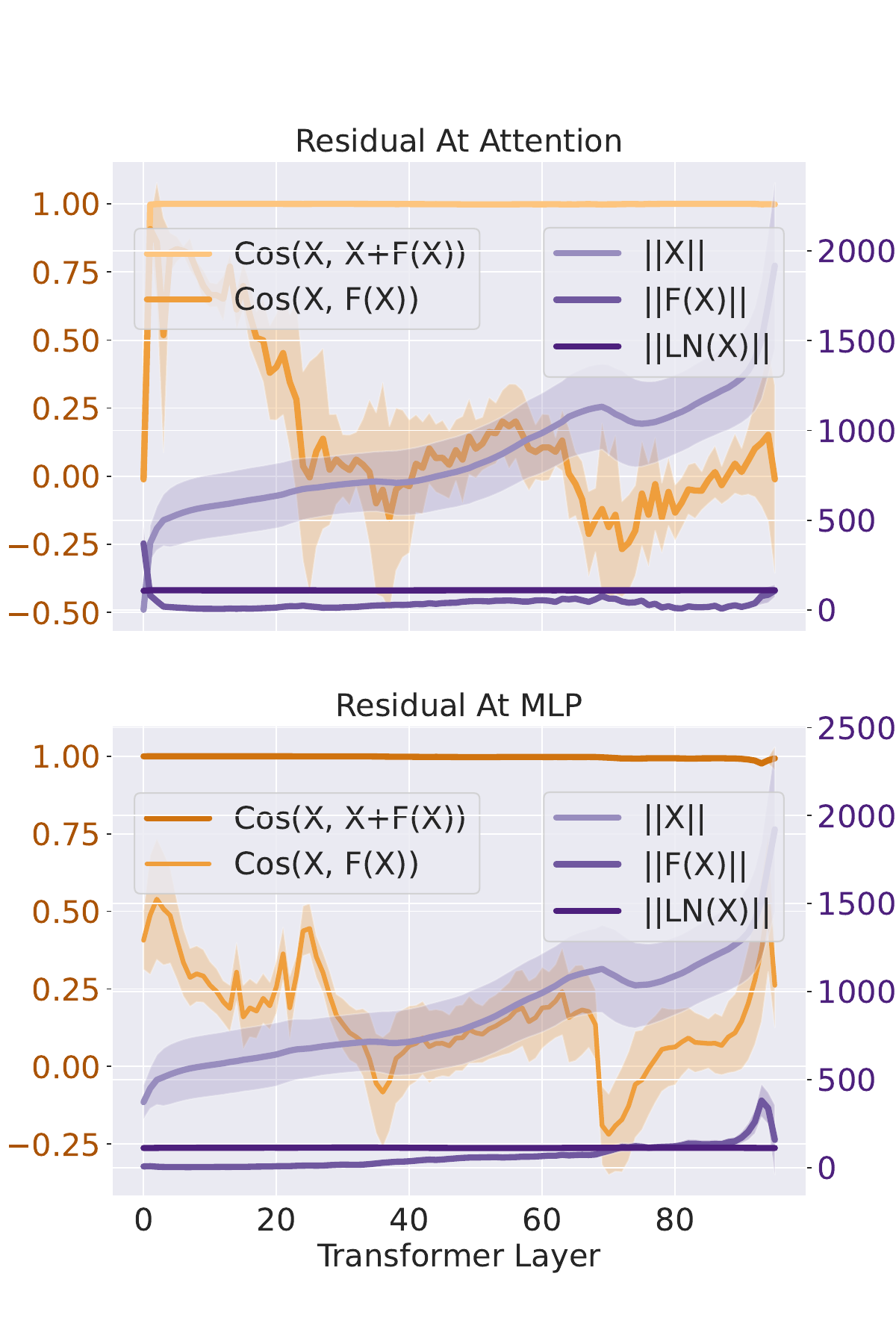}
  } 
  \caption{ Cosine similarity between $X$ and $F(X)$, and the cosine similarity between $X$ and $X'$ in orange color. $L2$ norm of $X$ and $F(X)$ and $X$ after layer normalization in purple on the right. Except on the first layer, $\|X\|$ is significantly higher than $\|F(X)\|$. $\|F(X)\|$ is higher at the first layer, which corresponds to the low cosine similarity at the first layer.}
  \label{appendix:residual} 
\end{figure}

\section{Additional Experiment Detail}
\label{sec:appendix-exp}
\subsection{Large Batch Size}
To help understand where the speed-up comes from when batch size is greater than 1, we present the Union Contextual Sparsity (fraction of neurons/heads that are not used by any of the inputs in the batch) of different batches sizes for MLP and Attention blocks, respectively, in Figure~\ref{appendix:exp_sparsity_batch}. Union Contextual Sparsity is calculated as 1.0 - the union of activated MLP neurons or Attention heads in the batch / total neurons or heads. The union operation is essential to realize a fast sparse GEMM. 

Surprisingly the number of MLP neurons/Attention heads that \name{} activated does not grow linearly with the batch size. This suggests a power law distribution rather than a uniform distribution of parameter access from all input examples. Further, a larger batch size can easily lead to out-of-memory for long sequence settings due to the limited GPU memory, the giant large model size, and the stored KV cache. For example, the total GPU memory of 8 80GB A100 is 640GB. Model parameters are around 350GB for OPT175B. The KV cache for a batch size 32 with a sequence longer than 1920 tokens has already filled up the GPU memory. 

\begin{figure}[h]
  \centering
     \subfigure[MLP]{
    \includegraphics[width=0.40\textwidth]{figure/experiment/batch_sparsity_mlp.pdf}
    }
   \subfigure[Attention]{
    \includegraphics[width=0.40\textwidth]{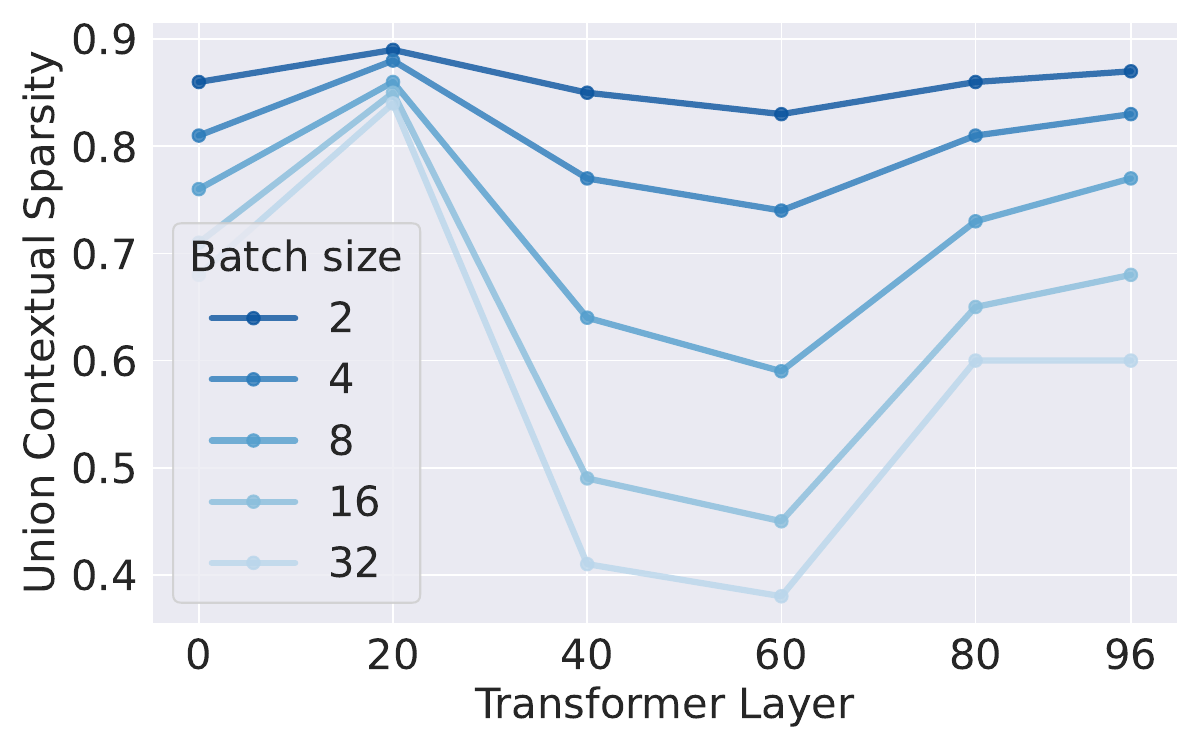}
    }
  \caption{Union contextual sparsity with larger batch size.}
  \label{appendix:exp_sparsity_batch} 
\end{figure}
\subsection{Near Neighbor classifier}
In the \name{} framework, any near-neighbor search method under the inner product metric would be sufficient to predict a sparsity pattern. "Training predictor" is to reduce the cost of on-the-fly prediction, rather than training the model itself.

For example, in our exploration stage mentioned in Section 4.1, we adopt HNSW, a state-of-art near-neighbor search method, to predict MLP sparse pattern, and we can see from the following table there is no drop in the perplexity at 90 \% sparsity ratio. However, due to the high dimensionality of embedding and HNSW’s reliance on CPU, the time HNSW took to identify the sparsity pattern is 10ms, which is longer than the MLP computation.

\begin{table}[h]
\centering
\begin{tabular}{ccc}
\hline
          & OPT-1.3B & OPT-1.3B + HNSW \\
\hline
Hellaswag & 0.4154   & 0.4314          \\
C4        & 14.2     & 14.4           \\
\hline
\end{tabular}
\end{table}
In our paper, we choose a neural network classifier as our near neighbor search method to take advantage of the fast matrix multiplication on GPU. And training such classifiers to predict sparsity patterns is not only cheaper in terms of training cost but also inherently different from the method concept.
\subsection{Future Possibility: Skipping Layer}
\label{sec:exp_skip_layer}

\label{sec:block_parallel}
Deja Vu currently sparsifies from the perspective of model width. Here, we explore the possibility of sparsification from model depth. 
As observed in \cref{sec:obs}, we show that the activation of large language models changes slowly across blocks. This property can be leveraged to increase the efficiency of a trained model by parallelizing, reordering, or skipping certain intermediate sub-blocks without significantly impacting the overall accuracy. 
\begin{table}[t]
\scriptsize
\centering
\caption{ Sparsify from the Depth: Skipping or parallel entire transformer blocks may not lead to catastrophic drop in accuracy at test time.}
\resizebox{0.6\linewidth}{!}{
\centering
\Huge
\begingroup
\setlength{\tabcolsep}{10pt}
\renewcommand{\arraystretch}{1.3}
\begin{tabular}{l||cccccc}
\specialrule{.15em}{.05em}{.05em}
 Model & COPA & Hellaswag & Lambada & OpenBookQA & PIQA  & Winogrande	\\
\cline{1-7}
 OPT-175B        
            & 0.8600 & 0.7814  & 0.7584 & 0.4460 & 0.8096 &  0.7261 \\
 \, - Parallel 2  
            & 0.8300 & 0.7737  & 0.7762 & 0.4520 & 0.8030 & 0.7096 \\
 \, - Parallel 4  
            & 0.5200 & 0.2519  & 0      & 0.2720 & 0.5092 & 0.4870 \\
 \, - Skip 2/8  
            & 0.8000 & 0.7112  & 0.6387 & 0.4220 & 0.7840 & 0.6630 \\
 \, - Skip 2/4 
            & 0.6900 & 0.4409  & 0.0240 & 0.3400 & 0.6882 & 0.5383 \\
\cline{1-7}
Bloom         
            & 0.8000  & 0.7460  & 0.6771 & 0.4480 & 0.7949 & 0.7040  \\
\, - Parallel 2   
            & 0.8100  & 0.7404  & 0.6992 & 0.4360 & 0.7813 & 0.7048 \\
\, - Parallel 4
            & 0.6200  & 0.3176  & 0.1325 & 0.2720 & 0.5593 & 0.5217 \\
 \, - Skip 2/8  
            & 0.7900 & 0.6829  & 0.5936 & 0.4120 & 0.7699 & 0.6614 \\
 \, - Skip 2/4 
            & 0.6600 & 0.5538  & 0.3023 & 0.3580 & 0.7046 & 0.5549 \\
\specialrule{.15em}{.05em}{.05em}
\end{tabular}
\endgroup
}
\end{table}
\begin{table}[ht]
\scriptsize
\centering
\resizebox{0.3\linewidth}{!}{
\centering
\Huge
\begingroup
\setlength{\tabcolsep}{10pt}
\renewcommand{\arraystretch}{1.3}
    \begin{tabular}{lrr}
    \toprule
    Setting                 & Wiki(ppl) & C4(ppl) \\
    \midrule
    Baseline                &  11.57    &   10.17    \\
    Skip every 2 layers      &  21.16    &   16.58   \\
    Skip every 4 layers      &  13.45    &   11.37  \\
    \bottomrule
    \end{tabular}
\endgroup
}
\label{table:corruption}
\end{table}

Improving the inference efficiency of Transformer models is a challenging task due to their sequential execution of Transformer layers.
Each sub-block depends on the output of the previous one, leading to low hardware efficiency, particularly during the token generation phase where each forward pass is computed for only one token.
However, the sequential execution of blocks and sub-blocks yields computation bubbles, and the latter involves a large amount of communication overhead. 
Here, we present an interesting observation that can potentially alleviate these challenges. We found that the activation of the model changes slowly across blocks. Specifically, the cosine similarity of activations between adjacent blocks is often above 0.99.
This suggests that the blocks might take the previous activation as input -- parallelize or reorder the blocks -- without significantly affecting the output.
Slowly changing activations suggest that it may be possible to parallelize, reorder, or even skip blocks while maintaining a similar output.
Some existing models, such as GPT-J~\citep{gpt-j}, GPT-NeoX~\citep{gpt-neox-20b}, and PaLM~\citep{chowdhery2022palm} already placed the Attention block and MLP block in parallel in training to facilitate parallel computation and reduce the communication overhead. 

Here we investigate the possibility at inference time. And surprisingly, we found parallelizing those blocks for models that are trained in a sequence manner will not hurt the performance of downstream tasks significantly. And surprisingly, we found parallelizing those blocks for models that are trained in a sequence manner will not hurt the performance of downstream tasks significantly. Table\ref{table:corruption} presents some preliminary results of OPT-175B and Bloom

Given the activation $y$ and Transformer layer $l$, we have:
\begin{align*}
    \widetilde{y}_l \leftarrow y_l + \mathsf{MHA}^{l}(y_l) \\
   \widehat{y}_l \leftarrow \widetilde{y}_l + \mathsf{MLP}^{l}(\widetilde{y}_l)
\end{align*}
Parallelizing two blocks refers to placing the Attention and MLP blocks in parallel, i.e.:
\begin{align*}
    \widehat{y}_l \leftarrow y + \mathsf{MHA}^{l}(y_l) + \mathsf{MLP}^{l}(y_l)
\end{align*}
Parallelizing four blocks then parallelize the blocks of two Transformer layers, defined as follows:
\begin{align*}
    \widehat{y}_{l+1} \leftarrow y_l + \mathsf{MHA}^{l}(y_l) + \mathsf{MLP}^{l}(y_l)
    + \mathsf{MHA}^{l+1}(y_l) + \mathsf{MLP}^{l+1}(y_l)
\end{align*}
Skipping layers is straightforward, which drops an entire Transformer layer for every $n$ layers.

We are surprised to find that parallel two layers preserve accuracy on a series of tasks across models. Besides, randomly skipping 25\% layers doesn't lead to catastrophic quality. Our findings suggest from the downstream task perspective, the activation patterns within the model are relatively consistent across different blocks, providing a potential avenue for future research on model compression and optimization.




\section{Implementation Details}
\label{appendix:method}

Figure~\ref{fig:sparse_computation_diagram} presents a more detailed workflow of \name{}. The left diagram shows how an input $y$ performs the sparse MHA with selected indices ${0, 3}$, predicted by the head predictor. Similarly, the right diagram shows how an input $y$ performs the sparse MLP with selected indices ${0, 2}$, predicted by the neuron predictor of that layer.

\begin{figure}[h]
  \centering
  \includegraphics[width=0.8\textwidth]{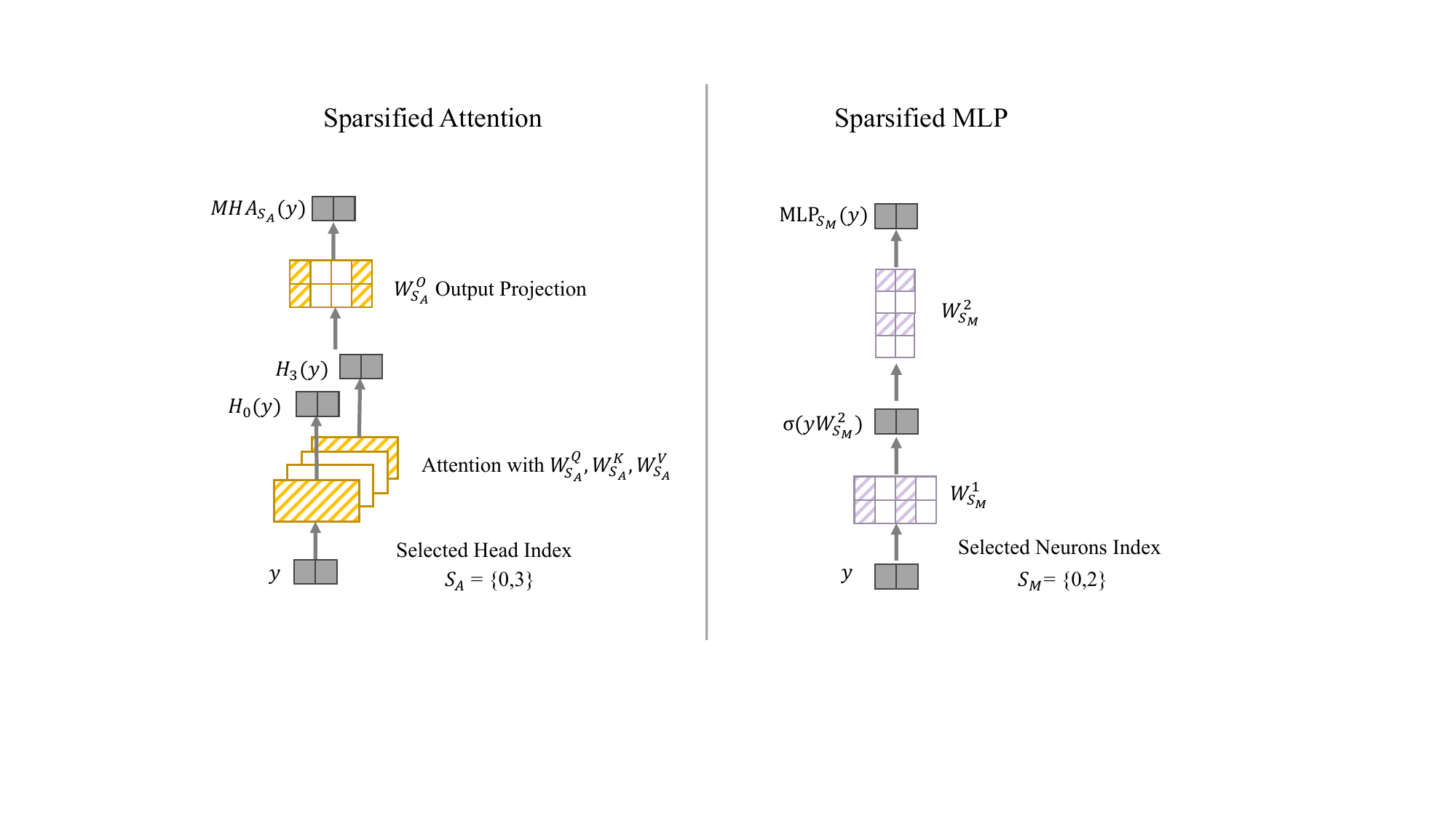}
  \caption{Detailed diagram on the sparsified computation process of MLP and Attention. Notation refers to Section~\ref{sec:formulation}}
  \label{fig:sparse_computation_diagram}
\end{figure}

Next, we will present a general explanation of two optimizations we used in \name{} implementation. Kernel fusion: A standard implementation of sparse matrix-vector
  multiply (e.g., $Wx$ in PyTorch) that separately indexes a subset of the matrix
  $W[\mathrm{idx}, :]$ before multiplying with input $x$ would incur 3$\times$ the
  amount of memory IOs: one to load a subset of $W$ from GPU memory, one to
  write that subset to a different contiguous region in memory, and one to load
  that (now contiguous) subset in again to multiply with $x$.
  Similarly, to use sparse matrix multiply routines (e.g., cuSparse), we would
  first need to convert $W[\mathrm{idx}, :]$ to sparse format, again incurring
  more memory IOs.
  We instead fuse the indexing and the multiplication step: we load a subset of
  $W[\mathrm{idx}, :]$ to memory, along with $x$, perform the multiply, then
  write down the result.
  This fused implementation (in Triton~\citep{tillet2019triton}) yields up to
  4$\times$ speedup compared to a standard PyTorch implementation (Section~\ref{sec:mlp_attn_benchmarks}). Memory coalescing: the weight matrices are conventionally stored in
  row-major format.
  This allows us to load $W[\mathrm{idx}, :]$ optimally (as the second dimension
  is contiguous in memory).
  However, for cases where we need to load $W[:, \mathrm{idx}]$ (attention
  output projection and the 2nd weight matrix in the MLP) this format
  significantly slows down memory loading, as $\mathrm{idx}$ could contain
  indices pointing to non-contiguous memory.
  A simple solution is to store these matrices in column-major format (i.e.,
  storing $W^\top$ in contiguous row-major format), then use the same fused kernel
  above.
  This transposition is done once when loading the model, and incurs no added
  cost during generation.
\section{Benchmarking Sparse MLP and Sparse Attention}
\label{sec:mlp_attn_benchmarks}

\begin{figure}[h]
  \centering
  \includegraphics[width=0.8\textwidth]{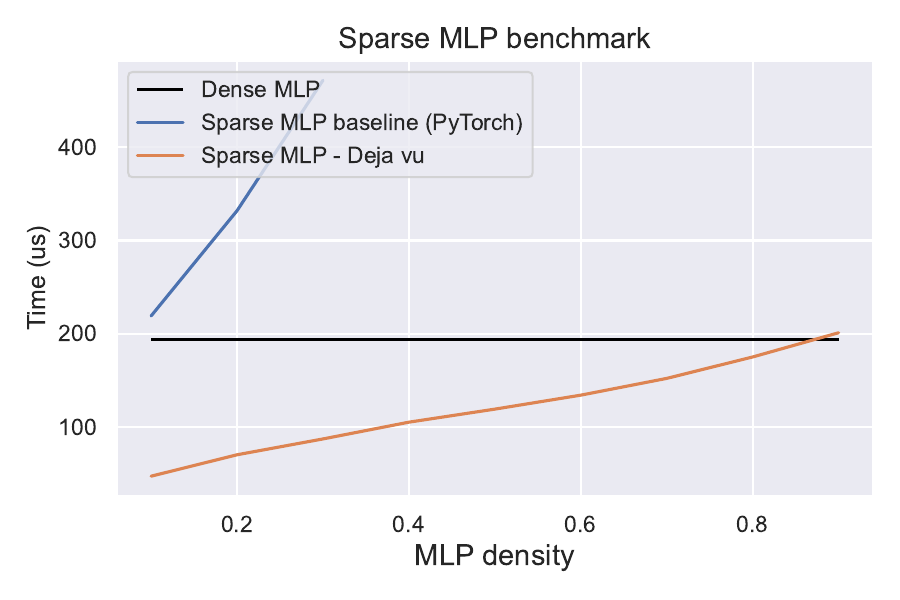}
  \caption{\label{fig:mlp_sparse_speed}Speed benchmarking of the MLP layer of
    OPT-175B on 8xA100s. Our sparse implementation is up to 4.5$\times$ faster than
    the baseline implementation in PyTorch. Our sparse MLP implementation
    remains faster than dense MLP for density up to 0.8.}
\end{figure}
\begin{figure}[h]
  \centering
  \includegraphics[width=0.8\textwidth]{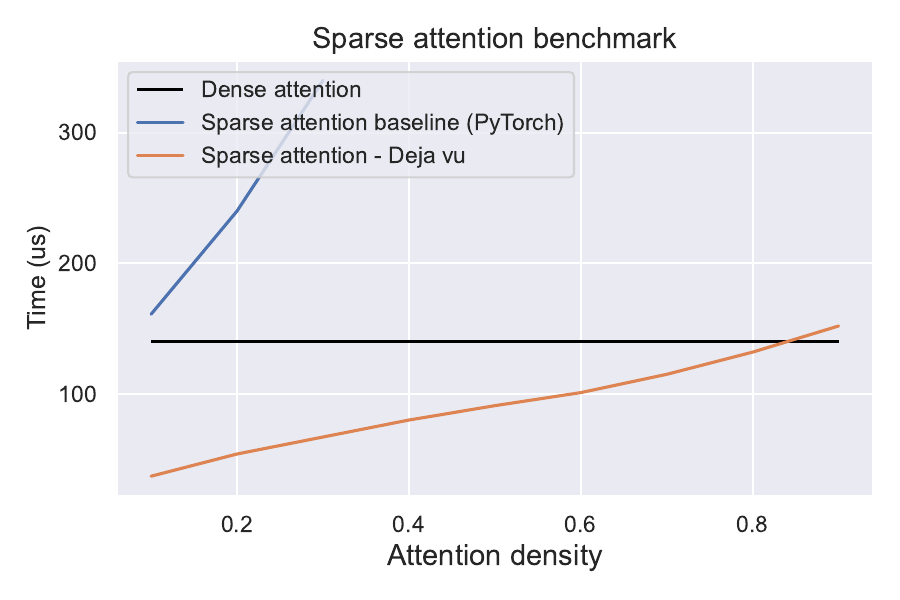}
  \caption{\label{fig:attn_sparse_speed}Speed benchmarking of the attention layer of
    OPT-175B on 8xA100s. Our sparse implementation is up to 5$\times$ faster than
    the baseline implementation in PyTorch. Our sparse attention implementation
    remains faster than dense MLP for density up to 0.8.}
\end{figure}

We validate that our hardware-aware implementation of sparse MLP and sparse
attention (Section~\ref{sec:sparse_matmul}) yields wall-clock speed up compared to both dense MLP/attention and
compared to the standard implementation in PyTorch.

Recall that our implementation fuses the sparse indexing and the multiplication
$(W_{S_M}^1)^\top y$ for weight matrices $(W^1)^\top$ and input vector $y$.
In cases where we need to index $W^2_{S_M}$, we store the transpose of
$W^2$ to ensure memory coalescing.
For the baseline implementation in PyTorch, we index
$(W_{S_M}^1)^\top$ as a separate operation before multiplying with $y$, which
incurs more memory reads/writes. 

Similarly, we fuse the sparse indexing and the multiplication
$(W_{S_A}^Q)^\top y$, $(W_{S_A}^K)^\top y$, $(W_{S_A}^V)^\top y$ for weight matrices $(W^Q)^\top$, $(W^K)^\top$, $(W^V)^\top$ and input vector $y$. Note we usually concatenate all three matrices in the standard implementation, but we separate them here for clarity. In cases where we need to index $W^O_{S_A}$, we store the transpose of
$W^O$ to ensure memory coalescing.

In Figure~\ref{fig:mlp_sparse_speed} and Figure~\ref{fig:attn_sparse_speed}, our
sparse MLP and attention implementations are 4-5$\times$ faster than the baseline
implementation in Pytorch, and remains faster than the dense version for density
up to 0.8.

\section{Notations and Basic Definitions}\label{sec:notation_definition}

For a positive integer $n$, let $[ n ] := \{ 1, 2, \cdots, n\}$. For a matrix $A \in \R^{n \times n}$, let $A_{i, :}$ and $A_{:, j}$ be two column vectors corresponding to the $i$-th row and the $j$-th column of $A$ respectively, and $A_{i, j}$ be the entry at the $i$-th row and the $j$-th column. For a vector $x \in \R^n$, let $\sqrt{x} \in \R^{n}$ denote the vector with the $i$-th entry being $\sqrt{x_i}$ and $\diag ( x ) \in \R^{n \times n}$ denote the diagonal matrix with the $i$-th diagonal entry being $x_i$. For two matrices $A, W \in \R^{n \times n}$, let $\| A \|_W := (\sum_{i=1}^n \sum_{j=1}^n W_{i, j} A_{i, j}^2)^{1/2}$ and $W \circ A$ denote the matrix where $(W \circ A)_{i,j} = W_{i,j} A_{i,j}$. For matrix $W \in \R^{n \times n}$, let $D_{W_i} := \diag( W_{i, :} )$ with $i \in [n]$. 

For two vectors $x \in \R^n$ and $w \in \R^n_{\geq 0}$, let $\| x\|_w := (\sum_{i=1}^n w_i x_i^2)^{1/2}$.   For a vector $x$, we denote $\| x \|_2 := ( \sum_{i=1}^n x_i^2 )^{1/2}$ as its $\ell_2$ norm. We denote $\| x \|_p := (\sum_{i=1}^n |x_i|^p)^{1/p}$ as its $\ell_p$ norm. For a square matrix $A$, we denote $\tr[A]$ as the trace of matrix $A$. 

For a matrix $A \in \R^{n \times k}$ (suppose $n \geq k$), we use $\| A \|$ to denote its spectral norm, i.e., $\| A \| = \sup_{x} \| A x \|_2 / \| x \|_2$. We use $\| A \|_F$ to denote its Frobenius norm $\| A \|_F : = (\sum_{i=1}^n \sum_{j=1}^k A_{i,j}^2 )^{1/2}$.

Suppose matrix $A \in \R^{n \times k}$ has SVD decomposition $U \Sigma V^\top$ where $U \in \R^{n \times k}$ (this matrix has orthonormal columns), $\Sigma \in \R^{k \times k}$ is a diagonal matrix, and $V \in \R^{k \times k}$. We call columns of $U$ are singular vectors. We use $A^\dagger \in \R^{k \times n}$ to denote the Moore-Penrose pseudoinverse, then $A^\dagger = V \Sigma^{-1} U^\top$. Suppose $\Sigma \in \R^{k \times k}$ is sorted diagonal matrix, let $\sigma_1, \cdots, \sigma_k$ denote the diagonal entries of $\Sigma$. Then we call $\sigma_i$ the $i$-th singular value of matrix, and we write it as $\sigma_i(A)$.

For any symmetric matrix $B \in \R^{k \times k}$, we define its eigenvalue decomposition as $U \Lambda U^\top$, where $\Lambda$ is a diagonal matrix. Let $\lambda_1, \cdots, \lambda_k$ denote the entries on diagonal of $\Lambda \in \R^{k \times k}$. We say $\lambda_i$ is the $i$-th eigenvalue. Usually we write it as $\lambda_{i}(B)$.

The connection between eigenvalues and singular values is 
\begin{align*}
\sigma_i^2(A) = \lambda_i (A^\top A)
\end{align*}

We use notation $A \succeq 0$ to denote that matrix $A$ is positive semidefinite (psd). Mathematically, $A\succeq 0$ means for all vectors $x$, we have $x^\top A x \geq 0$.

Similarly, for two squarer matrices $A$ and $B$, we use $A \succeq B$ to denote the case where for all vectors $x$, $x^\top  Ax \geq x^\top B x$. 

We use $\Pr[]$ and $\E[]$ for probability and expectation. We denote $\max\{a,b\}$ as the maximum between $a$ and $b$. We denote $\min \{a,b\}$ (resp. $\max\{a,b\}$) as the minimum (reps. maximum) between $a$ and $b$. 

 Throughout, for non-negative real numbers
a and b, we use the notation $a = (1 \pm \epsilon)b ~\text{if}~a \in [(1 - \epsilon)b,(1 + \epsilon)b]$.

\section{Subspace Embeddings and Norm Preserving}\label{sec:subspace_embedding}
In Section \ref{sec:soft_max_func}, we show the norm preserving of the soft-max functions.
In Section \ref{sec:relu_func}, we show the norm preserving of the ReLU function.
In Section \ref{sec:folded_dist}, we introduce the folded Guassian distribution.
In Section \ref{sec:l2_subspace}, we introduce the $\ell_2$ subspace embedding.
In Section \ref{sec:l1_subspace}, we introduce the $\ell_1$ subspace embedding.
In Section \ref{sec:rand_mat}, we introduce different sketching matrices for subspace embedding.


\subsection{Soft-Max Functions}\label{sec:soft_max_func}

Let $K \in \R^{s \times d}$ and $V \in \R^{d \times s}$.
Inspired by the softmax unit in the attention scheme of large language models. The softmax related regression has been studied in many settings \cite{zhdk23,as23,bsz23,lsz23,dms23,dls23,gms23,lsx+23,gsy23_dp}. In this work, we follow the standard softmax definition. 
We define $\sigma_1 : \R^s \rightarrow \R^s$ to be a softmax function, i.e., for any vector $y \in \R^s$, the $\sigma(y)$ can be written as
\begin{align*}
\sigma_1( y )_i =  \frac{ \exp(y_i) }{ \sum_{j=1}^d \exp(y_j) } , ~~~ \forall i \in [d]
\end{align*}

The standard softmax is $\ell_1$ version. In this work, we also consider the $\ell_2$ generalization. 
We define $\sigma_2: \R^s \rightarrow \R^s$ to be  a softmax function ($\ell_2$ version), i.e., for any vector $y \in \R^s$, the $\sigma(y)$ can be written as
\begin{align*}
\sigma_2( y )_i =  \frac{ \exp( y_i) }{ ( \sum_{j=1}^d \exp(2 y_j) )^{1/2} } , ~~~ \forall i \in [d]
\end{align*}

We define function $f : \R^d \rightarrow \R^d$
\begin{align}\label{eq:f_x}
f (x) =  V \cdot ( \sigma (K \cdot x) ) 
\end{align}

\begin{definition}
We say ${\cal X} \subset \R^d$ is a rank-$k$ subspace, if there is an orthonormal basis $U \in \R^{d \times k}$, for any $x \in {\cal X}$, there is $ y \in \R^k$ such that
\begin{align*}
x = U y.
\end{align*}
\end{definition}

We can have
\begin{lemma}\label{lem:ell_2_subspace_imply_norm_preserve}
Let $\tau \in (0,1)$. Let ${\cal X} \subset \R^d$ denote a subspace with rank $k$.
Let $f$ be defined based on $\sigma_2$ function. 
Let $\ov{V}$ is a random Gaussian matrices with $d \geq \Omega( \epsilon^{-2} ( k + \log(1/\delta) ))$ rows. Let $V = \tau \ov{V} $, then we have with probability $1-\delta$ 
\begin{align*}
(1-\epsilon) \tau \| x \|_2 \leq \| f(x) \| \leq (1+\epsilon) \tau \| x \|_2.
\end{align*}
for all unit vectors $x \in {\cal X}$.

Further, if $d=  O(k + \log(1/\delta))$, then we have
\begin{align*}
0.5 \tau \| x \|_2 \leq \| f(x) \| \leq 2 \tau \| x \|_2.
\end{align*}
\end{lemma} 
\begin{remark}
The above condition implies that $f$ is a shrinking operator but also not shrinking arbitrarily small.
\end{remark}
\begin{proof}
Given $d \geq \Omega( \epsilon^{-2} ( k + \log(1/\delta) ) )$, by using Lemma~\ref{lem:rand_gauss} 
, we have
\begin{align*}
(1-\epsilon)  \| y \|_2 \leq \| \ov{V} y \|_2 \leq (1+\epsilon)  \| y \|_2
\end{align*}
As the input of the function $f$ here is the output of a softmax function ($\ell_2$ version), we know that $\| y \|_2 = 1$.

Thus, we have
\begin{align*}
(1-\epsilon)  \leq \| \ov{V} y \|_2 \leq (1+\epsilon) 
\end{align*} 
By rescaling $V$, we have
\begin{align*}
(1-\epsilon)  \| x \|_2 \leq \| V y \|_2 \leq (1+\epsilon)   \| x \|_2.
\end{align*} 
\end{proof}

\begin{lemma}
Let $\tau \in (0,1)$. Let ${\cal X} \subset \R^d$ denote a subspace with rank $k$. 
Let $f$ be defined based on $\sigma_1$ function. Suppose $\ov{V}$ is a random Gaussian matrix with $d \geq \Omega( (k + \log(1/\delta)) )$ 
rows. Let $V = \frac{1}{2} \tau \ov{V}$.

Then we have
\begin{align*}
\frac{1}{4\sqrt{s}} \tau \cdot \| x \|_2 \leq \| f( x ) \|_2 \leq  \tau \cdot \| x \|_2
\end{align*}
for all unit vectors $x$.
\end{lemma}
\begin{proof}

By property of subspace embedding, we know that if $d \geq \Omega(\epsilon^{-2} (s+\log(1/\delta)))$,
\begin{align*}
(1-\epsilon) \| y \|_2 \leq \| \ov{V} y \|_2 \leq (1+\epsilon) \| y \|_2
\end{align*}
By property of function of $f$, we know we only need to care $\| y \|_1 = 1$, this implies that 
\begin{align*}
\frac{1}{\sqrt{s}} \| y \|_1 \leq \| y \|_2 \leq \| y \|_1
\end{align*}

On one hand, we have
\begin{align}\label{eq:vy_upper}
\| \ov{V} y \|_2 \leq &~  (1+\epsilon) \cdot \| y \|_2 \notag\\
\leq &~  (1+\epsilon) \cdot \| y \|_1 \notag\\
=&~  (1+\epsilon),
\end{align}
where the first step follows from $\| \ov{V} y \|_2 \leq (1+\epsilon) \| y \|_2$, the second step follows from $\| y \|_2 \leq \| y \|_1$ and the last step follows from $\| y \|_1  = 1$.

On the other hand, we have
\begin{align}\label{eq:vy_lower}
\| \ov{V} y \|_2 \geq &~ (1-\epsilon) \| y \|_2 \notag\\
\geq &~  \frac{1}{\sqrt{s}} (1-\epsilon) \| y \|_1 \notag\\
=&~ \frac{1}{\sqrt{s}} (1-\epsilon),
\end{align}
where the first step follows from $(1-\epsilon) \| y \|_2 \leq \| \ov{V} y \|_2 $, the second step follows from $\frac{1}{\sqrt{s}} \| y \|_1 \leq \| y \|_2 $ and the last step follows from $\| y \|_1  = 1$.

Combining Eq.~(\ref{eq:vy_lower})
 and Eq.~(\ref{eq:vy_upper}) together, we have
\begin{align*}
(1-\epsilon) \frac{1}{\sqrt{s}}  \leq \| \ov{V} y \|_2 \leq (1+\epsilon) 
\end{align*} 
Choosing $\epsilon = 1/2$, we have 
\begin{align*}
 \frac{1}{2\sqrt{s}}  \leq \| \ov{V} y \|_2 \leq 2.
\end{align*}
By $V = \frac{1}{2} \tau \ov{V}$ and $\|x\|_2 = 1$, we have
\begin{align*}
\frac{1}{4\sqrt{s}} \tau \| x \|_2 \leq \| V y \|_2 \leq   \tau \| x \|_2.
\end{align*} 

\end{proof}

\subsection{ReLU Functions}\label{sec:relu_func}

We use $\phi : \R \rightarrow \R$ to denote ReLU function, i.e., $\phi(z) = \max\{z,0\}$.

We define function $g : \R^d \rightarrow \R^d$
\begin{align}\label{eq:g_x}
g (x) =  V \cdot ( \phi (K \cdot x) ) 
\end{align}

Let $K \in \R^{s \times d}$ and $V \in \R^{d \times s}$.

\begin{lemma}
Let ${\cal X} \subset \R^d$ denote a rank-$k$ subspace. 
Let $K$ denote a random Gaussian matrix. Let $V$ denote a random Gaussian matrix. Let $s \geq \Omega( \epsilon^{-2}  k\log(1/ (\delta \epsilon ) ) )$. Let $d \geq \Omega(\epsilon^{-2} (k + \log(1/\delta)))$. Then we know with high probability $1-\delta$, for all unit vector $x \in {\cal X}$
\begin{align*}
(1-\epsilon) \| x \|_2 \leq \| f(x) \|_2 \leq (1+\epsilon) \| x \|_2
\end{align*}
\end{lemma}
\begin{proof}
Suppose $s \geq \Omega( \epsilon^{-2} \log(1/\delta ) )$.

Using Lemma~\ref{lem:lm}, Fact~\ref{fac:key_property_ReLU}, we can show that for each fixed 
\begin{align*}
(1-\epsilon) \| x \|_2 \leq \| \phi(K x ) \|_2 \leq (1+\epsilon) \| x \|_2
\end{align*}
holds with probability $1-\delta$.

By a standard $\epsilon$-net argument (Lemma~\ref{lem:epsilon_net}), the net points in ${\cal X}$ is at most $(10/\epsilon)^{O(k)}$.

Taking a union bound over all the net points, we can show that for all $x \in {\cal X}$
\begin{align*}
(1-\epsilon) \| x \|_2 \leq \| \phi(K x ) \|_2 \leq (1+\epsilon)\| x \|_2
\end{align*}
holds with probability $1-\delta/2$ and $s \geq \Omega( \epsilon^{-2} k \log(1/ (\delta \epsilon) ) )$.

Further, we using Lemma~\ref{lem:rand_gauss}, we can show that
\begin{align*}
(1-\epsilon) \| \phi(K x) \|_2 \leq \| f(x) \|_2 \leq (1+\epsilon) \| \phi(K x) \|_2
\end{align*}
holds with probability $1-\delta/2$.

Combining together,
\begin{align*}
 (1-\epsilon)^2  \| x \|_2 \leq \| f(x) \|_2 \leq (1+\epsilon)^2 \| x \|_2 
\end{align*}
holds with probability $1-\delta$.

Rescaling the $\epsilon$, we complete the proof.

\end{proof}

\subsection{Folded Gaussian Distribution}\label{sec:folded_dist}

We state a standard tool from literature,
\begin{lemma}[Lemma 1 on page 1325 of Laurent and Massart \cite{lm00}
]\label{lem:lm}
    Let $X \sim \mathcal{X}_k^2$ be a chi-squared distributed random variable with $k$ degrees of freedom. Each one has zero means and $\sigma^2$ variance. 
    
    Then,
    \begin{align*}
        \Pr[X - k\sigma^2 \geq (2\sqrt{kt} + 2t) \sigma^2]
        \leq & ~ \exp{(-t)}\\
        \Pr[k\sigma^2 - X \geq 2\sqrt{kt}\sigma^2]
        \leq & ~ \exp{(-t)}
    \end{align*}
    Further if $k \geq \Omega(\epsilon^{-2} t)$ and $t \geq \Omega(\log(1/\delta))$, then we have
    \begin{align*}
    \Pr[ | X - k \sigma^2 | \leq \epsilon k \sigma^2 ] \leq \delta.
    \end{align*}
\end{lemma}

We prove the following property,
\begin{fact}\label{fac:key_property_ReLU}
Let $h, q \in \R^p$ be fixed vectors and $h \neq 0, W \in \R^{m \times p}$ be random matrix with i.i.d. entries $W_{i, j} \sim \mathcal{N} (0, \frac{2}{m} )$, and vector $v \in \R^m$ defined as $v_i=\phi ((W h)_i )=\mathbf{1}_{(W(h+q))_i \geq 0}(W h)_i$. Then,
\begin{itemize}
    \item $ |v_i |$ follows i.i.d. from the following distribution: with half probability $ |v_i |=0$, and with the other half probability $ |v_i |$ follows from folded Gaussian distributions $ |\mathcal{N} (0, \frac{2\|h\|^2}{m} ) |$.
    \item $\frac{m\|v\|^2}{2\|h\|^2}$ is in distribution identical to $\chi_\omega^2$ (chi-square distribution of order $\omega$ ) where $\omega$ follows from binomial distribution $\mathcal{B}(m, 1 / 2)$.
\end{itemize}

\end{fact}

\begin{proof}
    We assume each vector $W_i$ is generated by first generating a gaussian vector $g \sim \mathcal{N} (0, \frac{2   I}{m} )$ and then setting $ W_i=\pm g$ where the sign is chosen with half-half probability. Now, $ | \langle W_i, h \rangle |=|\langle g, h\rangle|$ only depends on $g$, and is in distribution identical to $ |\mathcal{N} (0, \frac{2\|h\|^2}{m} ) |$. Next, after the sign is determined, the indicator $\mathbf{1}_{ \langle W_i, h+q \rangle \geq 0}$ is $1$ with half probability and $0$ with another half. Therefore, $ |v_i |$ satisfies the aforementioned distribution. As for $\|v\|^2$, letting $\omega \in\{0,1, \ldots, m\}$ be the variable indicator how many indicators are $1$ , then $\omega \sim \mathcal{B}(m, 1 / 2)$ and $\frac{m\|v\|^2}{2\|h\|^2} \sim \chi_\omega^2$.
\end{proof}

\subsection{$\ell_2$ subspace embedding}\label{sec:l2_subspace}

We define a standard notion in sketching technique.\footnote{We remark that sketching technique has widely applied to many applications such as linear regression, low-rank approximation \cite{cw13,nn13,ldfu13,bwz16,c16,rsw16,swz17,swz19}, linear programming \cite{sy21,dly21,jswz21,gs22}, semi-definite programming \cite{gs22,syyz23}, empirical risk minimization\cite{lsz19,qszz23}, training over-parameterized neural network \cite{bpsw21,szz21,als+22,hswz22,z22}.}
\begin{definition}[$\ell_2$ subspace embedding \cite{s06}]\label{def:l2_subspace_embedding}
A $(\epsilon, \delta, \ell_2)$-subspace embedding for the column space of an $n \times d$ matrix $A$ is a matrix $S$ for which 
\begin{align*}
  \Pr [ \forall x \in \mathbb{R}^d,  \|S Ax\|_2^2 = (1 \pm \epsilon) \|Ax\|_2^2 ] \geq 1-\delta.
\end{align*}
The above condition is equivalent to
\begin{align*}
\Pr[ \| U^\top U - U^\top S^\top S U \| \leq \epsilon ] \geq 1-\delta.
\end{align*}
where the $U$ is the orthonormal basis of $A$.
\end{definition}
For the reason of above conditions are equivalent, we refer the readers to the survey \cite{w14}.

We state a standard tool in literature,
\begin{lemma}[Lemma 5 in \cite{w14}]\label{lem:epsilon_net}
Let ${\cal X} \subset\R^d$ be rank $k$. For any $\gamma \in (0,1)$, there is a $\gamma$-net $N$ of ${\cal X}$ for which $|N| \leq (1+4 /\gamma)^k$.
\end{lemma}

\subsection{$\ell_1$ subspace embedding}\label{sec:l1_subspace}

 When $p=1$, using Cauchy random variables, Sohler and Woodruff \cite{sw11} showed there exist $\ell_1$ oblivious subspace embeddings with $O(d \log d)$ rows and $\kappa=O(d \log d)$. This approach was generalized by using $p$-stable random variables in work of Meng and Mahoney \cite{mm13} to $\ell_p$-norms when $1<p<2$, where they showed there exist $\ell_p$ oblivious subspace embeddings with $O(d \log d)$ rows and $\kappa=O ((d \log d)^{1 / p} )$. Unlike the case when $p=2$, due to the large distortion


 In \cite{ww18}, they show for every $1 \leq p<2$, any oblivious subspace embedding with dimension $r$ has distortion $\kappa=\Omega (\frac{1}{ (\frac{1}{d} )^{1 / p} \cdot \log ^{2 / p} r+ (\frac{r}{n} )^{1 / p-1 / 2}} )$. 
 They also give sparse oblivious subspace embeddings for every $1 \leq p<2$ which are optimal in dimension and distortion, up to poly $(\log d)$ factors. Importantly for $p=1$, they achieve $r=O(d \log d), \kappa=O(d \log d)$ and $s=O(\log d)$ non-zero entries per column. 

\begin{definition}[$\ell_1$ subspace embedding]\label{def:l1_subspace_embedding}
Let $0< \alpha < \beta$ be parameters. We will say a matrix $S$ is an $\ell_1$ subspace embedding for an $n \times d$ matrix $A$ if there are constants $c_1, c_2 > 0$ so that for all $x \in \R^{d}$,
\begin{align*}
    \|Ax\|\leq \|SAx\|_1 \leq d^{c_1} \|Ax\|_1,
\end{align*}
and $S$ has at most $d^{c_2}$ rows.
\end{definition}

\subsection{Random Matrices}\label{sec:rand_mat}

\begin{table*}[!ht]
    \centering
\begin{tabular}{|c|l|l|l|}
\hline {\bf Matrices} & $b$ & {\bf Time for} $R \cdot A$ & {\bf Reference} \\
\hline Random Gaussian & $\epsilon^{-2}(d+\log (1 / \delta))$ & $\Tmat(b, n, d)$ & Thm. 6 of \cite{w14} \\
\hline SRHT & $\epsilon^{-2}(\sqrt{d}+\sqrt{\log n})^2 \log (d / \delta)$ & $n d \log  (\epsilon^{-1} d(\log n) )$ & Thm. 7 of \cite{w14} \\
\hline AMS & $\epsilon^{-2}(d+\log (1 / \delta))$ & $\Tmat(b, n, d)$ & Follow from JL guarantee \\
\hline Count-sketch & $\epsilon^{-2} \delta^{-1} d^2$ & $\mathrm{nnz}(A)$ & Thm. 9 of \cite{w14} \\
\hline Sparse embedding & $\epsilon^{-2} d \cdot$ poly $\log (d /(\epsilon \delta))$ & $\epsilon^{-1} \mathrm{nnz}(A)$ poly $\log (d /(\epsilon \delta))$ & Thm. 10 (2) of \cite{w14} \\
\hline Sparse embedding& $\epsilon^{-2} d^{1+\gamma}$ & $\epsilon^{-1} \mathrm{nnz}(A) \operatorname{poly}(1 / \gamma)$ & Thm. 10 (1) of \cite{w14} \\
\hline
\end{tabular}
\caption{Summary for different sketching matrices for subspace embedding. The sketching matrix $R$ has size $b \times n$. The vectors are from the column subspace of matrix $A$ with size $n \times d . \epsilon \in(0,1)$ is the error parameter, and $\delta \in(0,1)$ is the probability parameter. $\mathcal{T}_{\text {mat }}(a, b, c)$ denotes the running time of fast matrix multiplication of two matrices with size $a \times b$ and $b \times c .$ In the first sparse embedding matrix, each column has $s \geq \epsilon^{-1}$ poly $\log (d /(\epsilon \delta))$ non-zero entries;  In the second sparse embedding matrix, each column has $s \geq \epsilon^{-1}$ poly $(1 / \gamma)$ non-zero entries, $\gamma>0$ is a tunable parameter that gives different trade-offs, and $\delta$ can be as small as $1 /$ poly $(d).$ For count-sketch matrices, the subspace embedding guarantee is proved from JL moment property, instead of directly from JL guarantee.}
\label{tab:summary_sketching}
\end{table*}

\begin{lemma}[Theorem 6 of \cite{w14}]\label{lem:rand_gauss}
    Let $0<\epsilon, \delta <1$ and $S=\frac{1}{\sqrt{k}}   R \in \R^{k \times n}$ where the entries $  R_{i, j}$ of $  R$ are independent standard normal random variables. Then if $k=$ $\Theta (\epsilon^{-2}(d+\log (1 / \delta))  )$, then for any fixed $n \times d$ matrix $A$, with probability $1-\delta, S$ is a $(1 \pm \epsilon) \ell_2$-subspace embedding for $A$, that is, simultaneously for all $x \in \R^d, \| S A x\|_2=(1 \pm \epsilon)\|A x\|_2$. Here $C>0$ is an absolute constant.
\end{lemma}

We consider several standard sketching matrices:
\begin{enumerate}
    \item Random Gaussian matrices.
    \item Subsampled randomized Hadamard/Fourier transform (SRHT) matrices \cite{ldfu13}.
    \item AMS sketch matrices \cite{ams96}, random $\{-1,+1\}$ per entry.
    \item Count-Sketch matrices \cite{ccf02}, each column only has one non-zero entry, and is $-1,+1$ half probability each.
    \item Sparse embedding matrices \cite{nn13}, each column only has $s$ non-zero entries, and each entry is $-\frac{1}{\sqrt{s}},+\frac{1}{\sqrt{s}}$ half probability each.
    \item Uniform sampling matrices.
\end{enumerate}

\begin{definition}[Random Gaussian matrix]
    We say $R \in \R^{b \times n}$ is a random Gaussian matrix if all entries are sampled from $\mathcal{N}(0,1 / b)$ independently.
\end{definition}

\begin{definition}[Subsampled randomized Hadamard/Fourier transform matrix \cite{ldfu13}]
    We say $R \in \R^{b \times n}$ is a subsampled randomized Hadamard transform (SRHT) matrix \footnote{ In this case, we require  $\log n$  o be an integer. } if it is of the form $R=\sqrt{n / b} S H D$, where $S \in \R^{b \times n}$ is a random matrix whose rows are b uniform samples (without replacement) from the standard basis of $\R^n, H \in \R^{n \times n}$ is a normalized Walsh-Hadamard matrix, and $D \in \R^{n \times n}$ is a diagonal matrix whose diagonal elements are i.i.d. Rademacher random variables.
\end{definition}

\begin{definition}[AMS sketch matrix \cite{ams96}]
   Let $h_1, h_2, \cdots, h_b$ be $b$ random hash functions picking from a 4-wise independent hash family $\mathcal{H}=\{h:[n] \rightarrow\{-\frac{1}{\sqrt{b}},+\frac{1}{\sqrt{b}}\}\}$. Then $R \in \R^{b \times n}$ is a AMS sketch matrix if we set $R_{i, j}=h_i(j)$
\end{definition}

\begin{definition}[Count-sketch matrix \cite{ccf02}]
    Let $h:[n] \rightarrow[b]$ be a random 2-wise independent hash function and $\sigma:[n] \rightarrow\{-1,+1\}$ be a random 4-wise independent hash function. Then $R \in \R^{b \times n}$ is a count-sketch matrix if we set $R_{h(i), i}=\sigma(i)$ for all $i \in[n]$ and other entries to zero.
\end{definition}

\begin{definition}[Sparse embedding matrix I \cite{nn13}]
    We say $R \in \R^{b \times n}$ is a sparse embedding matrix with parameter $s$ if each column has exactly $s$ non-zero elements being $\pm 1 / \sqrt{s}$ uniformly at random, whose locations are picked uniformly at random without replacement (and independent across columns) \footnote{For our purposes the signs need only be $O(\log d)$-wise independent, and each column can be specified by a $O(\log d)$-wise independent permutation, and the seeds specifying the permutations in different columns need only be $O(\log d)$-wise independent.}.
\end{definition}

\begin{definition}[Sparse embedding matrix II \cite{nn13}]
   Let $h:[n] \times[s] \rightarrow[b / s]$ be a random 2-wise independent hash function and $\sigma:[n] \times[s] \rightarrow\{-1,1\}$ be a 4-wise independent. Then $R \in \R^{b \times n}$ is a sparse embedding matrix II with parameter $s$ if we set $R_{(j-1) b / s+h(i, j), i}=\sigma(i, j) / \sqrt{s}$ for all $(i, j) \in[n] \times[s]$ and all other entries to zero \footnote{This definition has the same behavior as sparse embedding matrix I for our purpose}.
\end{definition}

\begin{definition}[Uniform sampling matrix]
    We say $R \in \R^{b \times n}$ is a uniform sampling matrix if it is of the form $R=\sqrt{n / b} S D$, where $S \in \R^{b \times n}$ is a random matrix whose rows are b uniform samples (without replacement) from the standard basis of $\R^n$, and $D \in \R^{n \times n}$ is a diagonal matrix whose diagonal elements are i.i.d. Rademacher random variables.
\end{definition}

\section{Distances, Angles, and Inner Product}\label{sec:distances_angles}

Most of the properties in this section are very standard in literature, e.g., see \cite{gsyz23}.

Let $U \in \R^{n \times k}$ denote an orthonormal basis, we use $U_{\bot} \in \R^{n \times (n-k)}$ denote the matrix such that $U U^\top + U_{\bot} U_{\bot}^\top = I_n$.

\begin{definition}\label{def:angle_and_distance}
Let $X \in \R^{n \times k}$ and $Y \in \R^{n \times k}$.

For any matrix $X$, and for orthogonal matrix $Y$ ($Y^\top Y = I_k$) we define
\begin{itemize}
    \item $\tan \theta(Y,X) := \| Y_{\bot}^\top X ( Y^\top X )^{-1} \|$
\end{itemize}
For orthogonal matrices $Y$ and $X$ ($Y^\top Y = I_k$ and $X^\top X = I_k$), we define 
\begin{itemize}
    \item $\cos \theta (Y,X) := \sigma_{\min} (Y^\top X)$. 
    \begin{itemize} 
        \item It is obvious that $\cos (Y,X) = 1/ \| (Y^\top X)^{-1} \|$ and $\cos(Y,X) \leq 1$.
    \end{itemize}
    \item $\sin \theta(Y,X) := \| (I - Y Y^\top) X \|$.
    \begin{itemize} 
        \item It is obvious that $\sin \theta(Y,X) = \| Y_{\bot} Y_{\bot}^\top X \| = \| Y_{\bot}^\top X \|$ and $\sin \theta(Y,X) \leq 1$.
    \end{itemize}
    \item $\dist(Y,X) := \min_{Q \in O_k} \| YQ - X \|$
\end{itemize}  
where $O_k$ is the set of $k \times k$ orthogonal matrices. 
\end{definition}

\begin{lemma}[Structural lemma for orthogonal matrices]
\label{lem:trig_structural}
Let $X, Y\in \R^{n\times k}$ be orthogonal matrices. Then
\begin{align*}
    (Y^\top X)_\bot = & ~ Y_\bot^\top X.
\end{align*}
\end{lemma}

\begin{proof}
    Let us first compute the Gram of $Y^\top X$, which is
\begin{align*}
    X^\top YY^\top X 
    = & ~ X^\top (I-Y_\bot Y_\bot^\top) X \\
    = & ~ X^\top X-X^\top Y_\bot Y_\bot^\top X \\
    = & ~ I_k-X^\top Y_\bot Y_\bot^\top X,
\end{align*}
where the first step follows from $Y_\bot Y_\bot^\top + YY^\top = I$, the second step follows from simple algebra, and the last step follows from $X$ is an orthogonal matrix, so $X^\top = X^{-1}$.

This means that $(Y^\top X)_\bot=Y_\bot^\top X$.
\end{proof}

\begin{lemma}[Orthogonal and inverse share singular vectors]
\label{lem:perp_inv_singular}
Let $A\in \R^{k\times k}$ be non-singular, then $A_\bot$ and $A^{-1}$ have the same set of singular vectors. Consequently, $\|A_\bot A^{-1}\|=\|A_\bot\|\|A^{-1}\|$. 

\end{lemma}

\begin{proof}
     Let $A\in \R^{k\times k}$ and $A^\top A+A_\bot^\top A_\bot=I_k$, we will show that $\|A_\bot A^{-1}\|=\|A_\bot \| \|A^{-1}\|$. Let $x\in \R^k$ be the unit eigenvector of $A$ that realizes the spectral norm, note that
\begin{align*}
    \|A_\bot x\|_2^2 = & ~ 1-\|A\|^2,
\end{align*}
we argue that $x$ corresponds to the smallest singular value of $A_\bot$ via contradiction. Suppose there exists some unit vector $y$ with $\|A_\bot y\|_2 < \|A_\bot x\|_2$, by definition, we know that $\|A_\bot y\|_2^2+\|Ay\|_2^2=1$, this means that $\|Ay\|_2>\|Ax\|_2=\|A\|$, contradicts the definition of spectral norm. Similarly, if $z$ is the unit vector that realizes the spectral norm of $A_\bot$, then it is also singular vector corresponds to the smallest singular value of $A$, or equivalently, the spectral norm of $A^{-1}$. Our above argument essentially implies that $A_\bot$ and $A^{-1}$ have the same set of singular vectors. The proof is then straightforward: suppose $A_\bot z=\lambda z$ and $A^{-1}z=\mu z$, then
\begin{align*}
    A_\bot A^{-1} z 
    = & ~ A_\bot \mu z \\
    = & ~ \mu (A_\bot z) \\
    = & ~ \lambda \mu z,
\end{align*}
where the first step follows from our assumption, the second step follows from $\mu$ is a real number and a real number multiplying a matrix is commutative and follows from the associative property, and the third step follows from our assumption.

Thus, we have $\|A_\bot A^{-1}\|=\|A_\bot\|\|A^{-1}\|$, and we have proved the assertion.
\end{proof}

\begin{lemma}\label{lem:tan_is_sin_cos}
Let $X, Y\in \R^{n\times k}$ be orthogonal matrices, then 
\begin{align*}
    \tan \theta(Y,X) = & ~ \frac{\sin \theta(Y,X)}{\cos \theta(Y,X)}.
\end{align*}
\end{lemma}

\begin{proof}
Due to Lemma~\ref{lem:trig_structural}, we have $(Y^\top X)_\bot=Y^\top_\bot X$. Thus, $\tan\theta(Y,X)=\|(Y^\top X)_\bot (Y^\top X)^{-1}\|$. The proof then follows straightforwardly from Lemma~\ref{lem:perp_inv_singular}.
\end{proof}

\begin{lemma}\label{lem:sin^2_and_cos^2_is_1}
Let $X, Y\in \R^{n\times k}$ be orthogonal matrices, then $\sin^2\theta(Y, X) + \cos^2\theta(Y,X) =1$.
\end{lemma}
\begin{proof}
Recall that $\cos\theta(Y,X)=\frac{1}{\|(Y^\top X)^{-1}\|}$ and $\sin\theta(Y,X)=\|Y_\bot^\top X\|$, by Lemma~\ref{lem:trig_structural}, we know that $(Y^\top X)_\bot=Y^\top_\bot X$, therefore $\sin\theta(Y,X)=\|(Y^\top X)_\bot \|$. Let $A:=Y^\top X$, by Lemma~\ref{lem:perp_inv_singular}, we know that $A_\bot$ and $A^{-1}$ have the same singular vectors, or equivalently, the singular vector realizing $\|A_\bot\|$ corresponds to the smallest singular value of $A$. Let $z\in \R^k$ be the unit singular vector with singular value $\|A_\bot\|$, then
\begin{align*}
    z^\top A^\top Az+z^\top A_\bot^\top A_\bot z = & ~ 1, \\
    \|A_\bot\|^2+\sigma_{\min}^2(A) = & ~ 1, \\
    \cos^2\theta(Y,X)+\sin^2\theta(Y,X) = & ~ 1.
\end{align*}
This completes the proof.
\end{proof}

\subsection{Angle is close}

\begin{lemma}
Let $\epsilon \in (0,0.1)$
Let $x$ denote a unit vector, i.e., $\| x \|_2 = 1$.

Let $z = (x+y) / \| x + y\|_2$.

If $ \| y \|_2 \leq \epsilon \cdot \| x \|_2$, then
\begin{align*}
\sqrt{1-\langle x, z \rangle^2} \leq 2 \sqrt{\epsilon} 
\end{align*}
\end{lemma}
\begin{proof}
We have
\begin{align*}
 \| x + y \|_2 
 \geq & ~\| x \|_2 - \| y \|_2 \\
\geq & ~ 1- \epsilon
\end{align*}
where the first step follows from triangle inequality.

We also have
\begin{align}\label{eq:x_y_bound}
 \| x + y \|_2 
 \leq & ~\| x \|_2 + \| y \|_2 \notag \\
\leq & ~ 1 + \epsilon
\end{align}

We have
\begin{align}\label{eq:1_minus_eps}
(1- \epsilon)^2 \geq 1- 2\epsilon
\end{align}

We also have
\begin{align}\label{eq:1_plus_eps}
\frac{1}{(1+\epsilon)^2} \geq 1- 3 \epsilon
\end{align}
where $\epsilon \in (0,0.1)$.

Combining Eq.~(\ref{eq:1_minus_eps}) and Eq.~(\ref{eq:1_plus_eps}), we have
\begin{align}\label{eq:1_2eps}
\frac{1}{(1+\epsilon)^2} \cdot ( 1 - \epsilon )^2 \geq &~ (1-2\epsilon) \cdot (1-3 \epsilon) \notag\\
=&~ 1 -5 \epsilon + 6 \epsilon^2 \notag\\
\geq &~  1 - 5\epsilon + \epsilon \notag\\
=&~ 1 - 4 \epsilon 
\end{align}
where the first step follows from Eq.~(\ref{eq:1_minus_eps}) and Eq.~(\ref{eq:1_plus_eps}) and the rest of them follow from simple algebra.

Finally, we have
\begin{align*}
1 - \langle x, z \rangle^2
= & ~ 1 - \langle x, \frac{x+y}{ \| x + y \|_2 } \rangle^2 \\
= & ~ 1 - \frac{1}{\| x + y \|_2^2} \langle x , x+y \rangle^2 \\
= & ~ 1 - \frac{1}{\| x + y \|_2^2} \cdot ( \| x \|_2^2 + \langle x , y \rangle )^2 \\
= & ~  1 - \frac{1}{\| x + y \|_2^2} \cdot ( 1 + \langle x , y \rangle )^2 \\
\leq & ~ 1 - \frac{1}{(1+\epsilon)^2} \cdot ( 1 + \langle x , y \rangle )^2 \\
\leq & ~ 1 - \frac{1}{(1+\epsilon)^2} \cdot ( 1 - \epsilon )^2 \\
\leq & ~ 1 -  (1- 4 \epsilon) \\
= & ~ 4 \epsilon,
\end{align*}
where the first step follow the definition of $z$, the second step follows from the reorganization, the third step follows from the definition of inner product, the fourth step follows from $\| x \|_2 = 1$, the fifth step follows from Eq.~(\ref{eq:x_y_bound}), the sixth step follows from  $1+\langle x,y\rangle \geq 1 - | \langle x , y \rangle | \geq 1 - \| x \|_2 \cdot \| y \|_2 \geq 1- \epsilon$, the seventh step follows from Eq.~(\ref{eq:1_2eps}) and the last step follows from simple algebra.

\end{proof}

\section{Function Approximations}\label{sec:function_approx}

We first we show the function approximation for two operators in Section \ref{sec:fun_app_app}, which means that there are two functions. Then we show the function approximations for four operators in Section \ref{sec:fun_app_operators_app}.

\subsection{Function Approximations for Two Operators}\label{sec:fun_app_app}

\begin{lemma}
Let $f_1 : \R^d \rightarrow \R^d$ and let $f_2 : \R^d \rightarrow \R^d$.

Assume the the following conditions
\begin{itemize}
    \item Condition 1a. $f_1$ is a linear function 
    \item Condition 1b. $\| f_1(x) \|_2 \leq \epsilon_1 \| x \|_2$ ($f_1$ is shrinking)
    \item Condition 1c. $\| f_1 (x)  -f_1(y) \|_2 \leq L_1 \| x - y \|_2$ ($f_1$ is Lipschitz)
    \item Condition 2a. $f_2$ is a linear function 
    \item Condition 2b. $\| f_2(x) \|_2 \leq \epsilon_2 \| x \|_2$  ($f_2$ is shrinking)
    \item Condition 2c. $\| f_2 (x)  -f_2(y) \|_2 \leq L_2 \| x - y \|_2$ ($f_2$ is Lipschitz)
\end{itemize}

We define three functions
\begin{itemize}
\item 
\begin{align*}
g_1 (x) = : & ~ (I + f_1) \cdot (I + f_2) (x)  \\
= & ~ x + f_2(x) + f_1 ( x + f_2(x) )
\end{align*}
\item
\begin{align*}
g_2 (x) = : & ~ (I + f_2) \cdot (I + f_1) (x)  \\
= & ~ x + f_1(x) + f_2 ( x + f_1(x) )
\end{align*}
\item
\begin{align*}
g_3 (x) = : & ~ (I + f_1 + f_2) (x)  \\
= & ~ x + f_1(x) + f_2(x)
\end{align*}
\end{itemize}
Then we can show that
\begin{itemize}
    \item Part 1. $\| g_1 (x) - g_2 (x) \|_2 \leq 2 \epsilon_1 \epsilon_2 \| x \|_2$(if $f_1$ and $f_2$ are linear functions)
    \item Part 2. $\| g_1 (x) - g_2 (x) \|_2 \leq (\epsilon_2 \cdot L_1+ \epsilon_1 \cdot L_2) \| x \|_2  $  (if $f_1$ and $f_2$ are Lipschitz functions)
    \item Part 3. $\| g_1 (x) -g_3 (x) \|_2 \leq \epsilon_1 \epsilon_2 \| x \|_2$ (if $f_1$ is a linear function)
    \item Part 4. $\| g_1 (x) -g_3 (x) \|_2 \leq \epsilon_2 \cdot L_1 \| x \|_2$ (if $f_1$ is a Lipschitz function)
    \item Part 5.  $\| g_2 (x) -g_3 (x) \|_2 \leq \epsilon_1 \epsilon_2 \| x \|_2$ (if $f_2$ is a linear function)
    \item Part 6. $\| g_2 (x) -g_3 (x) \|_2 \leq \epsilon_1 \cdot L_2 \| x \|_2$ (if $f_2$ is a Lipschitz function)
\end{itemize}
\end{lemma}
\begin{proof}

{\bf Part 1.}

We have 
\begin{align*}
\| g_1 (x) - g_2 (x) \|_2 \leq &~ \|g_1(x) - g_3(x) \|_2 + \|g_3(x) - g_2(x)\|_2\\
\leq &~   \epsilon_1 \epsilon_2 \| x \|_2 + \epsilon_1 \epsilon_2 \| x \|_2 \\
= &~  2 \epsilon_1 \epsilon_2 \| x \|_2
\end{align*}
where the first step follows from triangular inequality, the second step follows from Part 3 and Part 5 and the last step follows from simple algebra.

{\bf Part 2.}

We have 
\begin{align*}
\| g_1 (x) - g_2 (x) \|_2 \leq &~ \|g_1(x) - g_3(x) \|_2 + \|g_3(x) - g_2(x)\|_2\\
\leq &~  \epsilon_2 \cdot L_1 \| x \|_2+ \epsilon_1 \cdot L_2 \| x \|_2 \\
= &~  (\epsilon_2 \cdot L_1+ \epsilon_1 \cdot L_2) \| x \|_2 
\end{align*}
where the first step follows from triangular inequality, the second step follows from Part 4 and Part 6 and the last step follows from simple algebra.

{\bf Part 3.}

We have 
\begin{align*}
\| g_1 (x) - g_3(x) \|_2
= & ~ \| f_1(x+ f_2(x) ) - f_1(x) \|_2 \\
= & ~ \| f_1 (x+ f_2(x) - x) \|_2 \\
= & ~ \| f_1 (f_2(x) ) \|_2 \\
\leq & ~ \epsilon_1 \cdot \| f_2(x) \|_2 \\
\leq & ~ \epsilon_1 \cdot \epsilon_2 \cdot \| x \|_2,
\end{align*}
where the first step follows from the definition of $g_1$ and $g_3$, the second step follows from the fact that $f_1$ is a linear function, the third step follows from simple algebra, the fourth step follows from Condition 1b and the last step follows from Condition 2b.

{\bf Part 4.}

\begin{align*}
\| g_1 (x) - g_3(x) \|_2
= & ~ \| f_1(x+ f_2(x) ) - f_1(x) \|_2 \\
\leq & ~L_1 \cdot \| x+ f_2(x) - x \|_2  \\
= & ~  L_1 \cdot \| f_2 (x) \|_2 \\
\leq & ~ L_1 \cdot \epsilon_2 \| x \|_2,
\end{align*}
where the first step follows from definition of $g_1$ and $g_3$, the second step follows from Condition 1c, the third step follows from simple algebra and the last step follows from Condition 2b. 
\end{proof}

{\bf Part 5.}

We have 
\begin{align*}
\| g_2 (x) - g_3(x) \|_2
= & ~ \| f_2(x+ f_1(x) ) - f_2(x) \|_2 \\
= & ~ \| f_2 (x+ f_1(x) - x) \|_2 \\
= & ~ \| f_2 (f_1(x) ) \|_2 \\
\leq & ~ \epsilon_2 \cdot \| f_1(x) \|_2 \\
\leq & ~ \epsilon_2 \cdot \epsilon_1 \cdot \| x \|_2,
\end{align*}
where the first step follows from the definition of $g_2$ and $g_3$, the second step follows from the fact that $f_2$ is a linear function, the third step follows from simple algebra, the fourth step follows from Condition 2b and the last step follows from Condition 1b.

{\bf Part 6.}

\begin{align*}
\| g_2 (x) - g_3(x) \|_2
= & ~ \| f_2(x+ f_1(x) ) - f_2(x) \|_2 \\
\leq & ~L_2 \cdot \| x+ f_1(x) - x \|_2  \\
= & ~  L_2 \cdot \| f_1 (x) \|_2 \\
\leq & ~ L_2 \cdot \epsilon_1 \| x \|_2,
\end{align*}
where the first step follows from definition of $g_1$ and $g_3$, the second step follows from Condition 2c, the third step follows from simple algebra and the last step follows from Condition 1b. 

\subsection{Function Approximations for Four Operators}\label{sec:fun_app_operators_app}

\begin{lemma}
For each $i \in [4]$, we assume the following conditions
\begin{itemize}
    \item i(a) $f_i$ is a linear function
    \item i(b) $\| f_i(x) \|_2 \leq \epsilon_i \| x \|_2$ ($f_i$ is shriking)
    \item i(c)  $\| f_i(x) - f_i(y) \|_2 \leq L_i \| x-y\|_2$ ($f_i$ is Lipschitz)
\end{itemize}
We define three functions
\begin{itemize}
    \item $g_1(x):= (I+ f_1) \cdot (I + f_2) \cdot (I + f_3) \cdot (I + f_4) (x)$
    \item $g_2(x): = (I+ f_1) \cdot (I + f_3) \cdot (I + f_2) \cdot (I + f_4) (x)$
    \item $g_3(x) := (I + f_1 + f_2 + f_3 + f_4) (x)$
\end{itemize}
Then, we can show that
\begin{itemize}
    \item Part 1. $\| g_1 (x) - g_2 (x ) \|_2 \leq 2 ( \epsilon_1 \epsilon_2+ \epsilon_1 \epsilon_3 + \epsilon_1 \epsilon_4
        +\epsilon_2\epsilon_3+ \epsilon_2 \epsilon_4
        +\epsilon_3 \epsilon_4 
         + \epsilon_1 \epsilon_2 \epsilon_3 + \epsilon_1 \epsilon_2 \epsilon_4 + \epsilon_1 \epsilon_3 \epsilon_4 +\epsilon_2 \epsilon_3 \epsilon_4 + \epsilon_1 \epsilon_2 \epsilon_3 \epsilon_4) \|x\|_2$ (if $f_i,~\forall i \in [4]$ are linear functions)
    \item Part 2. $\| g_1 (x) - g_2 (x ) \|_2  \leq (2L_1 \epsilon_2  + 2L_1 \epsilon_3  +2L_1 \epsilon_4   +
        L_2 \epsilon_3  + 2L_2  \epsilon_4  + 2L_3  \epsilon_4  +
        2L_1 \epsilon_2 \epsilon_3  +2L_1 \epsilon_2 \epsilon_4   +2L_1 \epsilon_3 \epsilon_4 +
         L_2 \epsilon_3 \epsilon_4  +
        2L_1 \epsilon_2 \epsilon_3 \epsilon_4 + 
         L_3 \epsilon_2 +
          L_3 \epsilon_2 \epsilon_4) \| x\|_2$ (if $f_i,~\forall i \in [4]$ are Lipschitz functions)
    \item Part 3. $\| g_1 (x) - g_3 (x ) \|_2 \leq ( \epsilon_1 \epsilon_2+ \epsilon_1 \epsilon_3 + \epsilon_1 \epsilon_4
        +\epsilon_2\epsilon_3+ \epsilon_2 \epsilon_4
        +\epsilon_3 \epsilon_4 
         + \epsilon_1 \epsilon_2 \epsilon_3 + \epsilon_1 \epsilon_2 \epsilon_4 + \epsilon_1 \epsilon_3 \epsilon_4 +\epsilon_2 \epsilon_3 \epsilon_4 + \epsilon_1 \epsilon_2 \epsilon_3 \epsilon_4) \|x\|_2$ (if $f_i,~\forall i \in [4]$ are linear functions)
    \item Part 4. $\| g_1 (x) - g_3 (x ) \|_2 \leq (L_1 \epsilon_2  + L_1 \epsilon_3  +L_1 \epsilon_4   +
        L_2 \epsilon_3  + L_2  \epsilon_4  + L_3  \epsilon_4  +
        L_1 \epsilon_2 \epsilon_3  +L_1 \epsilon_2 \epsilon_4   +L_1 \epsilon_3 \epsilon_4 +
        L_2 \epsilon_3 \epsilon_4  +
        L_1 \epsilon_2 \epsilon_3 \epsilon_4)  \|x \|_2   $ (if $f_i,~\forall i \in [4]$ are Lipschitz functions)
    \item Part 5. $\| g_2 (x) - g_3 (x ) \|_2\leq ( \epsilon_1 \epsilon_2+ \epsilon_1 \epsilon_3 + \epsilon_1 \epsilon_4
        +\epsilon_2\epsilon_3+ \epsilon_2 \epsilon_4
        +\epsilon_3 \epsilon_4 
         + \epsilon_1 \epsilon_2 \epsilon_3 + \epsilon_1 \epsilon_2 \epsilon_4 + \epsilon_1 \epsilon_3 \epsilon_4 +\epsilon_2 \epsilon_3 \epsilon_4 + \epsilon_1 \epsilon_2 \epsilon_3 \epsilon_4) \|x\|_2$ (if $f_i,~\forall i \in [4]$ are linear functions)
    \item Part 6.$\| g_2 (x) - g_3 (x ) \|_2\leq  (L_1 \epsilon_2 + L_1 \epsilon_3 + L_1 \epsilon_4 +
        L_2 \epsilon_4+
        L_3 \epsilon_2 + L_3 \epsilon_4+
        L_1 \epsilon_2 \epsilon_3 + L_1 \epsilon_2 \epsilon_4 + L_1 \epsilon_3 \epsilon_4+
        L_3 \epsilon_2 \epsilon_4+
        L_1  \epsilon_2 \epsilon_3 \epsilon_4) \| x\|_2\\$ (if $f_i,~\forall i \in [4]$ are Lipschitz functions)
\end{itemize}
\end{lemma}
\begin{proof}

{\bf Part 1.}

We have 
\begin{align*}
\| g_1 (x) - g_2 (x) \|_2 \leq &~ \|g_1(x) - g_3(x) \|_2 + \|g_3(x) - g_2(x)\|_2\\
\leq &~   2 ( \epsilon_1 \epsilon_2+ \epsilon_1 \epsilon_3 + \epsilon_1 \epsilon_4
        +\epsilon_2\epsilon_3+ \epsilon_2 \epsilon_4
        +\epsilon_3 \epsilon_4 
         + \epsilon_1 \epsilon_2 \epsilon_3 + \epsilon_1 \epsilon_2 \epsilon_4 + \epsilon_1 \epsilon_3 \epsilon_4 +\epsilon_2 \epsilon_3 \epsilon_4 + \epsilon_1 \epsilon_2 \epsilon_3 \epsilon_4) \|x\|_2
\end{align*}
where the first step follows from triangular inequality and the last step follows from Part 3 and Part 5.

{\bf Part 2.}

We have 
\begin{align*}
\| g_1 (x) - g_2 (x) \|_2 \leq &~ \|g_1(x) - g_3(x) \|_2 + \|g_3(x) - g_2(x)\|_2\\
\leq &~    (2L_1 \epsilon_2  + 2L_1 \epsilon_3  +2L_1 \epsilon_4   +
        L_2 \epsilon_3  + 2L_2  \epsilon_4  + 2L_3  \epsilon_4  +
        2L_1 \epsilon_2 \epsilon_3  +2L_1 \epsilon_2 \epsilon_4   +2L_1 \epsilon_3 \epsilon_4 +\\
        &~ L_2 \epsilon_3 \epsilon_4  +
        2L_1 \epsilon_2 \epsilon_3 \epsilon_4 + 
         L_3 \epsilon_2 +
          L_3 \epsilon_2 \epsilon_4) \| x\|_2
\end{align*}
where the first step follows from triangular inequality and the last step follows from Part 4 and Part 6.

    {\bf  Part 3.}
    We have
    \begin{align*}
       \| g_1 (x) - g_3 (x ) \|_2 =  &~ \| (I+ f_1) \cdot (I + f_2) \cdot (I + f_3) \cdot  (x + f_4(x) ) - (I + f_1 + f_2 + f_3 + f_4) (x) )  \|_2 \\
        =  &~ \| (x + f_4(x)  + f_3(x + f_4(x)) + f_2(x + f_4(x)  +\\
        &~ f_3(x + f_4(x)))+ f_1(x + f_4(x)  + f_3(x + f_4(x) ) + f_2(x + f_4(x)  + f_3(x + f_4(x) )))  \\
        &~ - (I + f_1 + f_2 + f_3 + f_4) (x) )  )  \|_2 \\
        =  &~ \|    f_3( f_4(x)) + f_2( f_4(x)  + f_3(x + f_4(x)))+ f_1( f_4(x)  +\\
        &~ f_3(x + f_4(x) ) + f_2(x + f_4(x)  + f_3(x + f_4(x) )) )  )  \|_2 \\
        =  &~ \|    f_3( f_4(x)) + f_2( f_4(x))  + f_2(f_3(x)) + f_2(f_3(f_4(x)))+ \\
        &~  f_1(f_4(x))  + f_1(f_3(x)) + f_1(f_3(f_4(x)))  + f_1(f_2(x)) + f_1(f_2(f_4(x)))  \\
        &~ + f_1(f_2(f_3(x))) + f_1(f_2(f_3(f_4(x) )))  )  \|_2 \\
        \leq  &~ \| f_3( f_4(x)) \|_2 + \|f_2( f_4(x))\|_2  + \|f_2(f_3(x))\|_2 + \|f_2(f_3(f_4(x)))\|_2+ \\
        &~  \|f_1(f_4(x))\|_2  +\| f_1(f_3(x))\|_2 + \|f_1(f_3(f_4(x)))\|_2  + \|f_1(f_2(x))\|_2 + \|f_1(f_2(f_4(x)))\|_2  + \\
        &~ \|f_1(f_2(f_3(x)))\|_2 + \|f_1(f_2(f_3(f_4(x) )))\|_2  \\
        \leq &~ (\epsilon_3 \epsilon_4 + \epsilon_2 \epsilon_4+\epsilon_2\epsilon_3+\epsilon_2 \epsilon_3 \epsilon_4 + \epsilon_1 \epsilon_4 + \epsilon_1 \epsilon_3 + \epsilon_1 \epsilon_3 \epsilon_4 + \epsilon_1 \epsilon_2 + \epsilon_1 \epsilon_2 \epsilon_4 + \epsilon_1 \epsilon_2 \epsilon_3 + \epsilon_1 \epsilon_2 \epsilon_3 \epsilon_4) \|x\|_2\\
        = &~ ( \epsilon_1 \epsilon_2+ \epsilon_1 \epsilon_3 + \epsilon_1 \epsilon_4
        +\epsilon_2\epsilon_3+ \epsilon_2 \epsilon_4
        +\epsilon_3 \epsilon_4 
         + \epsilon_1 \epsilon_2 \epsilon_3 + \epsilon_1 \epsilon_2 \epsilon_4 + \epsilon_1 \epsilon_3 \epsilon_4 +\epsilon_2 \epsilon_3 \epsilon_4 + \epsilon_1 \epsilon_2 \epsilon_3 \epsilon_4) \|x\|_2,
    \end{align*}
where the first step follows from the definition of $g_1$ and $g_3$, the second step follows from simple algebra, the third step follows from reorganization, the fourth step follows from the fact that all $f_i, \forall i \in [4]$ are linear function, the fifth step follows from triangular inequality, the sixth step follows from $i(b)$ and the last step follows from reorganization.

    {\bf  Part 4.}
    We have
    \begin{align*}
       \| g_1 (x) - g_3 (x ) \|_2 =  &~ \| (I+ f_1) \cdot (I + f_2) \cdot (I + f_3) \cdot  (x + f_4(x) ) - (I + f_1 + f_2 + f_3 + f_4) (x) )  \|_2 \\
        =  &~ \| x + f_4(x)  + f_3(x + f_4(x)) + f_2(x + f_4(x)  +\\
        &~ f_3(x + f_4(x)))+ f_1(x + f_4(x)  + f_3(x + f_4(x) ) + f_2(x + f_4(x)  + f_3(x + f_4(x) )))  \\
        &~ - (I + f_1 + f_2 + f_3 + f_4) (x)   )  \|_2 \\
        =  &~ \|  f_3(x + f_4(x)) + f_2(x + f_4(x)  + f_3(x + f_4(x)))\\
        &~ + f_1(x + f_4(x)  + f_3(x + f_4(x) ) + f_2(x + f_4(x)  + f_3(x + f_4(x) )))  \\
        &~ -  f_1 (x) - f_2 (x) - f_3 (x) )  \|_2 \\     
        =  &~ \|  f_3(x + f_4(x))- f_3 (x)  + f_2(x + f_4(x)  + f_3(x + f_4(x)))- f_2 (x)\\
        &~ + f_1(x + f_4(x)  + f_3(x + f_4(x) ) + f_2(x + f_4(x)  + f_3(x + f_4(x) )))-  f_1 (x)   \| \\
        \leq  &~  L_3 \|  x + f_4(x)- x  \|_2   + L_2 \| x + f_4(x)  + f_3(x + f_4(x))- x \|_2\\
        &~ + L_1 \| x + f_4(x)  + f_3(x + f_4(x) ) + f_2(x + f_4(x)  + f_3(x + f_4(x) ))-  x\|_2    \\
        \leq  &~  L_3 \|  f_4(x)\|_2   + L_2 \|  f_4(x)  + f_3(x + f_4(x)) \|_2\\
        &~ + L_1 \| f_4(x)  + f_3(x + f_4(x) ) + f_2(x + f_4(x)  + f_3(x + f_4(x) ))\|_2    \\
        \leq  &~  L_3  \epsilon_4 \|x\|_2   + L_2  \epsilon_4 \|x\|_2 + L_2 \epsilon_3 \| x + f_4(x) \|_2\\
        &~ + L_1 \epsilon_4 \| x\|_2  + L_1 \epsilon_3 \|x + f_4(x) \|_2 + L_1 \epsilon_2 \|x + f_4(x)  + f_3(x + f_4(x) )\|_2    \\
        \leq  &~  L_3  \epsilon_4 \|x\|_2   + L_2  \epsilon_4 \|x\|_2 + L_2 \epsilon_3 \| x\| +L_2 \epsilon_3 \epsilon_4 \|  x \|_2\\
        &~ + L_1 \epsilon_4 \| x\|_2  + L_1 \epsilon_3 \|x\|_2 +L_1 \epsilon_3 \epsilon_4 \| x \|_2 + L_1 \epsilon_2 \|x\|_2 + L_1 \epsilon_2 \epsilon_4 \| x\|_2  + L_1 \epsilon_2 \epsilon_3 \|x + f_4(x) \|_2    \\
        \leq  &~  L_3  \epsilon_4 \|x\|_2  + L_2  \epsilon_4 \|x\|_2 + L_2 \epsilon_3 \| x\| +L_2 \epsilon_3 \epsilon_4 \|  x \|_2\\
        &~ + L_1 \epsilon_4 \| x\|_2  + L_1 \epsilon_3 \|x\|_2 +L_1 \epsilon_3 \epsilon_4 \| x \|_2 + L_1 \epsilon_2 \|x\|_2 + L_1 \epsilon_2 \epsilon_4 \| x\|_2  + L_1 \epsilon_2 \epsilon_4 \|x\|_2 +  L_1 \epsilon_2 \epsilon_3 \epsilon_4\|x \|_2    \\
        =  &~  (L_3  \epsilon_4  + L_2  \epsilon_4  + L_2 \epsilon_3  +L_2 \epsilon_3 \epsilon_4  + L_1 \epsilon_4   + L_1 \epsilon_3  +L_1 \epsilon_3 \epsilon_4 + L_1 \epsilon_2  + L_1 \epsilon_2 \epsilon_4   + L_1 \epsilon_2 \epsilon_3  +  L_1 \epsilon_2 \epsilon_3 \epsilon_4)  \|x \|_2   \\
        =  &~  (L_1 \epsilon_2  + L_1 \epsilon_3  +L_1 \epsilon_4   +
        L_2 \epsilon_3  + L_2  \epsilon_4  + L_3  \epsilon_4  +
        L_1 \epsilon_2 \epsilon_3  +L_1 \epsilon_2 \epsilon_4   +L_1 \epsilon_3 \epsilon_4 +
        L_2 \epsilon_3 \epsilon_4  +
        L_1 \epsilon_2 \epsilon_3 \epsilon_4)  \|x \|_2   \\
    \end{align*}
where the first step follows from the definition of $g_1$ and $g_3$, 
the second step follows from simple algebra, 
the third step follows from simple algebra, 
the fourth step follows from reorganization, 
the fifth step follows from the fact that all $f_i, \forall i \in [4]$ are Lipschitz functions,
the sixth step follows from simple algebra, 
the seventh step follows from $i(b)$,
the eighth step follows from triangular inequality,
the ninth step follows from $i(b)$,
the tenth step follows from $i(b)$
and the last step follows from reorganization. 

    {\bf  Part 5.}
    We have
    \begin{align*}
       \| g_2 (x) - g_3 (x ) \|_2 =  &~ \| (I+ f_1) \cdot (I + f_3) \cdot (I + f_2) \cdot  (x + f_4(x) ) - (I + f_1 + f_2 + f_3 + f_4) (x) )  \|_2 \\
        =  &~ \| (x + f_4(x)  + f_2(x + f_4(x)) + f_3(x + f_4(x)  +\\
        &~ f_2(x + f_4(x)))+ f_1(x + f_4(x)  + f_2(x + f_4(x) ) + f_3(x + f_4(x)  + f_2(x + f_4(x) )))  \\
        &~ - (I + f_1 + f_2 + f_3 + f_4) (x) )  )  \|_2 \\
        =  &~ \|  f_2( f_4(x)) + f_3( f_4(x))  +\\
        &~ f_3(f_2(x + f_4(x)))+ f_1( f_4(x))  + f_1 (f_2(x + f_4(x) )) + f_1(f_3(x + f_4(x)  + f_2(x + f_4(x) )))  )  \|_2 \\
        \leq  &~ ( \epsilon_2\epsilon_4 + \epsilon_3\epsilon_4  + \epsilon_3 \epsilon_2 + \epsilon_3 \epsilon_2 \epsilon_4 +\epsilon_1 \epsilon_4 + \epsilon_1 \epsilon_2 + \epsilon_1 \epsilon_2 \epsilon_4 + \epsilon_1 \epsilon_3 + \epsilon_1 \epsilon_3 \epsilon_4 + \epsilon_1 \epsilon_3 \epsilon_2 + \epsilon_1 \epsilon_3 \epsilon_2 \epsilon_4 ) \|x\|_2 \\
        = &~ ( \epsilon_1 \epsilon_2+ \epsilon_1 \epsilon_3 + \epsilon_1 \epsilon_4
        +\epsilon_2\epsilon_3+ \epsilon_2 \epsilon_4
        +\epsilon_3 \epsilon_4 
         + \epsilon_1 \epsilon_2 \epsilon_3 + \epsilon_1 \epsilon_2 \epsilon_4 + \epsilon_1 \epsilon_3 \epsilon_4 +\epsilon_2 \epsilon_3 \epsilon_4 + \epsilon_1 \epsilon_2 \epsilon_3 \epsilon_4) \|x\|_2,
    \end{align*}
where the first step follows from the definition of $g_2$ and $g_3$, 
the second step follows from simple algebra, 
the third step follows from the fact that all $f_i, \forall i \in [4]$ are linear function, 
the fourth step follows from triangular inequality and $i(b)$,
and the last step follows from reorganization.

    {\bf  Part 6.}
    We have
    \begin{align*}
       \| g_2 (x) - g_3 (x ) \|_2 =  &~ \| (I+ f_1) \cdot (I + f_3) \cdot (I + f_2) \cdot  (x + f_4(x) ) - (I + f_1 + f_2 + f_3 + f_4) (x) )  \|_2 \\
        =  &~ \| (x + f_4(x)  + f_2(x + f_4(x)) + f_3(x + f_4(x)  + f_2(x + f_4(x))) \\
        &~ + f_1(x + f_4(x)  + f_2(x + f_4(x) ) + f_3(x + f_4(x)  + f_2(x + f_4(x) )))  \\
        &~ - (I + f_1 + f_2 + f_3 + f_4) (x) )  )  \|_2 \\
        =  &~ \| f_2(x + f_4(x))- f_2(x) + f_3(x + f_4(x)  + f_2(x + f_4(x))) - f_3(x)  \\
        &~ + f_1(x + f_4(x)  + f_2(x + f_4(x) ) + f_3(x + f_4(x)  + f_2(x + f_4(x) )))- f_1(x) \|_2  \\
        \leq  &~ \| f_2(x + f_4(x))- f_2(x) \|_2  +  \|f_3(x + f_4(x)  + f_2(x + f_4(x))) - f_3(x) \|_2 \\
        &~ + \|f_1(x + f_4(x)  + f_2(x + f_4(x) ) + f_3(x + f_4(x)  + f_2(x + f_4(x) )))- f_1(x) \|_2  \\
        \leq  &~ L_2 \epsilon_4 \| x \|_2  +  L_3 \epsilon_4 \| x \|_2   + L_3 \epsilon_2 \| x + f_4(x)  \|_2 \\
        &~ + L_1 \epsilon_4 \| x\|_2  + L_1 \epsilon_2 \|x + f_4(x) \|_2 + L_1 \epsilon_3 \|x + f_4(x)  + f_2(x + f_4(x) )\|_2  \\
        \leq  &~ L_2 \epsilon_4 \| x \|_2  +  L_3 \epsilon_4 \| x \|_2   + L_3 \epsilon_2 \| x\|_2  +  L_3 \epsilon_2 \epsilon_4 \|x\|_2 \\
        &~ + L_1 \epsilon_4 \| x\|_2  + L_1 \epsilon_2 \|x\|_2 +  L_1 \epsilon_2 \epsilon_4\| x \|_2 + L_1 \epsilon_3 \|x\| +L_1 \epsilon_3 \epsilon_4 \| x\|_2  + L_1 \epsilon_3 \epsilon_2 \|x\|_2  + L_1 \epsilon_3 \epsilon_2 \epsilon_4 \|x \|_2  \\
        = &~  (L_1 \epsilon_2 + L_1 \epsilon_3 + L_1 \epsilon_4 +
        L_2 \epsilon_4+
        L_3 \epsilon_2 + L_3 \epsilon_4+
        L_1 \epsilon_2 \epsilon_3 + L_1 \epsilon_2 \epsilon_4 + L_1 \epsilon_3 \epsilon_4+
        L_3 \epsilon_2 \epsilon_4+
        L_1  \epsilon_2 \epsilon_3 \epsilon_4) \| x\|_2\\
    \end{align*}
where the first step follows from the definition of $g_2$ and $g_3$, 
the second step follows from simple algebra, 
the third step follows from reorganization, 
the fourth step follows from triangular inequality, 
the fifth step follows from the fact that all $f_i, \forall i \in [4]$ are Lipschitz functions and $i(b)$,
the sixth step follows from triangular inequality,
and the last step follows from reorganization.

\end{proof}

\section{Nearest Neighbor Search Data Structure}\label{sec:nearest_neighbor}


We use the reduction-based approximate $\maxip$ method with {\lsh} data-structure to achieve sublinear iteration cost. 
Note that we choose this method due to its clear theoretical guarantee on the retrieval results. It is well-known that an {\lsh} data-structures is used for approximate nearest neighbor problem. 
The following definition of approximate nearest neighbor search is very standard in literature~\cite{am93,im98,diim04,ainr14,ailrs15,ar15,iw18,alrw17,air18,dirw19,ccd+20,pmlr-v162-li22m,li2019re}.

\subsection{\texorpdfstring{{\lsh}}{~} and \texorpdfstring{$\maxip$}{~}}

We start with the defining the Approximate  Nearest Neighbor ($\ann$) problem~\cite{am93,im98,diim04,ainr14,ailrs15,ar15,iw18,alrw17,air18,dirw19,ccd+20} as:
\begin{definition}[Approximate Nearest  Neighbor ($\ann$)]\label{def:ann:formal}
Let $\ov{c} >1$ and $r \in (0,2)$ denote two parameters.  Given an $n$-vector set $Y \subset \mathbb{S}^{d-1}$ on a unit sphere, the objective of the $(\ov{c},r)$-Approximate Nearest Neighbor ($\ann$) is to construct a data structure that, for any query $x \in \mathbb{S}^{d-1}$ such that $\min_{y\in Y}\| y - x \|_2 \leq r$, it returns a vector $z$ from $Y$ that satisfies $\| z - x \|_2 \leq \ov{c} \cdot r$.
\end{definition}

The $\ann$ problem can be solved via locality sensitive hashing ({\lsh})~\cite{im98,diim04,iw18}. In this paper, we use the standard definitions of {\lsh} (see Indyk and Motwani~\cite{im98}).
\begin{definition}[Locality Sensitive Hashing]
Let $\ov{c}>1$ denote a parameter. Let $p_1, p_2\in (0,1)$ denote two parameters and $p_1 > p_2 $. We say a function family $\mathcal{H}$ is $(r,\ov{c} \cdot r,p_1,p_2)$-sensitive if and only if, for any vectors $x,y \in \R^d$, for any $h$ chosen uniformly at random from $\mathcal{H}$, we have:
\begin{itemize}
    \item if $\| x-y\|_2 \leq r$, then $\Pr_{h\sim {\cal H}} [ h(x)=h(y) ] \geq p_1$,
    \item if $ \|x-y\|_2 \geq \ov{c} \cdot r$, then $\Pr_{h\sim {\cal H}} [ h(x)=h(y) ] \leq p_2$.
\end{itemize}
\end{definition}

Next, we show that {\lsh} solves $\ann$ problem with sublinear query time complexity.

\begin{theorem}[Andoni, Laarhoven, Razenshteyn and Waingarten \cite{alrw17}]\label{thm:ar17:formal}
Let $\ov{c} > 1$ and $r \in (0,2)$ denote two parameters. One can solve $(\ov{c},r)$-$\ann$ on a unit sphere in query time $O(d \cdot n^{\rho})$ using preprocessing time $O(dn^{1+o(1)})$ and space $O(n^{1+o(1)} + d n)$, where $\rho = \frac{2}{\ov{c}^2} -\frac{1}{\ov{c}^4}+o(1)$.
\end{theorem}
Here we write $o(1)$ is equivalent to $O(1/\sqrt{\log n})$. Note that we could reduce $d$ to $n^{o(1)}$ with Johnson–Lindenstrauss Lemma~\cite{jl84}. Besides, we could achieve better $\rho$ using {\lsh} in~\cite{ar15} if we allowed to have more proprocessing time.

In this work, we focus on a well-known problem in computational complexity: approximate $\maxip$. In this work, we follow the standard notation in \cite{c18} and define the approximate $\maxip$ problem as follows:

\begin{definition}[Approximate $\maxip$]\label{def:approximate_maxip}
Let $c \in (0,1)$ and $\tau \in (0,1)$ denote two parameters.
Given an $n$-vector dataset $Y \subset \mathbb{S}^{d-1}$ on a unit sphere, the objective of the $(c,\tau)$-{$\maxip$} is to construct a data structure that, given a query $x \in \mathbb{S}^{d-1}$ such that $\max_{y\in Y}\langle x , y \rangle \geq \tau$, it retrieves a vector $z$ from $Y$ that satisfies $\langle x , z \rangle \geq c \cdot \max_{y \in Y} \langle x,y \rangle$.
\end{definition}

In many applications, it is more convenient to doing inner product search in a transformed/projected space compared to doing inner product search in the original space. Thus, we propose the following definitions (Definition~\ref{def:projected_maxip} and Definition~\ref{def:proj_approximate_maxip})
\begin{definition}[Projected $\maxip$]\label{def:projected_maxip} 
Let $\phi, \psi: \R^d \rightarrow \R^k$ denote two transforms. Given a data set $Y\subseteq \R^d$ and a point $x\in\R^d$, we define $(\phi, \psi)$-$\maxip$ as follows:
\begin{align*}
    (\phi, \psi)\text{-}\maxip (x,Y) := \max_{y \in Y} \langle \phi(x),\psi(y) \rangle
\end{align*}
\end{definition}

\begin{definition}[Projected approximate $\maxip$]\label{def:proj_approximate_maxip}
Let $\phi, \psi: \R^d \rightarrow \R^k$ denote two transforms. Given an $n$-vector dataset $Y \subset \R^d $ so that $\psi(Y) \subset \mathbb{S}^{k-1}$, the goal of the $(c,\phi, \psi,\tau)$-{$\maxip$} is to construct a data structure that, given a query $x\in \R^d$ and $\phi(x) \in \mathbb{S}^{k-1}$ such that $\max_{y\in Y}\langle \phi(x) , \psi(y) \rangle \geq \tau$, it retrieves a vector $z \in Y$ that satisfies $\langle \phi(x) , \psi(z) \rangle \geq c \cdot (\phi, \psi)\text{-}\maxip (x,Y)$.
\end{definition}

\subsection{Connections}
\begin{fact}\label{fac:x_x_eps}
Let $\widetilde{x}$ denote the vector that $\langle \widetilde{x}, x \rangle \geq 1-\frac{1}{2}\epsilon^2$, where both $\widetilde{x}$ and $x$ are unit vectors. We have
    \begin{align*}
        \| \widetilde{x} - x \|_2 \leq \epsilon
    \end{align*}
\end{fact}
\begin{proof}
    \begin{align*}
        \| \widetilde{x} - x \|_2 = &~ (\| \widetilde{x}\|_2^2 + \|x\|_2^2 - 2 \langle x,\widetilde{x} \rangle )^{1/2}\\
        = &~ (2 - 2\langle x,\widetilde{x} \rangle)^{1/2} \\
        \leq &~ (2 - 2(1-\frac{1}{2}\epsilon^2))^{1/2}\\
        = &~ \epsilon
    \end{align*}
    Now, we complete the proof.
\end{proof}
\begin{lemma}
Let $\widetilde{x}$ denote the vector that $\langle \widetilde{x}, x \rangle \geq 1-\frac{1}{2}\epsilon^2$, where both $\widetilde{x}$ and $x$ are unit vectors. Let $0.01 c \cdot \tau > \epsilon$. 

Suppose there is a $z \in Y$, where $\|z\|_2 = 1$, such that
\begin{align*}
\langle x, z \rangle \geq c \cdot \max_{y \in Y} \langle x , y \rangle
\end{align*}
Note that $\max_{y \in Y}\langle x , y \rangle \geq \tau$.
Then, we can find a $z \in Y$ such that 
\begin{align*}
\langle \widetilde{x}, z \rangle \geq \frac{1}{2} c  \cdot \max_{y \in Y} \langle x , y \rangle
\end{align*}
\end{lemma}
\begin{proof}
We have
\begin{align*}
\langle \widetilde{x}, z \rangle 
= & ~ \langle x , z \rangle + \langle \widetilde{x}-x, z\rangle \\
\geq & ~ \langle x , z \rangle - | \langle \widetilde{x}-x, z\rangle | \\
\geq & ~  \langle x , z \rangle -  \| \widetilde{x} - x \|_2 \cdot \| z \|_2 \\
\geq & ~  \langle x , z \rangle -\epsilon \\
\geq & ~ c \cdot \max_{y \in Y} \langle x, y \rangle - \epsilon \\
\geq & ~ 0.99 \cdot c \cdot \max_{y \in Y} \langle x, y \rangle
\end{align*}
where the first step follows from simple algebra, the second step follows from the fact that $\langle x,y\rangle \geq -| \langle x , y \rangle |$, the third step follows from the property of inner product, the fourth step follows from Fact \ref{fac:x_x_eps}, the fifth step follows from $\langle x, z \rangle \geq c \cdot \max_{y \in Y} \langle x , y \rangle$ and the final step follows from the fact that $0.01 c \cdot \tau > \epsilon$.

\end{proof}

\subsection{Efficient Transformations}\label{sec:transform_intro}

We have learned from 
that $(c,\tau)$-{$\maxip$} on a unit sphere $\mathcal{S}^{d-1}$ using {\lsh} for {$\ann$}. Therefore, the next step is to transform the direction search procedure in iterative optimization algorithm into a $\maxip$ on a unit sphere. To achieve this, we formulate the direction search as a projected approximate $\maxip$ (see Definition~\ref{def:projected_maxip}). We start with presenting a pair of transformation $\phi_0,\psi_0:\R^{d} \rightarrow \R^{d+1}$ such that, given a function $g : \R^d \rightarrow \R$, for any $x,y$ in a convex set $\mathcal{K}$, we have
 
\begin{align}\label{eq:asym_trans_direct}
\phi_0 (x) := &  [\nabla g(x) ^\top, x^\top\nabla g(x)]^\top, ~~~
\psi_0(y) := [ -y^\top,1]^\top.
\end{align}

In this way, we show that
\begin{align}\label{eq:asym_trans_direct_res}
    \langle y-x,\nabla g(x) \rangle =&~-\langle  \phi_0(x) , \psi_0(y) \rangle,\notag\\
    \arg\min_{y\in Y} \langle y-x,\nabla g(x) \rangle=&~\arg\max_{y\in Y} \langle  \phi_0(x) , \psi_0(y) \rangle
\end{align}

Therefore, we could transform the direction search problem into a $\maxip$ problem.

Next, we present a standard transformations~\cite{ns15} that connects the $\maxip$ to $\ann$ in unit sphere. For any $x,y\in \R^d$, we propose transformation $\phi_1,\psi_1:\R^{d} \rightarrow \R^{d+2}$ such that
\begin{align}\label{eq:asym_trans_mips}
  \phi_1(x) =&~ \begin{bmatrix} (D_{x}^{-1}x)^\top & 0 & \sqrt{1-\|x D_{x}^{-1}\|_2^2}  \end{bmatrix}^\top      \notag\\
  \psi_1(y) =&~ \begin{bmatrix} (D_{y}^{-1}y)^\top & \sqrt{1-\|yD_{y}^{-1}\|_2^2} & 0 \end{bmatrix}^\top
\end{align}

Here $D_x$, $D_y$ are some constant that make sure both $x/D_x$ and $y/D_y$ have norms less than $1$. Under these transformations, both $\phi_1(x)$ and $\psi_1(y)$ have norm $1$ and $\arg\max_{y\in Y} \langle \phi_1(x),\psi_1(y)\rangle=\arg\max_{y\in Y} \langle x,y\rangle$.

Combining transformations in Eq.~(\ref{eq:asym_trans_direct}) and Eq.~(\ref{eq:asym_trans_mips}), we obtain query transform $\phi:\R^d\rightarrow \R^{d+3}$ with form $\phi(x)=\phi_1(\phi_0(x))$ and data transform $\phi:\R^d\rightarrow \R^{d+3}$ with form $\psi(y)=\psi_1(\psi_0(y))$. Using $\phi$ and $\psi$, we transform the direction search problem in optimization into a $\maxip$ in unit sphere. Moreover, given a set $Y\subset \R^d$ and a query $x\in \R^d$, the solution $z$ of $(c,\phi,\psi,\tau)$-$\maxip$ over $(x,Y)$ has the propriety that $\langle z-x,\nabla g(x)\rangle\leq c\cdot \min_{y\in Y} \langle y-x,\nabla g(x) \rangle$. Thus, we could approximate the direction search with {\lsh} based {$\maxip$} data-structure.  

Note that only $\maxip$ problem with positive inner product values could be solved by {\lsh}. We found the direction search problem naturally satisfies this condition. We show that if $g$ is convex, given a set $S \subset \R^d$, we have $\min_{s \in S } \langle \nabla g(x) ,  s-x \rangle\leq 0$ for any $x\in \mathcal{B}(S)$, where $\mathcal{B}$ is the convex hull of $S$. Thus, $\max_{y\in Y} \langle  \phi_0(x) , \psi_0(y) \rangle$ is non-negative following Eq.~(\ref{eq:asym_trans_direct_res}).



\subsection{Data Structures}\label{sec:data_structure}

In this section, we present a formal statement that solves $(c,\tau)$-$\maxip$ problem on unit sphere using {\lsh} for $(\ov{c},r)$-$\ann$.

\begin{theorem}[
]\label{coro:maxip_lsh_formal}
Let $c \in (0,1)$ and $\tau \in(0,1)$. Given a set of $n$-vector set $Y \subset {\cal S}^{d-1}$ on the unit sphere, there exists a data structure with $O(d n^{1+o(1)})$ preprocessing time and $O(n^{1+o(1)} + d n)$ space so that for any query $x \in {\cal S}^{d-1}$, we take $O(d\cdot n^{\rho})$ query time to retrieve the $(c,\tau)$-$\maxip$ of $x$ in $Y$ with probability at least $0.9$\footnote{It is obvious to boost probability from constant to $\delta$ by repeating the data structure $\log(1/\delta)$ times.}, where $\rho:=  \frac{2(1-\tau)^2}{(1-c\tau)^2}-\frac{(1-\tau)^4}{(1-c\tau)^4}+o(1)$
\end{theorem}

\begin{proof}
We know that $\|x-y\|_2^2= 2 - 2\langle x , y\rangle$ for all $x,y\in {\cal S}^{d-1}$. In this way,  if we have a {\lsh} data-structure for $(\ov{c}, r)$-$\ann$. It could be used to solve $(c, \tau)$-$\maxip$ with $\tau = 1-0.5 r^2$ and $c = \frac{1-0.5 \ov{c}^2 r^2}{1 - 0.5 r^2}$. Next, we write $\ov{c}^2$ as
\begin{align*}
\ov{c}^2= \frac{ 1 - c(1-0.5 r^2) }{0.5r^2} = \frac{1 - c \tau }{1-\tau} .
\end{align*}

Next, we show that if  the {\lsh} is initialized following Theorem~\ref{thm:ar17:formal}, it takes query time $O(d \cdot n^{\rho})$, space $O(n^{1+o(1)} + d n)$ and preprocessing time $O(dn^{1+o(1)})$ to solve $(c,\tau)$-$\maxip$ through solving $(\ov{c}, r)$-$\ann$,   where
\begin{align*}
    \rho =  \frac{2}{\ov{c}^2} -\frac{1}{\ov{c}^4} +o(1) = \frac{2(1-\tau)^2}{(1-c\tau)^2}-\frac{(1-\tau)^4}{(1-c\tau)^4}+o(1). 
\end{align*}
\end{proof}

In practice, $c$ is increasing as we set parameter $\tau$ close to $\maxip(x,Y)$. 
There is also another {\lsh} data structure~\cite{ar15} with longer preprocessing time and larger space that could solve the $(c, \tau)$-$\maxip$ with similar query time complexity. We refer readers to Section 8.2 in~\cite{ssx21} for more details\footnote{Recently, there a line of work that use fast $\maxip$ data structure to speedup the iterative-type optimization algorithms \cite{ssx21,sy23,qsw23,swy23}.}. Moreover, Corollary~\ref{coro:maxip_lsh_formal} could be applied to projected $\maxip$ problem.

\begin{theorem}[]\label{thm:proj_maxip_lsh}
Let $c \in (0,1)$ and $\tau \in(0,1)$. Let $\phi, \psi: \R^d \rightarrow \R^k$ denote two transforms.   Let ${\cal T}_{\phi}$ denote the time to compute $\phi(x)$ and ${\cal T}_{\psi}$ denote the time to compute $\psi(y)$. Given a set of $n$-points $Y\in \R^d$ with $\psi(Y) \subset {\cal S}^{k-1}$ on the sphere, one can construct a data structure with $O(dn^{1+o(1)}+{\cal T}_{\psi}n)$ preprocessing time and $O(n^{1+o(1)} + d n)$ space so that for any query $x \in \R^d$ with $\phi(x)\in {\cal S}^{k-1}$, we take query time complexity $O(d\cdot n^{\rho}+{\cal T}_{\phi})$ to solve $(c,\phi,\psi,\tau)$-$\maxip$ with respect to $(x,Y)$ with probability at least $0.9$, where $\rho:=  \frac{2(1-\tau)^2}{(1-c\tau)^2}-\frac{(1-\tau)^4}{(1-c\tau)^4}+o(1)$.
\end{theorem}

\begin{proof}

The preprocessing phase can be decomposed in two parts.
\begin{itemize}
    \item It takes $O({\cal T}_{\psi}n)$ time to transform every $y\in Y$ into $\psi(y)$.
    \item It takes $O(O(dn^{1+o(1)})$ time and $O(dn^{1+o(1)}+dn)$ to index every  $\psi(y)$ into {\lsh} using Theorem~\ref{coro:maxip_lsh_formal}.
\end{itemize}

The query phase can be decomposed in two parts.
\begin{itemize}
    \item It takes $O({\cal T}_{\phi})$ time to transform every $x\in \R^d$ into $\phi(x)$.
    \item It takes $O(d\cdot n^{\rho})$ time perform query for  $\phi(x)$ in {\lsh} using Theorem~\ref{coro:maxip_lsh_formal}.
\end{itemize}
\end{proof}

\section{Self-attention layer as a clustering algorithm}
\label{sec:clustering understanding}
The self-attention layer in the Transformer looks like mean-shift clustering. Suppose $\{(\vx_j, \vv_j)\}$ are a bunch of key and value pairs and $\vq$ is the query. Note that $\vq=W_q\vx$, $\vk=W_k \vx$ and $\vv=W_v \vx$ are computed by three projection matrices $W_k$, $W_q$ and $W_v$ from a common $\vx$. Then from self-attention we have:
\begin{equation}
    \vv = \sum_j p_j \vv_j = \frac{\sum_j \exp(\vx^\t W_q^\t W_k \vx_j) W_v \vx_j}{\sum_j \exp(\vx^\t W_q^\t W_k \vx_j)} = W_v \frac{\sum_j \exp(\vx^\t W_q^\t W_k \vx_j) \vx_j}{\sum_j \exp(\vx^\t W_q^\t W_k \vx_j)} 
\end{equation}
where $\sim(\vq, \vk_j) := \exp(\vq^\t \vk_j) = \exp(\vx^\t W_q^\t W_k\vx_j)$ and $p_j = \sim(\vq, \vk_j) / \sum_j \sim(\vq, \vk_j)$. 

On the other hand, mean-shift clustering looks like the following:
\begin{equation}
    m(\vx) = \frac{\sum_j K(\vx_j,\vx) \vx_j}{\sum_j K(\vx_j,\vx)}
\end{equation}
where $K(\vx_j,\vx)$ is a kernel matrix that measure the similarity between $\vx_j$ and $\vx$. According to the mean-shift algorithm, in the next iteration, we will simply replace $\vx$ with $m(\vx)$.

So in some sense, self-attention is just to do some kind of clustering for the input embedding $\vq$ and $\vk$, plus a transformation of the embedding to another place. The term ``projection'' is due to the fact that there is a projection matrix $W_v$ on $\vx$ for the next level.  


\textbf{Residue connection and LayerNorm}. Compared to mean-shift, Transformer layer has residue connection. Therefore, for single-headed attention, what you actually get is $\vv + \vx$, followed by a LayerNorm. 
For the residue connection, the mean-shift analog already shows the output $m(\vx)$ contains $\vx+$ part. The reason why we need residue connection is that the self-attention part might only model the ``change'' of $\vx$ in the mean-shift picture, rather than the full update of $\vx$.

\section{The role of self-attention}
Consider we have a vocabulary of size $m$ and $d$ dimensional embedding space. In practice, many papers in NLP have reported clustering behaviors of word embeddings: such a clustering of word embedding naturally occurs after training.  

An explanation for the above phenomenon is that, by grouping these word embedding together, we might generalize better, since similarity in word now can transfer (e.g., A linked to B, B linked to C, then A might link to C as well) and generalization follows. 

Let's treat it as a fact and focus on how this is achieved and how self-attention plays a role here.  

\subsection{The capacity of embedding layer}
First let us take a look at the following pairwise distance constraints between word embedding (e.g., some words should be close to each other, some should be far away from each other) as the following:
\begin{equation}
    \|\vx_i - \vx_j\| = D(i, j) \label{eq:mds}
\end{equation}
where $D(i,j)$ is large for $i$ and $j$ that should be far apart and $D(i,j)$ is small for $i$ and $j$ that are close to each other. In visualization, this is called Multidimensional Scaling (MDS)~\cite{cox2008multidimensional}. 

Note that in neural network training, the constraint (Eqn.~\ref{eq:mds}) is not directly enforced during training, but the clustering naturally happens. Since we talk about capacity, how we achieve Eqn.~\ref{eq:mds} doesn't matter for now. 

In general we cannot find a \emph{fixed} low-dimensional embedding ($d \ll m$) to satisfy these constraints, since we only have $md$ parameters ($m$ vectors, each has $d$ entries), but $m^2$ constraint. So two vectors that are supposed to be close may not be close enough (but hopefully they remain close to each other). 

\subsection{The role of self-attention}
For this, the self-attention mechanism comes to the rescue, trading model-size with additional computation. It fulfills what (static) embedding cannot achieve: to further group the embedding vectors together in a multi-layer structure.

Note that one sentence never covers all $d$ vocabularies. Once the words in the sentence are picked, they are grouped together via self-attention layers to collectively represent a concept that can be useful for the task. 

\subsection{How the clustering happens through self-attention?}
Now one fundamental questions arise: How the static clustering of embedding happens during end-to-end training? In practice, no one explicitly enforces the MDS constraint (Eqn.~\ref{eq:mds}). 

Let's start with a simple example. we have two unit embedding: $\vx$ and $\vy$ with the normalization condition that $\|\vx\|_2 = 1$ and $\|\vy\|_2 = 1$, and a simple self-attention layer (without projection) which output $\vz$:
\begin{equation}
    \vz = (1 - p)\vx + p \vy 
\end{equation}
Where the attention map is:
\begin{equation}
    p = \frac{e^{\vx^\t\vy}}{e^{\vx^\t\vx} + e^{\vx^\t\vy}} = \frac{1}{1 + e^{1 - \vx^\t\vy}}
\end{equation}
Note that here we attend to $\vx$ so $0 < p < 1/2$ always. The last two is due to normalization condition. 

Now we consider a loss function $L = -\frac{1}{2}\|\vz\|_2^2$. The intuition behind is that ``for some reason, we found that $\vz$ is a good representation for our task, and want to make sure its length is as long as possible''. 

Under this context, what would be the gradient rule for $\vx$ and $\vy$? Will they cluster together? 

The answer is yes! We could compute 
\begin{eqnarray}
    \frac{\partial \vz}{\partial \vx} &=& (1-p) I + \frac{\partial p}{\partial \vx} (\vy - \vx)^\t \\
    \frac{\partial \vz}{\partial \vy} &=& p I + \frac{\partial p}{\partial \vy} (\vy - \vx)^\t 
\end{eqnarray}
Let $t := 1 - \vx^\t\vy$ and define the following function with respect to $t$: 
\begin{equation}
    f(t) := (\vx-\vy)^\t\vz = (1-2p)(1 - \vx^\t\vy) > 0
\end{equation}
Therefore, we can compute the gradient for $\vx$ and gradient for $\vy$:
\begin{eqnarray}
    -\vg_\vx &:=& -\frac{\partial L}{\partial \vx} = -\frac{\partial \vz}{\partial \vx} \frac{\partial L}{\partial \vz} = (1 - p)^2\vx + p(1-p)(1 - f(t))\vy \\ 
    -\vg_\vy &:=& -\frac{\partial L}{\partial \vy} = -\frac{\partial \vz}{\partial \vy} \frac{\partial L}{\partial \vz} = p^2\vy + p(1-p)(1 - f(t))\vx  
\end{eqnarray}
Note that since $\vx$ and $\vy$ are kept to be normalized, the term $(1-p)^2\vx$ in $\partial L / \partial \vx$ is gone (and similarly $p^2\vy$ for $\vg_\vy$). So how $\vx$ and $\vy$ move depends on the sign of $1 - f(t)$. 

With some computation, we could see $0 < f(t) < 1$ when $t < 1.5424$. In summary, if $\vx^\t\vy > -0.4576$, then the (negative) gradient of $\vx$ pushes it towards $\vy$ and pushes $\vx$ towards $\vy$, and the clustering of static embedding happens during training. Note that since both $\vx$ and $\vy$ are normalized, $-1 \le \vx^\t\vy \le 1$, so this is a quite loose condition and can be easily satisfied.  

\subsection{Multiple embeddings}
People might wonder what happen to multiple unit embeddings $\vx, \vy_1, \vy_2, \ldots, \vy_K$? In this case, we can similarly define self-attention probability $p_i$ (note that here we consider the case that every embedding attends to $\vx$):
\begin{equation}
    p_i := \frac{e^{\vx^\t\vy_i}}{e^{\vx^\t\vx} + \sum_j e^{\vx^\t\vy_j}} = \frac{e^{\vx^\t\vy_i}}{1 + \sum_j e^{\vx^\t\vy_j}}
\end{equation}
Define $p_S := \sum_{i=1}^K p_i = 1 - \frac{1}{1 + \sum_j e^{\vx^\t\vy_j}} < 1$ and we have:
\begin{equation}
    \vz = (1 - p_S) \vx + \sum_i p_i \vy_i
\end{equation}
Let $\tilde p_i := p_i / p_S$ be the (normalized) probability on $\vy_i$ and $\bar\vy := \frac{1}{p_S}\sum_i p_i \vy_i = \sum_i \tilde p_i \vy_i$ be the weighted mean of $\{\vy_i\}$ other than $\vx$, then we have: 
\begin{equation}
    \vz = (1 - p_S)\vx + p_S \bar\vy
\end{equation}
Now we can still compute the partial derivative:
\begin{eqnarray}
    \frac{\partial p_j}{\partial \vx} &=& p_j \left[-p_S \bar \vy + \vy_j\right] \\
    \frac{\partial p_j}{\partial \vy_i} &=& p_i \left[-p_j + \mathbb{I}(i = j)\right]\vx 
\end{eqnarray}
which gives
\begin{eqnarray}
\frac{\partial \vz}{\partial \vx} &=& (1 - p_S)I + \sum_j \frac{\partial p_j}{\partial \vx}(\vy_j - \vx)^\t  \\
\frac{\partial \vz}{\partial \vy_i} &=& p_i I + \sum_j \frac{\partial p_j}{\partial \vy_i}(\vy_j - \vx)^\t
\end{eqnarray}
After some manipulation, we have:
\begin{equation}
\frac{\partial \vz}{\partial \vx} = (1-p_S) [I + p_S \bar\vy (\bar \vy - \vx)^\t ] + p_S Q 
\end{equation}
where $Q := \sum_j \tilde p_j (\vy_j - \bar \vy)(\vy_j - \bar\vy)^\t$ is the weighted covariance matrix of data points $\{\vy_j\}$. 

Similar to the two unit case, we want to check $-\vg_\vx$ to see how the embedding $\vx$ changes over time.
\begin{eqnarray}
    -\vg_\vx &=& -\frac{\partial L}{\partial \vx} = -\frac{\partial \vz}{\partial \vx} \frac{\partial L}{\partial \vz} \\
    &=& (1 - p_S)^2\vx + p_S \left[(1-2p_S)\vx^\t\bar\vy - (1 - p_S) + p_S \|\bar\vy\|^2\right]\bar\vy + p_S Q\vz \nonumber 
\end{eqnarray}
If things are already quite clustered, then $\|\bar\vy\| \approx 1$ (usually $\|\bar\vy\|_2 < 1$ since sphere is a convex set), $Q\vz \approx 0$ (since $Q$ spans on the tangent space of $\vz$ at the sphere and $\vz$ is perpendicular to it), and we have:
\begin{equation}
    -\vg_\vx \approx (1 - p_S)^2 \vx + p_S(1-2p_S)(\vx^\t\bar\vy - 1) \bar\vy
\end{equation}
It is clear that $\vx^\t\bar\vy < 1$. When $p_S > 1/2$, which is high likely for large $K$, then $-\vg_\vx$ has positive component of $\bar\vy$ and $\vx$ will move towards $\bar\vy$. 

On the other hand, we could also check
\begin{equation}
    \frac{\partial \vz}{\partial \vy_i} = p_i \left[I + (1 - p_S)\vx (\bar\vy - \vx)^\t\right] + p_i \vx (\vy_i - \bar\vy)^\t
\end{equation}
which gives an expression of $-\vg_\vy$:
\begin{equation}
\cdot    
\end{equation}
With the same argument, it moves towards $\bar\vy$ (so all $\vy_i$ will cluster together) and towards $\vx$. 

When there is a $W_k$ and $W_q$ before the embedding, following the same logic, only the column subspace of $W_k$ (or $W_q$) will be clustered together. On the other hand, the value part will be different in order to enable encoding of more complicated concepts based on co-occurrence of multiple tokens. 

\def\pr{\mathbb{P}}

\section{Link self-attention with generative models.}
Consider the following self-attention structure. Consider an embedding matrix $X\in \rr^{n\times d}$ and for embedding $\vx_i$ and $\vx_j$, let
\begin{equation}
    \vy_{ij} = \phi(\vx_i; \vx_j) := (1 - \beta_{ij})\vx_i + \beta_{ij} \vx_j, \quad \beta_{ij} := \frac{e^{\vx_i^\t \vx_j}}{e^{\vx_i^\t \vx_i} + e^{\vx_i^\t \vx_j}} \\ 
\end{equation}
Here $\phi(\vx_i;\vx_j) := \vx_i + \beta_{ij}(\vx_j-\vx_i)$ is the self-attention operation. More properties of this operator $\phi$ need to be explored. Then we want to maximize the following objective:
\begin{equation}
    \max_{X, \|\vx_i\|_2=1} \sum_{ijk} \pr(k|i,j) \vy^\t_{ij} \vx_k 
\end{equation}
or more formally, using a softmax to avoid trivial solution $\vx_i \equiv \vx$, we have:
\begin{equation}
    \max_{X, \|\vx_i\|_2=1} J := \max_{X, \|\vx_i\|_2=1} \sum_{ijk} \pr(k|i,j) \log\delta_{ijk}, \quad \delta_{ijk} := \frac{e^{\vy^\t_{ij} \vx_k}}{\sum_k e^{\vy^\t_{ij} \vx_k}} 
\end{equation}
which is:
\begin{equation}
    \max_{X, \|\vx_i\|_2=1} \sum_{ijk} \pr(k|i,j) \left[\vy^\t_{ij} \vx_k - \log \sum_k e^{\vy^\t_{ij} \vx_k} \right] 
\end{equation}
We can compute its gradient update. Here we assume the index $k$ never appears in index $i$ and $j$ (encoding and decoding matrices are decoupled), then by gradient rule, we have:
\begin{equation}
    \dot \vx_k = \frac{\partial L}{\partial \vx_k} = P^\perp_{\vx_k} \sum_{ij} \pr(k|i,j) (1 - \delta_{ijk}) \vy_{ij}
\end{equation}
where $P^{\perp}_{\vx_k}$ is the projection matrix that projects a vector to the orthogonal complement space of $\vx_k$. The projection is due to the constraint $\|\vx_k\|_2=1$. If the training converges ($\dot \vx_k = 0$), then we know that 
\begin{equation}
    \sum_{ij} \pr(k|i,j) (1 - \delta_{ijk}) \vy_{ij} = \gamma \vx_k
\end{equation}
for some $\gamma > 0$ (note that $\gamma < 0$ will be an unstable stationary point).  

Depending on different structure of the generative model specified by $P(k|i, j)$, we might end up learning different embedding matrix $X$. 

The first thing we want to check is independency. Assume that for some specific token $k$ and $i$, we have $\pr(k|i, j) = \pr(k|i)$ for any $j$, which means that the frequency of token $k$ has nothing to do with the second entry $j$. Furthermore, token $k$ is not connected with other token $i'\neq i$, i.e, $\pr(k|i', j) \equiv 0$. If we just let $\delta_{ijk} = \delta > 0$, then we have:
\begin{equation}
    \pr(k|i) \sum_j \vy_{ij} = \gamma' \vx_k
\end{equation}
which yields
\begin{equation}
    \pr(k|i) n\vx_i + \sum_j \beta_{ij}(\vx_j-\vx_i) = \gamma' \vx_k
\end{equation}
And we could possibly show that $\sum_j \beta_{ij}(\vx_j-\vx_i) \approx 0$ since $\beta_{ij} = 1 / (1 + e^{1 - \vx_i^\t\vx_j})$ applies equal weights for embeddings around $\vx_i$ and they cancel out. Therefore, $\vx_k$ is aligned with $\vx_i$.  

Another thing we might want to check is identification of two tokens. Assume that there exists two tokens $j_1$ and $j_2$ and specific $k$ and $i$, so that $\pr(k|i,j_1) = \pr(k|i,j_2)$. For other $k, i, j$ combination $\pr(k|i,j) \equiv 0$, then we have:
\begin{eqnarray}
    \pr(k|i,j_1) \vy_{ij_1} = \gamma_1 \vx_k
\end{eqnarray}
(not sure how to continue). 

If we have $W_q$, $W_k$ and $W_v$, then the formulation doesn't change that much. The only difference here is that now 
\begin{equation}
    \beta_{ij} := \frac{e^{\vx_i^\t W_{pq} \vx_j}}{e^{\vx_i^\t W_{pq} \vx_i} + e^{\vx_i^\t W_{pq} \vx_j}}
\end{equation}
and $\vy_{ij}^\t \vx_k$ now becomes $\vy_{ij}^\t W_v \vx_k$. 

\end{document}